\documentclass{article} 
\usepackage[accepted]{icml2021}[] 

\usepackage{times}
\usepackage{hyperref}
\usepackage{url}
\usepackage{graphicx}
\usepackage{multirow}
\usepackage{units}
\usepackage{lipsum}
\usepackage{textcomp}
\usepackage{fontawesome}
\usepackage{bbm}
\usepackage{todonotes}
\usepackage{listings}
\usepackage{textgreek} 
\usepackage{amsmath} 
\usepackage{amssymb,amsthm,mathtools}

\usepackage[capitalise]{cleveref}
\usepackage{stfloats} 
\usepackage{float}
\usepackage{placeins} 

\usepackage{tablefootnote} 
\usepackage[roman]{parnotes} 
\usepackage{tabulary}
\usepackage{booktabs}
\usepackage{tabularx}

\usepackage{changepage}

\usepackage{algorithm}  
\usepackage{algorithmicx}
\usepackage{algpseudocode}
\definecolor{mygray}{gray}{0.5}
\definecolor{cblue}{RGB}{8, 85, 153}
\definecolor{darkblue}{RGB}{31, 64, 96}

\definecolor{cgreen}{RGB}{8, 153, 83}
\definecolor{green}{RGB}{8, 200, 50}
\definecolor{cmaroon}{RGB}{128, 0, 0}
\algdef{SE}[DOWHILE]{Do}{doWhile}{\algorithmicdo}[1]{\algorithmicwhile\ #1}%

\renewcommand\algorithmicdo{}

\Crefname{section}{Sec}{Secs.}
\Crefname{figure}{Fig}{Figs.}

\usepackage{amsfonts}

\usepackage{tikzsymbols}

\usepackage[definitionLists,hashEnumerators,smartEllipses,hybrid]{markdown} 
\usepackage{todonotes} 

\newcommand{\method}{HS}
\newcommand{\methods}{HS }
\newcommand{\methodlong}{Hierarchical Shrinkage}
\newcommand{\leafbased}{LBS}

\newcommand{\bx}{\mathbf{x}}
\newcommand{\by}{\mathbf{y}}

\newcommand{\bbeta}{\boldsymbol{\beta}}

\newcommand{\R}{\mathbb{R}}

\newcommand{\E}{\mathbb{E}}

\newcommand{\Var}{\textnormal{Var}}

\DeclarePairedDelimiter{\braces}{\lbrace}{\rbrace}
\DeclarePairedDelimiter{\paren}{(}{)}

\DeclarePairedDelimiter{\abs}{|}{|}
\DeclarePairedDelimiter{\norm}{\|}{\|}

\newcommand{\data}[1][n]{\mathcal{D}_{#1}}

\newcommand{\node}{\mathfrak{t}}

\newtheorem{theorem}{Theorem}[]

\icmltitlerunning{\methodlong}

\begin{document} 

\twocolumn[
\icmltitle{\methodlong: \\  improving the accuracy and interpretability 
of tree-based methods}
\icmlsetsymbol{equal}{*}
\begin{icmlauthorlist}
    \icmlauthor{Abhineet Agarwal}{equal,ucbphysics}
    \icmlauthor{Yan Shuo Tan}{equal,ucbstat}
    \icmlauthor{Omer Ronen}{ucbstat}
    \icmlauthor{Chandan Singh}{ucbeecs}
    \icmlauthor{Bin Yu}{ucbstat,ucbeecs}
\end{icmlauthorlist}
\icmlcorrespondingauthor{Bin Yu}{binyu@berkeley.edu}
\icmlaffiliation{ucbphysics}{Department of Physics, UC Berkeley, Berkeley, California, USA}
\icmlaffiliation{ucbstat}{Department of Statistics, UC Berkeley, Berkeley, California, USA}
\icmlaffiliation{ucbeecs}{EECS Department, UC Berkeley, Berkeley, California, USA}
\vskip 0.3in
] 
\printAffiliationsAndNotice{\icmlEqualContribution} 

\begin{abstract}

Tree-based models such as decision trees and random forests (RF) are a cornerstone of modern machine-learning practice.
To mitigate overfitting, trees are typically regularized by a variety of techniques that modify their structure (e.g. pruning).
We introduce \methodlong~(HS), a post-hoc algorithm that does not modify the tree structure, and instead regularizes the tree 
by shrinking the prediction over each node towards the sample means of its ancestors. 
The amount of shrinkage is controlled by a single regularization parameter and the number of data points in each ancestor.
Since \methods is a post-hoc method, it is \emph{extremely fast}, compatible with \emph{any} tree-growing algorithm, and can be used synergistically with other regularization techniques. 
Extensive experiments over a wide variety of real-world datasets show that \methods substantially increases the predictive performance of decision trees, even when used in conjunction with other regularization techniques.
Moreover, we find that applying \methods to each tree in an RF often improves accuracy, as well as its interpretability by simplifying and stabilizing its decision boundaries and SHAP values. 
We further explain the success of \methods in improving prediction performance by showing its equivalence to ridge regression on a (supervised) basis constructed of decision stumps associated with the internal nodes of a tree.
All code and models are released in a full-fledged package available on Github.\footnote{\methods is integrated into the \texttt{imodels} package \href{https://github.com/csinva/imodels}{\faGithub\,github.com/csinva/imodels}~\cite{singh2021imodels} with an sklearn-compatible API. Experiments for reproducing the results here can be found at \href{https://github.com/Yu-Group/imodels-experiments}{\faGithub\,github.com/Yu-Group/imodels-experiments}.}
 
 \end{abstract}
\section{Introduction}
\label{sec:intro}

Decision tree models, used for supervised learning since the 1960s \cite{morgan1963problems,messenger1972modal,quinlan1986induction}, have recently attained renewed prominence because they embody key elements of interpretability: shallow trees are easily described and visualized, and can even be implemented by hand.
While the precise definition and utility of interpretability have been a subject of much debate \cite{murdoch2019definitions,doshi2017towards,rudin2019stop,rudin2021interpretable}, all agree that it is an important notion in high-stakes decision-making, such as medical-risk assessment and criminal justice.
For this reason, decision trees have been widely applied in both areas \cite{steadman2000classification, kuppermann2009identification, letham2015interpretable, angelino2017learning}. 

By far the most popular decision tree algorithm is \citet{breiman1984classification}'s Classification and Regression Trees (CART).
These can be ensembled to form a Random Forest (RF) \cite{breiman2001random} or
used as weak learners in Gradient Boosting (GB) \cite{friedman2001greedy}; both algorithms have achieved state-of-the-art performance over a wide class of prediction problems \cite{caruana2006empirical,caruana2008empirical,fernandez2014we,olson2018data,hooker2021bridging}, and are implemented in popular machine learning packages such as \texttt{ranger} \cite{wright2015ranger} and \texttt{scikit-learn} \cite{pedregosa2011scikit}. 
Variants of these algorithms, such as \citet{basu2018iterative}'s iterative random forest for finding stable interactions, have found use in scientific applications.

In view of the widespread use of tree-based methods, we seek to provide a new lens on their regularization. On the one hand, decision trees often obey traditional statistical wisdom in that they need to be regularized to prevent overfitting. In practice, this is carried out by specifying an early stopping condition for tree growth, such as a maximum depth, or alternatively, pruning the tree after it is grown \cite{friedman2001elements}. These procedures, however, only regularize tree models via their \emph{tree structure}, and it is usually taken for granted that the prediction over each leaf should be the average response of the training samples it contains. 
We show that this can be very limiting: \emph{shrinking these predictions in a hierarchical fashion can significantly reduce generalization error in both regression and classification settings} (e.g. see~\cref{fig:intro}).

\begin{figure*}[t]
    \centering
    \includegraphics[width=0.4\textwidth]{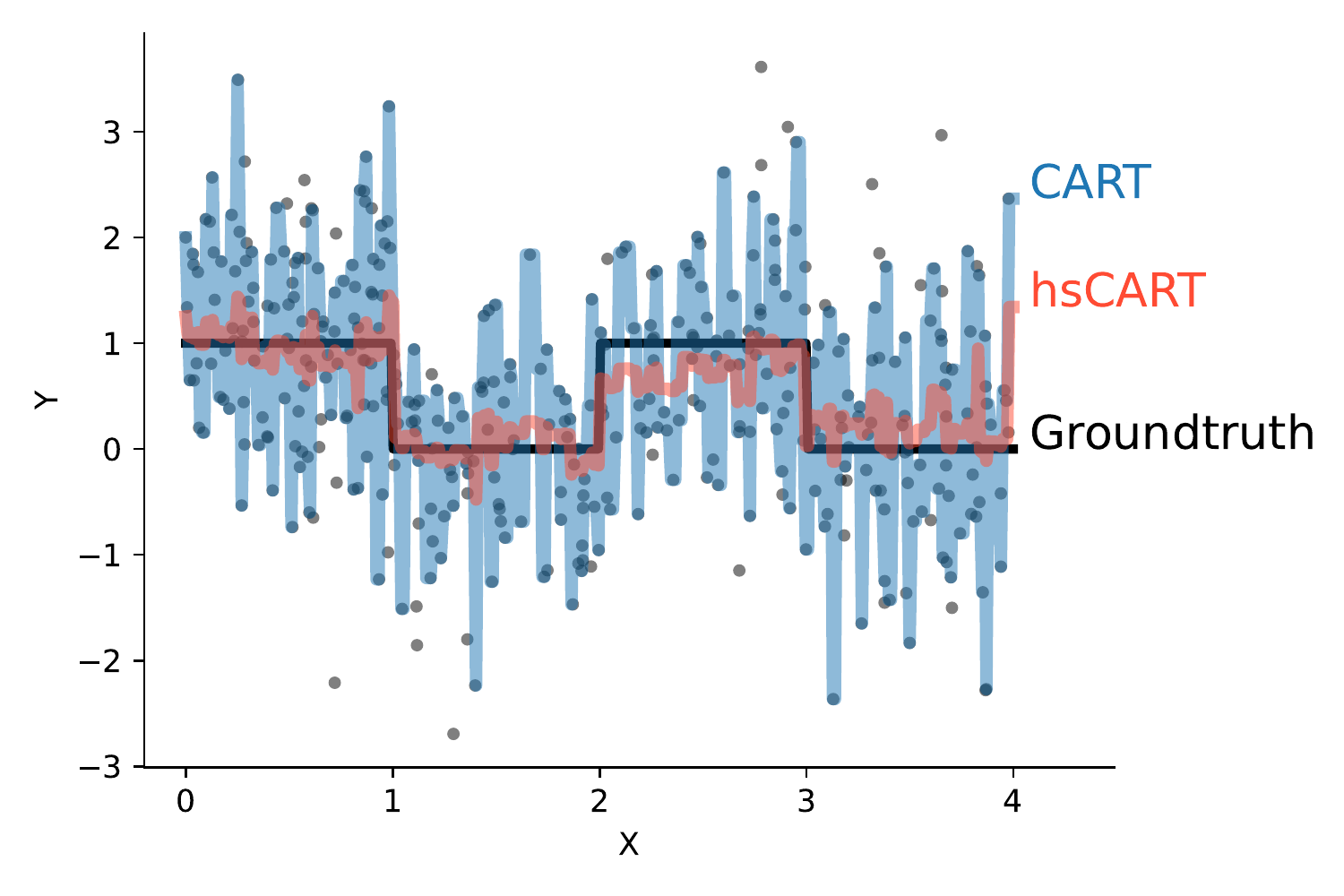}
    \includegraphics[width=0.4\textwidth]{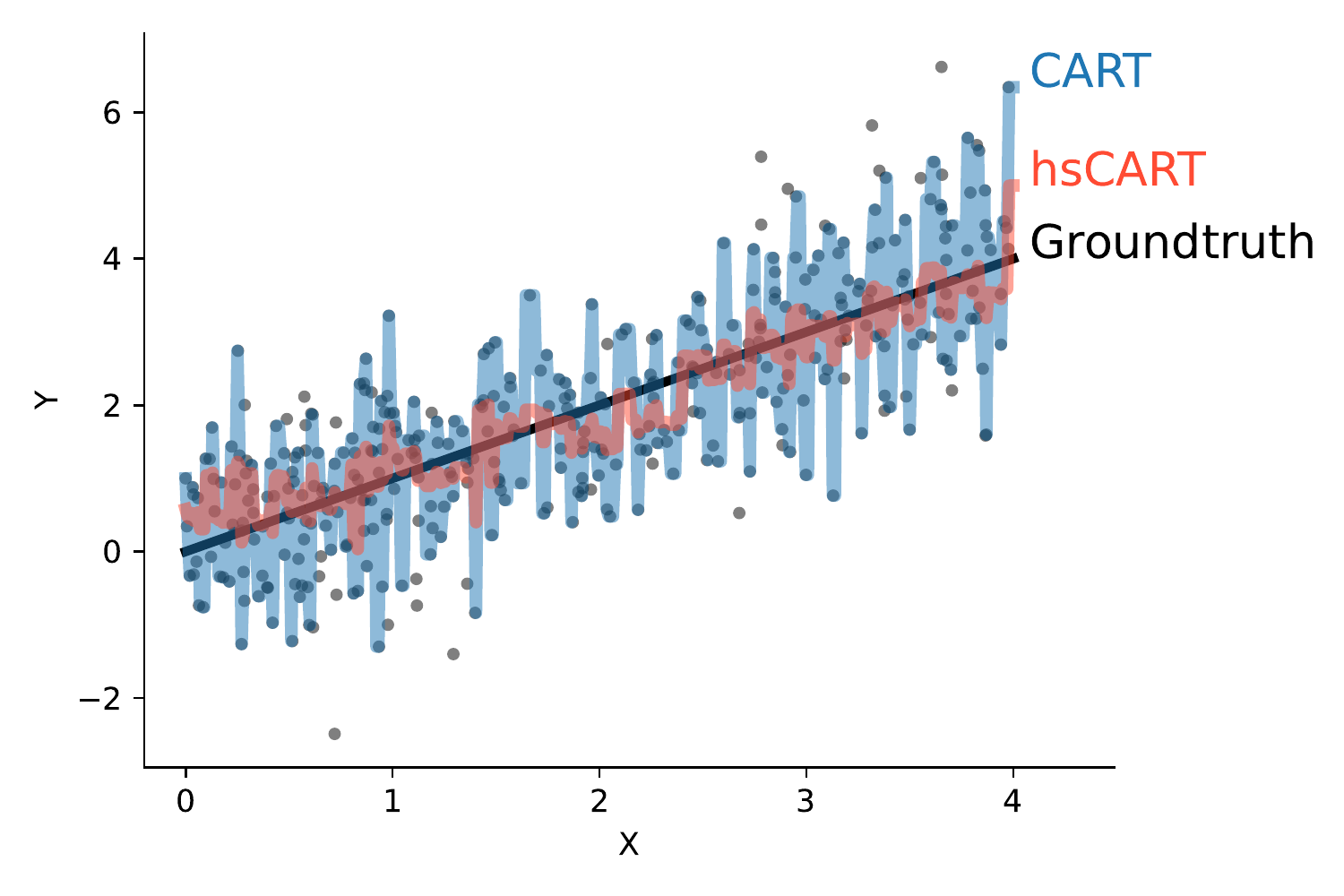}
    \vspace{-13pt}
    \caption{Example of \methods for toy univariate regression problems. \methods regularizes model predictions to improve estimates in noisy leaves that have few samples. CART is fit to the data in the blue dots and then \methods is applied posthoc (hsCART).} 
    \label{fig:intro}
\end{figure*}

On the other hand, trees used in an RF are usually not explicitly regularized and interpolate the data by being grown to purity (e.g. see the default settings of \texttt{scikit-learn} and \texttt{ranger}).
Instead, RF prevents overfitting by relying upon the randomness injected into the algorithm during tree growth, which acts as a form of implicit regularization~\cite{breiman2001random, mentch2019randomization}. We show that apart from this implicit regularization, more regularization, in the form of  hierarchical shrinkage, does improve generalization even while using a smaller ensemble for many data sets.

Equally important, regularizing RFs also improves the quality of their post-hoc interpretations.
RFs are usually interpreted via their feature and interaction importances, which have been used to provide scientific insight in areas such as remote sensing and genomics \cite{svetnik2003random,evans2011modeling,belgiu2016random,diaz2006gene,boulesteix2012overview,chen2012random,basu2018iterative,behr2020learning}.
The reproducibility and scientific meaning of such interpretations become questionable when the underlying RF model has poor predictive performance \cite{murdoch2019definitions}, or when they are highly sensitive to data perturbations \cite{yu2013stability}.
We show that \methods improves the interpretability of RF by both simplifying and stabilizing its decision boundaries and SHAP  values \cite{lundberg2017unified} on a number of real-world data sets.

Our proposed method, which we call \emph{\methodlong} (\method), is an extremely fast and simple yet effective algorithm for the \textit{post-hoc} regularization of \emph{any} tree-based model.
It does not alter the tree structure, and instead replaces the average response (or prediction) over a leaf in the tree with a \emph{weighted average of the mean or average responses over the leaf and each of its ancestors}.
The weights depend on the number of samples in each leaf, and are controlled by a single regularization parameter $\lambda$ that can be tuned efficiently via generalized cross validation.
\methods is agnostic to the way the tree is constructed and can be applied post-hoc to trees constructed with greedy methods such as CART and C4.5~\cite{quinlan2014c4}, as well optimal decision trees grown via dynamic programming or integer optimization techniques \cite{lin2020generalized}.

A more naive form of shrinkage, which we call \emph{leaf-based shrinkage} (\leafbased), appears as part of XGBoost \cite{chen2016xgboost}:
whenever a new tree is grown, the average response (prediction) over each leaf is shrunk directly towards the sample mean of the responses.
\leafbased~ also occurs\footnote{When conditioned on the structure of a given tree, as well as all other trees in the ensemble, the posterior distribution for the contribution of a leaf node is a product of Gaussian likelihood functions centered at the model residuals as well as a Gaussian prior.
A simple calculation shows that the posterior mean can be obtained from the residual mean via \leafbased.} in Bayesian Additive Regression Trees (BART) \cite{chipman2010bart}, which grows an ensemble of trees via a backfitting MCMC algorithm.
Comparing \leafbased~to \methods on several rea-world datasets shows that \methods uniformly has better predictive performance than \leafbased.

We explain the advantages of \methods by building on recent work by \citet{klusowski2021universal}, who uses decision stumps associated to each interior node of a tree to construct a new (supervised) feature representation. 
The original tree model is recovered as the linear model obtained by regressing the responses on the supervised features.
We show that \methods is exactly the \emph{ridge regression} solution in this supervised feature space, while \leafbased~can also be viewed as ridge regression, but with a different (supervised) feature space (of the same dimension) that relies only on the leaf nodes.  
This allows us to use ridge regression calculations heuristically to partially explain both the reasonableness of the shrinkage scaling in \method, as well as our empirical evidence that \methods achieves consistently better predictive accuracy than \leafbased~(see \cref{sec:theory}).

The rest of the paper is organized as follows.
\cref{sec:methods} gives a formal statement of \method, and discusses several computational considerations.
\cref{sec:theory} discuss the interpretation of \methods as ridge regression on the supervised features.
\cref{sec:results} presents the results of extensive numerical experiments on simulated and real world data sets that illustrate the gains in prediction accuracy from applying the method.
\cref{sec:interpretations} shows how \methods improves the interpretability of RFs.

\section{The \methodlong~(\method) algorithm}
\label{sec:methods}

Throughout this paper, we work in the supervised learning setting where we are given a training set $\data = \braces*{(\bx_{i},y_{i})}_{i=1}^n$, from which we learn a tree model $\hat{f}$ for the regression function.
Given a query point $\bx$, let $\node_L \subset \node_{L-1} \subset \cdots \subset \node_0$ denote its leaf-to-root path, with $\node_L$ and $\node_0$ representing its leaf node and the root node respectively.
For any node $\node$, let $N(\node)$ denote the number of samples it contains, and $\hat{\E}_{\node}\braces*{y}$ the average response.
The tree model prediction can be written as the telescoping sum:
$$
\hat f(\bx) 
= \hat\E_{\mathfrak{t}_0}\braces{y} + \sum_{l=1}^L 
\paren*{\hat \E_{\mathfrak{t}_l}\braces{y} - \hat \E_{\mathfrak{t}_{l-1}}\braces{y}}. 
$$

\methods transforms $\hat f$ into a shrunk model $\hat{f}_{\lambda}$ via the formula:
\begin{equation}
    \label{eq:hierarchical_shrinkage}
    \hat{f}_{\lambda}(\bx)
    \coloneqq \hat\E_{\mathfrak{t}_0}\braces{y} + \sum_{l=1}^L \frac{\hat \E_{\mathfrak{t}_l}\braces{y} - \hat \E_{\mathfrak{t}_{l-1}}\braces{y}}{1 + \lambda/N(\mathfrak{t}_{l-1})},
\end{equation}
where $\lambda$ is a hyperparameter chosen by the user, for example by cross validation. We emphasize that \methods maintains the tree structure, and \emph{only} modifies the prediction over each leaf node.

Since \methods continues to make a constant prediction over each leaf node, 
our method thus comprises a one-off modification of these values. This can be computed in $O(m)$ time, where $m$ is the total number of nodes in the tree.
No other aspects of the underlying data structure are modified, with test time prediction occurring in exactly the same way as in the original tree.
Moreover, our method \methods does not even need to see the original training data, and only requires access to the fitted tree model.
These features make it extremely lightweight and easy to implement, as we have done in the open-source package 
\texttt{imodels} \cite{singh2021imodels}.
By applying \methods to each tree in an ensemble, it can be generalized to methods such as RF and gradient boosting.

While not typically done, it is possible to regularize RFs via other hyperparameters such as maximum tree depth. 
Tuning these hyperparameters, however, requires refitting the RF at every value in a grid. This quickly becomes computationally expensive, even for moderate dataset sizes\footnote{Many popular tree-building algorithms such as CART have a run time of $O( p n^2)$ for constructing a binary tree.}, over multiple folds in a cross-validation (CV) set up.
In contrast, since \methods is applied post-hoc, \emph{we only need to fit the RF once per CV fold}, leading to potentially enormous time savings.
In addition, due to the connection between our method and ridge regression, it is even possible to get away with fitting the RF only \emph{once} by using generalized cross-validation \cite{Golub79GCV}\footnote{This allows for efficient computation of leave-one-out cross-validation error, which can be used to select $\lambda$, without refitting the RF.}.

We also note the formula for \leafbased:
\begin{equation}
\label{eq:leaf-based_shrinkage}
    \hat f_{\lambda}^l(\bx) 
    \coloneqq \hat\E_{\node_0}\braces*{y} 
    + \frac{\hat \E_{\mathfrak{t}_L}\braces{y} - \hat \E_{\mathfrak{t}_{0}}\braces{y}}{1 + \lambda/N(\mathfrak{t}_{L})}.
\end{equation}
Expanding this into a telescoping sum similar to \eqref{eq:hierarchical_shrinkage}, we see that the major difference between the two formulas is that whereas \leafbased~shrinks each term by the same factor, \methods shrinks each term by a different amount, with the amount of shrinkage controlled by the number of samples in the ancestor. 
This increased flexibility leads to better prediction performance for the final model, as evidenced by our results presented in the next section.
\section{\methods as ridge regression on supervised features}
\label{sec:theory}

Recent work by \citet{klusowski2021universal} showed that decision trees are linear models on features obtained via supervised feature learning.
To see this, consider a tree model $\hat f$, with a fixed indexing of its interior nodes $\braces*{\node_0, \node_1,\ldots,\node_{m-1}}$.
We first associate to each node $\node$ the decision stump
\begin{equation} \label{eq:decision_stump}
\psi_{\node}(\bx) = \frac{N(\mathfrak{t}_{R})\mathbf{1}\{\bx \in \mathfrak{t}_{L}\} - N(\mathfrak{t}_{L})\mathbf{1}\{\bx \in \mathfrak{t}_{R}\}}{\sqrt{N(\mathfrak{t}_{L})N(\mathfrak{t}_{R})}},
\end{equation}
where $\node_L$ and $\node_R$ denote the left and right children of $\node$ respectively.
This is a tri-valued function that is positive on the left child, negative on the right child, and zero everywhere else.
Concatenating the decision stumps together yields a supervised feature map
via $\Psi(\bx) = \paren*{\psi_{\node_0}(\bx), \ldots, \psi_{\node_{m-1}}(\bx)}$ and a transformed training set
$\Psi(\data) \in \mathbb{R}^{n \times m}$.
One can easily check that these feature vectors
are orthogonal in $\R^n$, and furthermore that their squared $\ell_2$ norms are the number of samples contained in their corresponding nodes: $\norm{\psi_{\node_i}}^2 = N(\node_i)$.

More interestingly, \citet{klusowski2021universal} was able to show (see Lemma 3.2 therein) that we have functional equality between the tree model and the kernel regression model with respect to the supervised feature map $\Psi$, or in other words, $\hat f(\bx) = \hat{\bbeta}^T \Psi(\bx)$, where $\hat{\bbeta} = \Psi(\data)^\dagger\by$. 
An easy extension of his proof yields the following result.

\begin{theorem}
    \label{thm:ridge_to_shrinkage} 
    Let $\hat \bbeta_\lambda$ be the solution to the ridge regression problem
    \begin{equation} \label{eq:ridge_equation}
        \min_{\bbeta} \braces*{\sum_{i=1}^n \paren*{\bbeta^T\Psi(\bx_i) - y_i}^2 + \lambda\norm*{\bbeta}^2}.
    \end{equation}
    We have the functional equality $\hat f_\lambda(\bx) = \hat \bbeta_\lambda^T \Psi(\bx)$.
\end{theorem}
\begin{proof}
See \cref{sec:supp_theory}.
\end{proof}

Since the decision stumps \eqref{eq:decision_stump} are orthogonal, we can decompose \eqref{eq:ridge_equation} into $m$ independent univariate ridge regression problems, one with respect to each node $\node$:
\begin{equation} \label{eq:univariate_ridge}
    \min_{\beta} \braces*{\sum_{i=1}^n \paren*{\beta\psi_\node(\bx_i) - y_i}^2 + \lambda\beta^2}.
\end{equation}

Next, we use this connection of \methods to ridge regression to argue heuristically that \emph{the same $\lambda$ works well for each regression subproblem \eqref{eq:univariate_ridge}}.
This helps to justify our choice of denominator for each term in the \methods formula \eqref{eq:hierarchical_shrinkage} (a different choice would have led to a rescaling of the features $\psi_{\node_i}$.)

Assume for the moment that the tree structure and hence the feature map is independent of the responses, which can be achieved via sample splitting. 
This is known in the literature as the ``honesty condition'', and has been widely used to simplify the analysis of tree-based methods \cite{athey2016recursive}. 
Define $\bbeta_* \in \R^m$ to have the value
\begin{equation} \label{eq:formula_for_beta_star}
    \bbeta_*(\node) \coloneqq \frac{\sqrt{N(\node_L)N(\node_R)}}{N(\node)}\paren*{\E\braces*{y~|~\node_L} - \E\braces*{y~|~\node_R}},
\end{equation}
for the coordinate associated with each node $\node$.
For any query point $\bx$, $\bbeta_*^T \Psi(\bx)$ gives the mean response over the leaf $\node(\bx)$ containing it.
Furthermore, knowing $\Psi(\bx)$ is equivalent to knowing the leaf containing $\bx$.
Putting these two facts together show that the population residuals $r_i \coloneqq y_i - \bbeta_*^T\Psi(\bx)$ satisfy $\E\braces*{r_i~|~\Psi(\bx_i)} = 0$, so that we have a generative linear model,
in which we can calculate that the optimal regularization parameter for \eqref{eq:univariate_ridge} is equal to $\lambda_{opt}(\node) = \sigma^2(\node)/\bbeta_*(\node)^2$, where $\sigma^2(\node)$ is roughly equal to the conditional variance\footnote{More precisely, it is equal to $\frac{N(\node_L)}{N(\node)}\Var\braces*{r~|~\node_R} + \frac{N(\node_R)}{N(\node)}\Var\braces*{r~|~\node_L}$.} of the residual over $\node$.

Given the connection between impurity and residual variance, if the tree model $\hat{f}$ considered in this section is grown using an impurity decrease stopping condition, we should expect the residual variance to be relatively similar over all leaves, so that 
$\sigma^2(\node)$ also does not vary too much over different nodes. 
Meanwhile (see \cref{subsec:motivating_constant_lambda_supp})
\begin{equation} \label{eq:motivating_constant_lambda}
    \paren*{\E\braces*{y~|~\node_L} - \E\braces*{y~|~\node_R}}^2 \approx \textnormal{diam}(\node)^2 \approx 2^{-2\textnormal{depth}(\node)/p}
\end{equation}
where $p$ is the dimension of the original feature space.
Since the maximum depth is typically $O(\log n)$, $\lambda_{opt}(\node)$ also does not vary too much across different nodes, and we do not lose too much by using a common value of $\lambda$ across all the univariate subproblems \eqref{eq:univariate_ridge} corresponding to different decision stump features.
 
A more naive (supervised) tree-based feature map is the one-hot encoding of an original feature vector obtained by treating the leaf index as a categorical variable.
We denote this using $\Xi$.
While $\Xi$ can be obtained from $\Psi$ via an invertible linear transformation, the two maps result in different kernels, and thus different ridge regression problems.
Indeed, the ridge regression solution with respect to $\Xi$ is equivalent to performing \leafbased~on the tree model.
The leaf indicator features are also orthogonal, so we may similarly decompose this ridge regression problem into independent univariate subproblems, one for each leaf.
However, in this case, the 
population regression vector $\bbeta_*^l$, for which $\bbeta_*^{lT}\Xi(\bx)$ gives the expected response over each leaf, has coordinates equal to the population expectation over the leaves: $\bbeta_*^l(\node) = \E\braces*{y~|~\node}$.
As such, 
the optimal regularization parameters for different leaf nodes could be very different, and we lose more by having to use a common value of $\lambda$.

In practice, sample splitting is rarely done, and the feature map depends on the responses.
Nonetheless, we believe that the heuristics detailed above continue to hold to a certain extent as shown in our experimental results.
\section{\methods improves predictive performance on real-world datasets}
\label{sec:results}

\subsection{Data overview}
\label{subsec:data}
In this section, we study the performance of \methods on a collection of classification and regression datasets selected as follows.
For classification, we consider a number of datasets used in the classic Random Forest paper ~\cite{breiman2001random,asuncion2007uci}, as well as two 
that are commonly used to evaluate rule-based models~\cite{pmlr-v97-wang19a}.
For regression, we consider all regression datasets used by \cite{breiman2001random} with at least 200 samples, as well as a variety of data-sets from the PMLB benchmark~\cite{romano2020pmlb} ranging from small to large sample sizes.
\cref{tab:datasets} displays the number of samples and features present in each dataset, with more details provided in \cref{sec:data_supp}.
In all cases, $\nicefrac 2 3$ of the data is used for training (hyperparameters are selected via 3-fold cross-validation on this set) and $\nicefrac 1 3$ of the data is used for testing.

\begin{table}[H]
    \footnotesize
    \centering
    \setlength{\tabcolsep}{4pt}
\begin{tabular}{llrr}
 \toprule
  &                                       Name &  Samples &  Features \\
 \midrule
 \parbox[c]{2mm}{\multirow{8}{*}{\rotatebox[origin=c]{90}{Classification}}}               &                    Heart &      270 &        15 \\
                                &Breast cancer &      277 &        17 \\
                                 &    Haberman &      306 &         3  \\
&Ionosphere \cite{sigillito1989classification} &      351 &        34  \\
 &             Diabetes \cite{smith1988using} &      768 &         8 \\
  &                              German credit &     1000 &        20  \\
   &       Juvenile \cite{osofsky1997effects} &     3640 &       286  \\
    &                              Recidivism &     6172 &        20 \\
              
\midrule
  \parbox[c]{0.5mm}{\multirow{8}{*}{\rotatebox[origin=c]{90}{Regression}}} &    Friedman1 \cite{friedman1991multivariate} &      200 &        10  \\
&     Friedman3 \cite{friedman1991multivariate} &      200 &         4 \\
&     Diabetes  \cite{efron2004least} &      442 &        10  \\
&     Geographical music &     1059 &       117 \\
&     Red wine &     1599 &        11 \\
&     Abalone \cite{nash1994population} &     4177 &         8  \\
&     Satellite image \cite{romano2020pmlb} &     6435 &        36  \\
&     CA housing \cite{pace1997sparse} &    20640 &         8  \\ 
 \bottomrule
 \end{tabular}

    \vspace{-5pt}
    \caption{Real-world datasets analyzed here for classification (top panel) and regression (bottom panel).}
    \label{tab:datasets}
\end{table}





\subsection{\methods improves prediction performance for commonly used tree methods}
\label{subsec:performance}

The prediction performance results for classification and regression are plotted in \cref{fig:performance_curves}A and \cref{fig:performance_curves}B respectively, with the number of leaves, as a proxy for the model complexity, plotted on the $x$-axis. 
We consider trees grown using four different techniques: CART, CART with cost-complexity pruning (CCP), C4.5, and GOSDT \cite{lin2020generalized}, a method that grows optimal trees in terms of the cost-complexity penalized misclassification loss.
To reduce clutter, we only display the classification results for CART and CART with CCP in \cref{fig:performance_curves}A/B and defer the results for C4.5 (\cref{fig:c45_classification_results}) and GOSDT~(\cref{subsec:gosdt}) to the appendix.

For each of the four tree-growing methods, we grow a tree to a fixed number of leaves $m$,\footnote{For CART (CCP), we grow the tree to maximum depth, and tune the regularization parameter to yield $m$ leaves.}
for several different choices of $m \in \{2,4,8,12,15,20,24,28,30,32\}$ (in practice, $m$ would be pre-specified by a user or selected via cross-validation).
For each tree, we compute its prediction performance before and after applying \method, where the regularization parameter for \methods is selected from the set $\lambda \in \{0.1,1.0,10.0,25.0,50.0,100.0\}$ via cross-validation.
Results for each experiment are averaged over 10 random data splits.
We observe that \methods (solid lines in \cref{fig:performance_curves}A,B) \emph{does not hurt prediction in any of our data sets, and often leads to substantial performance gains}.
For example, taking $m = 15$, we observe an average increase in relative predictive performance (measured by AUC) of $6.2\%,6.5\%,8\%$
for \methods applied to CART and CART with CCP, and C4.5 respectively for the classification data sets. For the regression data sets with $m = 15$, we observe an average relative increase in $R^{2}$ performance of $9.8\%,10.1\%$ for CART and CART with CCP respectively.

As expected, the improvements tend to be larger when $m$ increases and for smaller datasets (e.g. the top row of \cref{fig:performance_curves}A and \cref{fig:performance_curves}B), although for larger datasets we see substantial improvements using \methods once the number of leaves in the model is increased (\cref{subsec:deep_tree_supp_results}).

The fact that improvements hold for CART (CCP) shows that the effect of \methods is not entirely replicated by tree structure regularization, and instead, \emph{the two regularization methods can be used synergistically}.
Indeed, applying \methods can lead to the selection of a larger tree.
Since tree models are sometimes used for subgroup search, larger trees from \methods could allow for the discovery of otherwise undetected subgroups.

\cref{fig:bias_variance_linear_gaussian_uniform} shows a simulation result analyzing the bias-variance tradeoff for CART with and without \method.
Here, data is generated from a linear model with Gaussian noise added during training (see \cref{sec:sim_supp} for experimental details, and other simulations).
While predictive performance curves are often U-shaped because of the bias-variance tradeoff, those for \methods are monotonic since \methods is able to effectively reduce variance.
The optimal regularization parameter $\lambda$ decreases with the total number of leaves; this is corroborated by our calculations in \cref{sec:theory}.

\begin{figure}[htbp]
    \centering
    \includegraphics[width=0.8\columnwidth]{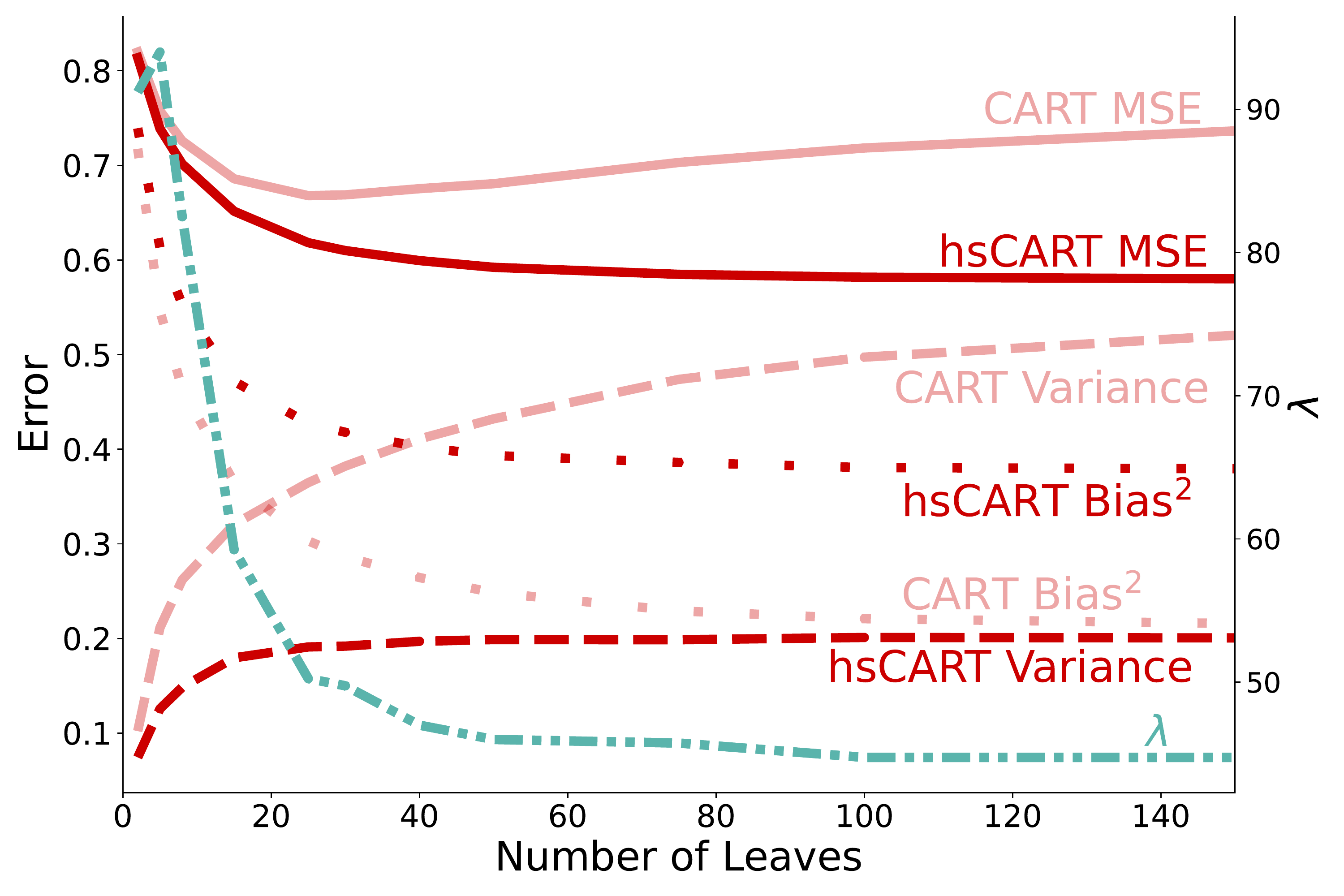}
    \vspace{-12pt}
    \caption{Test error for CART with \methods (hsCART) stays low as the number of leaves increases, whereas CART test error increases due to overfitting. Data is simulated from a linear model with Gaussian noise.} 
    \label{fig:bias_variance_linear_gaussian_uniform}
\end{figure}

\begin{figure*}[htbp]
    \vspace{-8pt}
    \aboverulesep = 0.0mm
    \belowrulesep = 0.0mm
    \centering
    \begin{tabular}{lc}
    \begin{minipage}{0.1cm}
    \rotatebox[origin=c]{90}{\large \textbf{(A)} Classification}
    \end{minipage}%
    & \raisebox{\dimexpr-.5\height-1em}{\includegraphics[width=0.91\textwidth]{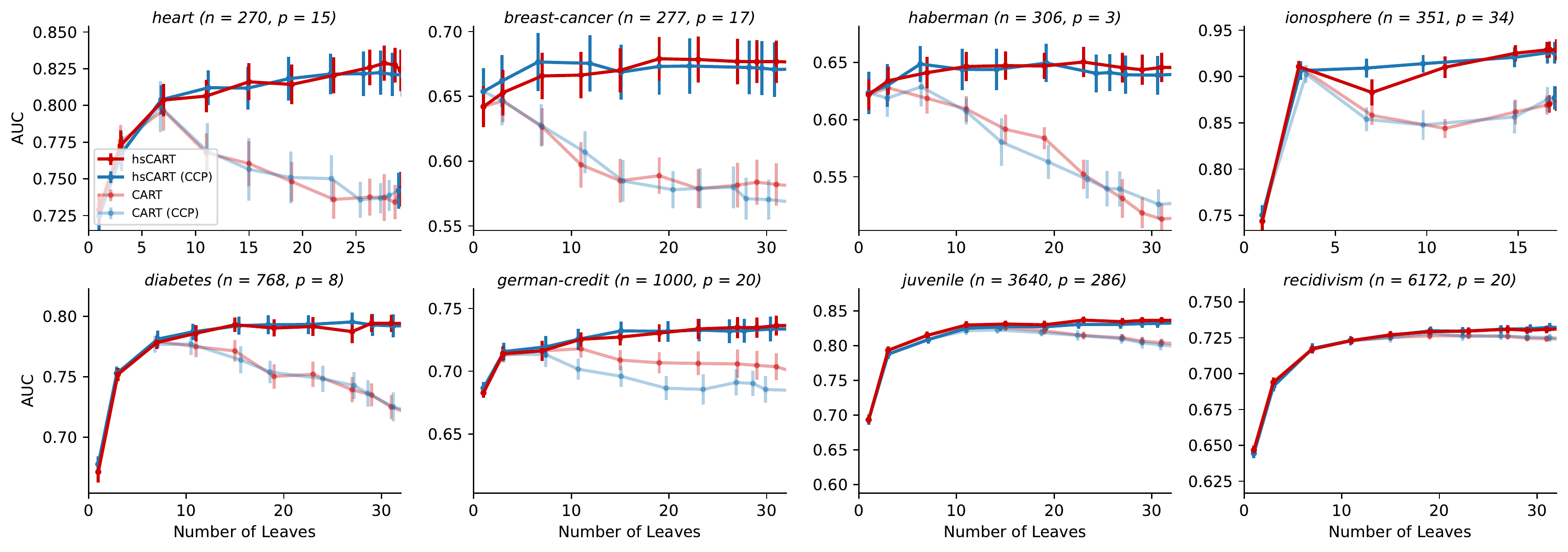}}\\
    \midrule
    \begin{minipage}{0.1cm}
    \rotatebox[origin=c]{90}{\large \textbf{(B)} Regression}
    \end{minipage}
    & \raisebox{\dimexpr-.5\height-1em}{\includegraphics[width=0.91\textwidth]{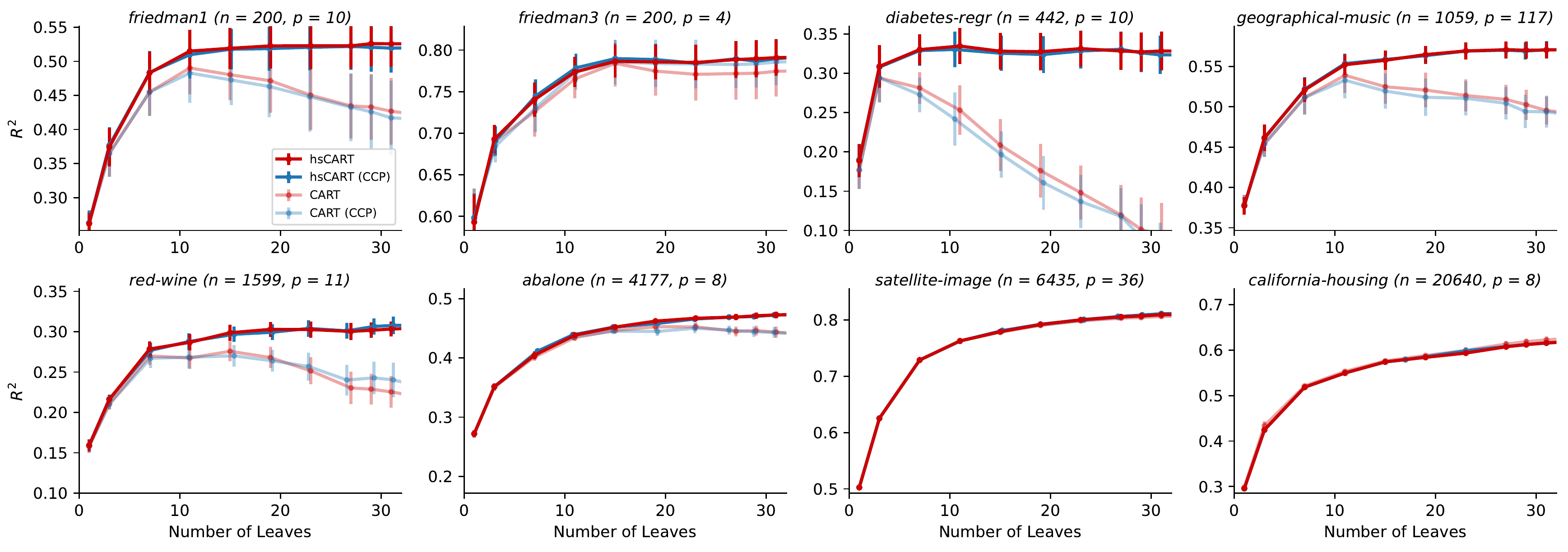}}\\
    \midrule
    \begin{minipage}{0.1cm}
    \rotatebox[origin=c]{90}{\large \textbf{(C)} LBS comparisons}
    \end{minipage}%
    & \raisebox{\dimexpr-.5\height-1em}{\includegraphics[width=0.91\textwidth]{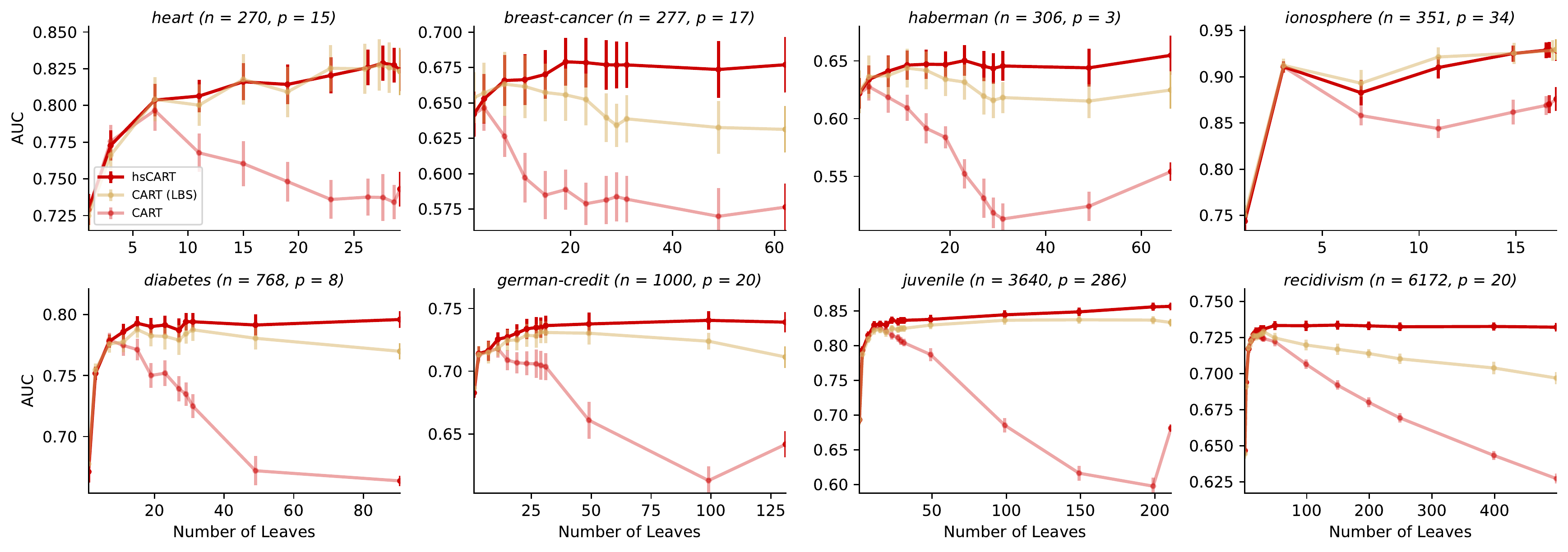}}\\
    \midrule
    \begin{minipage}{0.1cm}
    \rotatebox[origin=c]{90}{\large \textbf{(D)} RF Comparisons}
    \end{minipage}%
    & \raisebox{\dimexpr-.5\height-1em}{\includegraphics[width=0.91\textwidth]{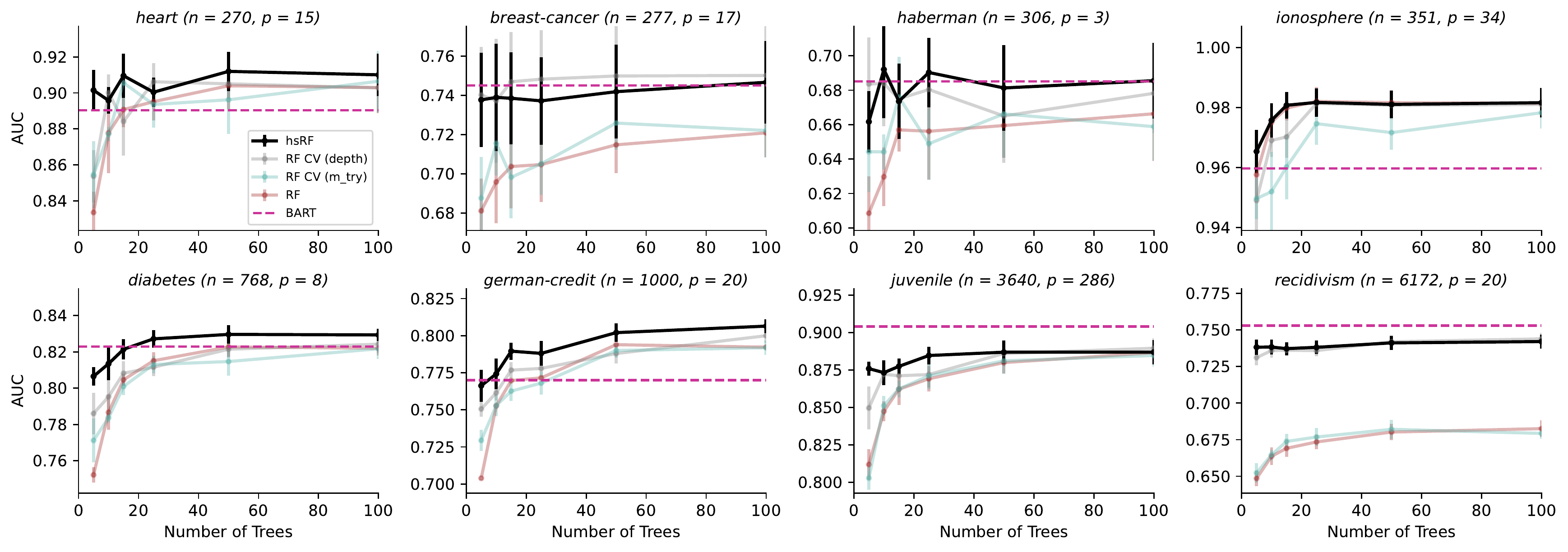}}
    \end{tabular}
    \vspace{-10pt}
    \caption{\methodlong~(solid lines) often improves predictive performance across various datasets, particularly for small datasets.
    \textbf{(A)} Top two rows show results for classification datasets (measured by AUC of the ROC curve) and \textbf{(B)} the next two rows show results for regression datasets (measured by $R^2$).
    \methods often significantly improves the performance over CART, CART with CCP, and \textbf{(C)} leaf-based shrinkage.
    \textbf{(D)} \methods even improves results for Random Forests as a function of the number of trees.
    Across all panels, errors bars show standard error of the mean computed over 10 random data splits. Note that the y-axis scales differ across plots.}
    \label{fig:performance_curves}
\end{figure*}

\subsection{\methods outperforms \leafbased}
\label{subsec:lbs_classification}

We next compare the performance of \methods to that of leaf-based shrinkage (\leafbased), which is used in XGBoost.
\cref{fig:performance_curves}C shows that hsCART tends to outperform CART (\leafbased), when repeating the same experiments as in \cref{fig:performance_curves}A); regression results are pushed to \cref{subsec:lbs_regression} due to space constraints.

\subsection{\methods improves prediction performance for RF}
\label{subsec:ensemble_results}

As mentioned earlier, trees in an RF are typically grown to purity without any constraints on depth or node size.
Nonetheless, \citet{mentch2019randomization} argues that the \texttt{mtry} parameter\footnote{This parameter is denoted \texttt{mtry} in \texttt{ranger} and \texttt{max\_features} in \texttt{scikit-learn}.}, which controls the degree of feature sub-selection, ``serves much the same purpose as the shrinkage penalty in explicitly regularized regression procedures like lasso and ridge regression.''
This parameter is typically set to a default value\footnote{Typically $\sqrt{p}$ for classification and $p/3$ for regression, where $p$ is the number of features.}, and is not tuned.
We compare the performance of regularizing RF via \methods against maximum tree depth and \texttt{mtry}, tuning the hyperparameter for each method via cross validation.
We repeat this for several different choices of $B$, the number of trees in the RF, and like before, average the results over 10 random data splits.

The results, displayed in \cref{fig:performance_curves}D
show that \emph{\methods significantly improves the prediction accuracy of RF} across the datasets we considered.
Moreover, \methods clearly outperforms the two RF-regularization methods (using \texttt{depth} and \texttt{mtry}) in all but one dataset (breast-cancer).
This is especially promising because \methods is also the fastest and easiest method to implement, as it does not require refitting the RF. 
Moreover, \emph{hsRF tends to achieve its maximum performance with fewer trees} than RF without regularization; as a consequence, RF with \methods is often able to achieve the same performance with an ensemble that is five times smaller, allowing us to achieve large savings in computational resources.

We also compare hsRF to the predictive performance of BART, and observe that hsRF and BART are comparable in terms of prediction performance.
However, hsRF is much faster to fit than BART (typically 10-15 times faster) and we can also apply \methods to BART, see \cref{sec:bart}.
\section{\methods improves RF interpretations by simplifying and stabilizing them}
\label{sec:interpretations}

In addition to improving predictive performance, \methods reduces variance and removes sampling artifacts, resulting in (i) simplified boundaries, (ii) stabilized feature importance scores, and (iii) making it easier to interpret interactions in the model.

\begin{figure}[H]
    \centering
         \includegraphics[width=0.7\columnwidth]{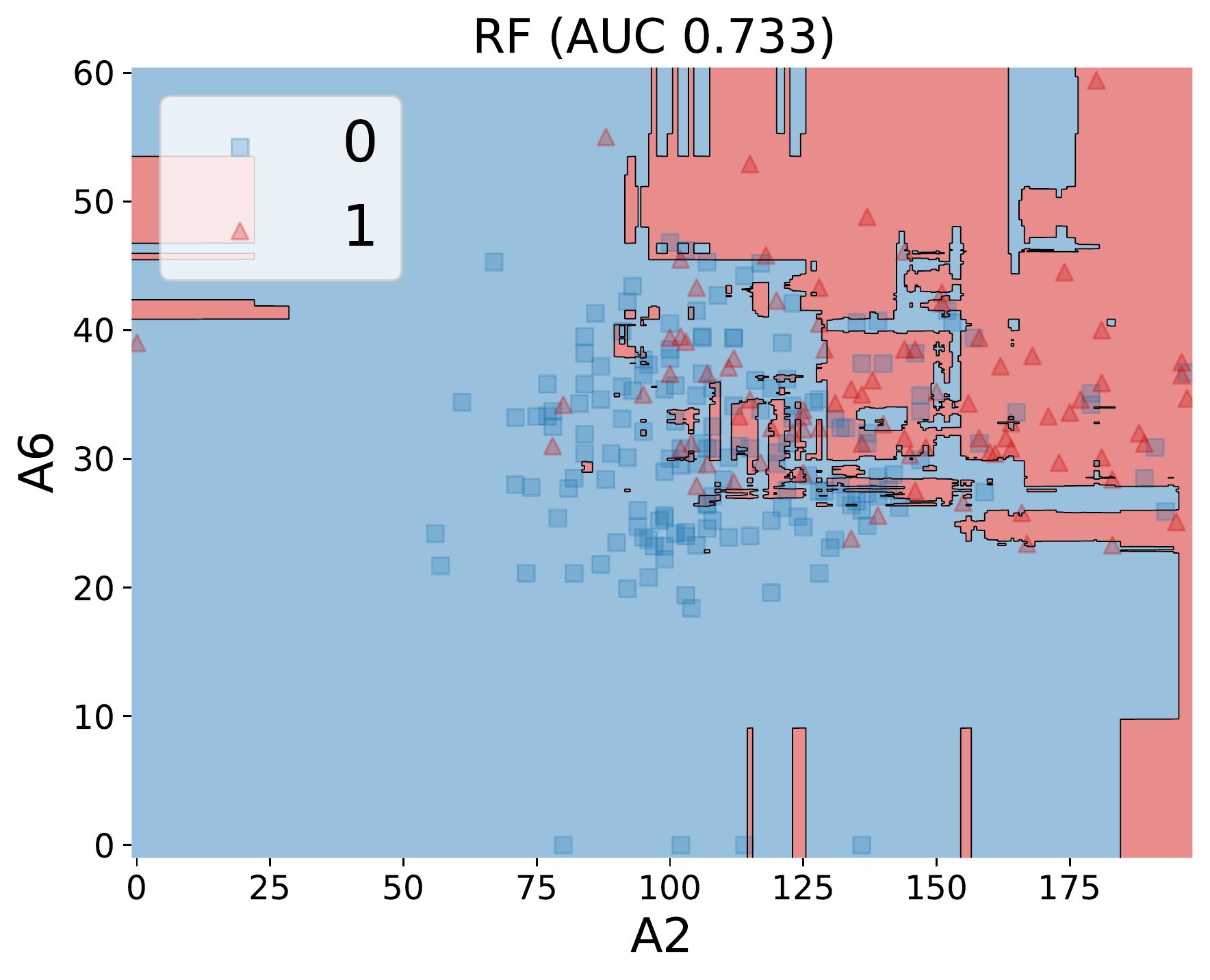} \includegraphics[width=0.7\columnwidth]{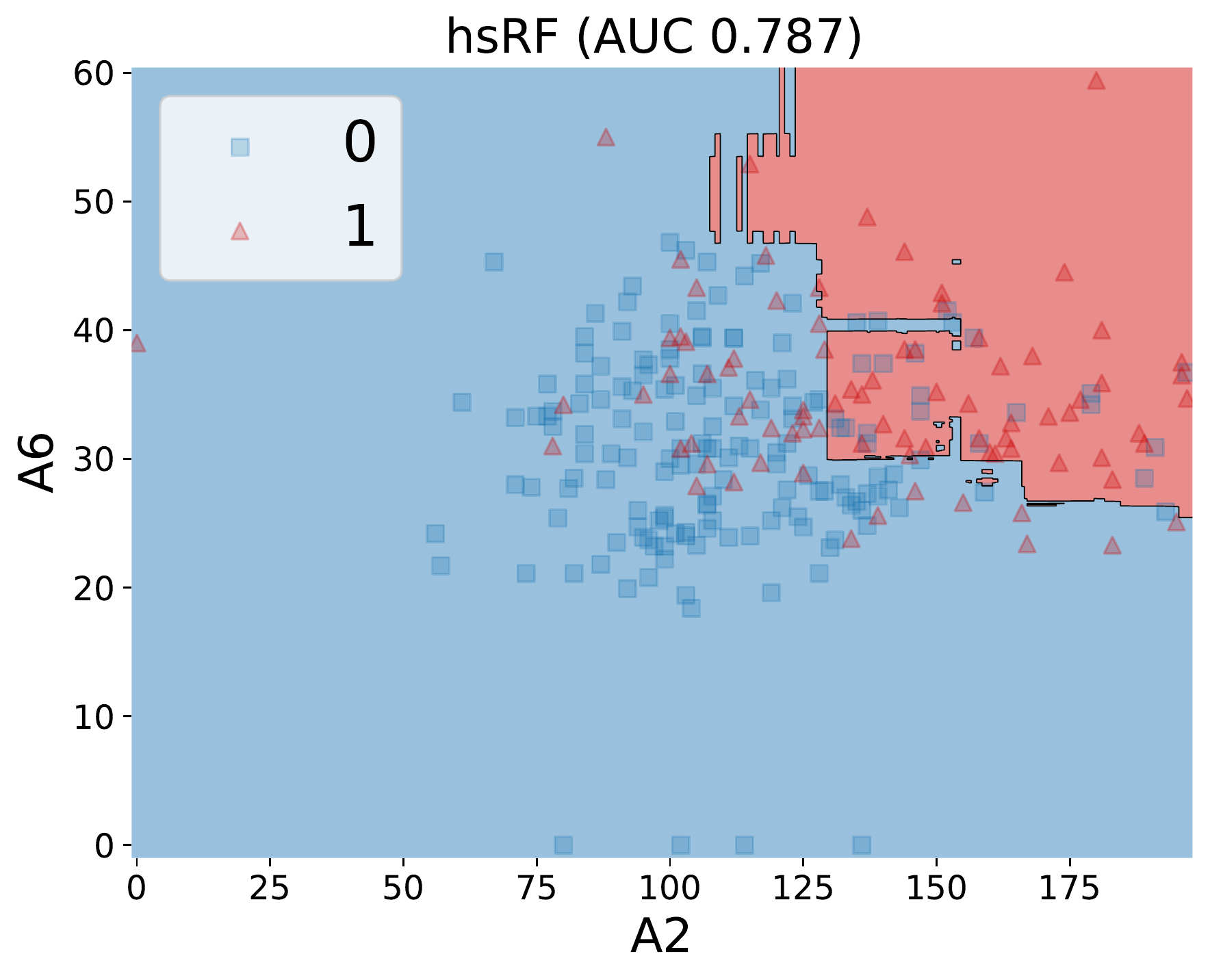}
    \vspace{-7pt}
    \caption{Comparison of decision boundary learned by RF vs hsRF on the Diabetes dataset, when fitted using only two features.
    \methods prevents overfitting by creating a smoother, simpler decision boundary, resulting in improved performance and interpretability.}
    \label{fig:boundary_diabetes}
\end{figure}

\cref{fig:boundary_diabetes} shows an example of simplification: smoothing decision boundaries.
On the diabetes dataset~\cite{smith1988using}, RF can achieve strong performance (AUC 0.733) even when fitted to only two features.
When \methods is applied to this RF, the performance increases (to an AUC of 0.787), but the decision boundary also becomes considerably smoother and less fragmented.
Since these two features are the only inputs to the model, these smooth boundaries enable a user to identify much clearer regions for high-risk predictions.
\cref{subsec:decision_boundary_supp} shows boundary maps for all 8 classification datasets, showing that \methods consistently makes boundaries smoother (these maps visualize the boundary with respect to the two features with the highest mean-decrease-impurity (MDI) score).

In models with many features, post-hoc interpretations, such as SHAP scores~\cite{lundberg2017unified}, can help a practitioner understand how a model makes its predictions.
\cref{fig:shap_variability_breast} shows that \methods improves the stability of SHAP scores.
Stability is measured using the variance of SHAP values when models are fit to 100 random train-test splits of the breast-cancer dataset.
This makes the model interpretations less sensitive to minor data perturbations and thus more trustworthy.
Moreover, these improvements in stability persist even for datasets such as Heart, Diabetes, and Ionosphere, for which \methods does not greatly improve prediction performance (see SHAP stability plots for all datasets in \cref{fig:shap_variability_small} and \cref{fig:shap_variability_big}).
Hence, \methods can improve the stability and interpretability of RF, even when it does not improve its predictive performance.

\begin{figure}[H]
    \centering \includegraphics[width=0.8\columnwidth]{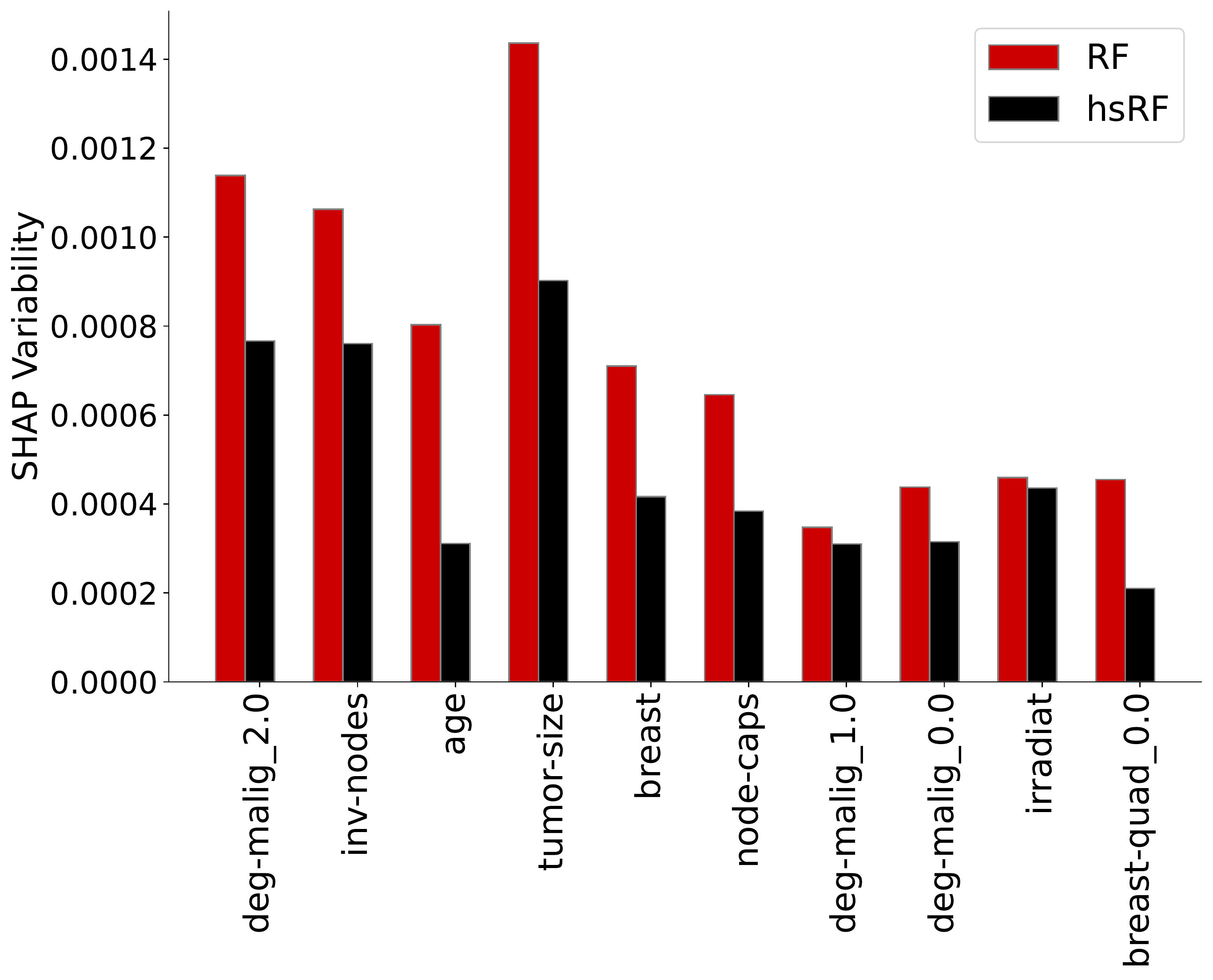}
    \caption{Comparison of SHAP plots learnt by Random Forests on the breast-cancer dataset before / after applying \method.
    \methods displays lower variability across different data perturbations (without dropping in predictive performance), indicating enhanced stability.}
    \label{fig:shap_variability_breast}
\end{figure}

\cref{fig:breast_cancer_shap_plot} shows an example investigating the SHAP scores across the breast-cancer dataset for a trained RF.
After applying \method, the SHAP values for each feature often cluster much more nicely.
Each cluster corresponds to a group of samples for which a feature contributes a similar amount to the predicted response, regardless of the value of other features.
As a result, each cluster can be interpreted without taking into account feature interactions, simplifying the user's interpretation.
Since \methods improves the model's predictive performance, the clustered SHAP scores suggest that \methods improves performance by regularizing some unnecessary interactions in the model, and makes the fitted function closer to being additive, which allows for simpler interpretations. \cref{subsec:shap_plots} shows the SHAP plots for all 8 classification datasets, showing that \methods consistently leads to more clustered SHAP values. 



\begin{figure*}[!b]
    \centering
    \includegraphics[width =0.7\textwidth]{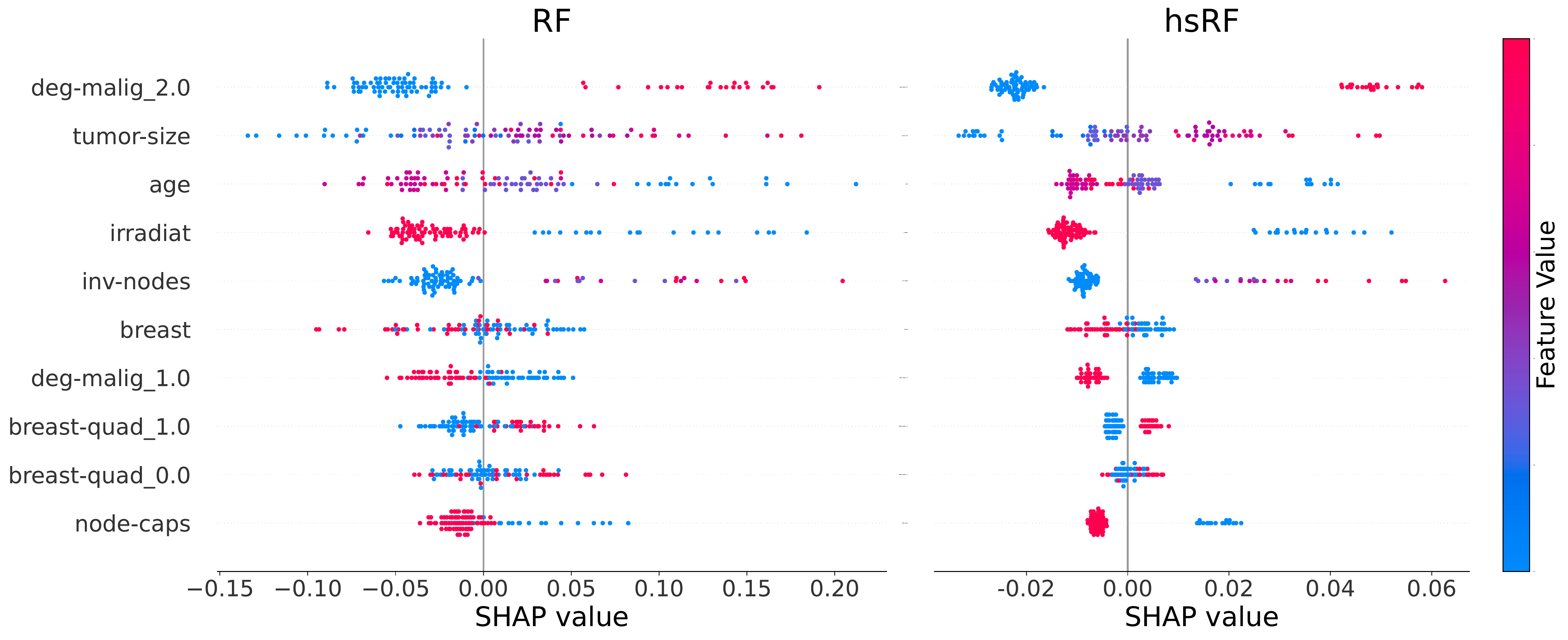}
    \vspace{-12pt}
    \caption{Comparison of SHAP values for RF on the breast-cancer dataset before/after applying \methods with features arranged by importance.
    Each point represents a single sample in the dataset.
    \methods leads to more clustered SHAP values for each feature, reflecting less heterogentity in the SHAP values of each feature.   }
    \label{fig:breast_cancer_shap_plot}
\end{figure*}

\section{Discussion}
\label{sec:discussion}

\methods is a fast yet powerful regularization procedure that can be applied to any tree-based model without changing its structure.
In our experiments, \methods never hurts prediction performance, and often leads to substantial gains in both predictive performance and interpretability. \methods is partly motivated by \citet{tan2021cautionary}'s non-minimax-optimal generalization lower bounds for decision trees that predict using average responses over each leaf, pointing to a possible limitation of averaging.\methods allows us to break this inferential barrier by pooling information from multiple leaves.

The work here naturally suggests many exciting future directions regarding the regularization of trees and RFs.
First, replacing ridge regression with more sophisticated linear methods such as lasso or elastic net can result in other promising regularization methods.
In addition, \methods could improve other structured rule-based models, such as rule lists and tree sums~\cite{tan2022fast}.

Meanwhile, we have only scratched the surface of the relationships between regularization, robustness, and interpretability.
Indeed, the connection between \methods and ridge regression suggests the relevance of \citet{donhauser2021interpolation}'s observations that ridge regression helps with robust generalization.
Furthermore, we have seen that the simpler and smoother decision boundaries resulting from \methods are more likely to generalize well.
Hence, we conjecture that \methods could improve the predictive performance of RF with respect to covariate shift.
Moreover, the improved clustering and stability of SHAP scores after applying \methods suggest that regularization via \methods could improve the identification of important features, even when using alternative methods such as MDI.
\section{Acknowledgements}
We gratefully acknowledge partial support from 
NSF TRIPODS Grant 1740855, DMS-1613002, 1953191, 2015341, IIS 1741340, ONR grant N00014-17-1-2176, the Center for Science of Information (CSoI), an NSF Science and Technology Center, under grant
agreement CCF-0939370, NSF grant 2023505 on
Collaborative Research: Foundations of Data Science Institute (FODSI),
the NSF and the Simons Foundation for the Collaboration on the Theoretical Foundations of Deep Learning through awards DMS-2031883 and 814639, and a grant from the Weill Neurohub.

{
    \small
    \bibliographystyle{icml2021}
    \bibliography{refs}
}
\appendix
\setcounter{figure}{0}
\setcounter{section}{0}

\renewcommand{\thetable}{S\arabic{table}}
\renewcommand{\thefigure}{S\arabic{figure}}
\renewcommand{\thesection}{S\arabic{section}} 

\newpage
\onecolumn

\begin{center}
    \Huge
    Supplement 
\end{center}
\label{sec:supp}

\section{Data details}
\label{sec:data_supp}

\begin{table}[htbp]
    \footnotesize
    \centering
    \begin{tabular}{lrrrrr}
\toprule
                                         Name &  Samples &  Features &  Class 0 &  Class 1 &  Majority class \% \\
\midrule
                                        Heart &      270 &        15 &      150 &      120 &              55.6 \\
                                Breast cancer &      277 &        17 &      196 &       81 &              70.8 \\
                                     Haberman &      306 &         3 &       81 &      225 &              73.5 \\
Ionosphere \cite{sigillito1989classification} &      351 &        34 &      126 &      225 &              64.1 \\
               Diabetes \cite{smith1988using} &      768 &         8 &      500 &      268 &              65.1 \\
                                German credit &     1000 &        20 &      300 &      700 &              70.0 \\
           Juvenile \cite{osofsky1997effects} &     3640 &       286 &     3153 &      487 &              86.6 \\
                                   Recidivism &     6172 &        20 &     3182 &     2990 &              51.6 \\
\bottomrule
\end{tabular}

    \caption{Classification datasets (extended).}
    \label{tab:data_classification_full}
\end{table}

\begin{table}[htbp]
    \footnotesize
    \centering
    \begin{tabular}{lrrrrrr}
\toprule
                                     Name &  Samples &  Features &  Mean &   Std &  Min &    Max \\
\midrule
    Friedman1 \cite{friedman1991multivariate} &      200 &        10 &  14.3 &   4.6 &  2.1 &   25.2 \\
    Friedman3 \cite{friedman1991multivariate} &      200 &         4 &   1.3 &   0.4 &  0.0 &    1.6 \\
    Diabetes  \cite{efron2004least} &      442 &        10 & 152.1 &  77.0 & 25.0 &  346.0 \\
    Geographical music &     1059 &       117 &   0.0 &   1.0 & -1.5 &    5.9 \\
    Red wine &     1599 &        11 &   5.6 &   0.8 &  3.0 &    8.0 \\
    Abalone \cite{nash1994population} &     4177 &         8 &   9.9 &   3.2 &  1.0 &   29.0 \\
    Satellite image \cite{romano2020pmlb} &     6435 &        36 &   3.7 &   2.2 &  1.0 &    7.0 \\
    California housing \cite{pace1997sparse} &    20640 &         8 &   2.1 &   1.2 &  0.1 &    5.0 \\

\bottomrule
\end{tabular}
    \caption{Regression datasets (extended).}
    \label{tab:data_regr_full}
\end{table}
\section{Additional simulations}
\label{sec:sim_supp}
In this section, we provide experimental details for our bias-variance trade-off simulation in \cref{fig:bias_variance_linear_gaussian_uniform} as well as other simulations settings that we display below. Note that we reproduce \cref{fig:bias_variance_linear_gaussian_uniform} in \cref{fig:linear_model_bias_variance} for ease of the reader. \\
\textbf{Experimental Design:} For \cref{fig:linear_model_bias_variance}, we simulate data via data via a linear model $y = \sum_{i=1}^{10}x_{i} + \epsilon$ with $\bx \sim \text{Unif}[0,1]^{50}$ and $\epsilon$ being drawn from a gaussian and laplacian distribution for the left and right panel respectively with $\sigma^2 = 0.01$. In \cref{fig:polynomial_model_bias_variance}, we simulate responses from a linear model with pairwise interactions  $y = \sum_{i=1}^{10}x_{i} + x_{1}x_{2} + x_{5}x_{6} + x_{11}x_{12} + \epsilon$. We used a training and test set of 500 samples, and computed the bias and variance over 100 independent draws of the training set. In all of our experiments, we fit CART and applied \methods post-hoc while varying the number of leaves. The regularization parameter $\lambda$ was chosen on the training set via 3-fold cross-validation. 
\begin{figure}[H]
    \centering 
    \begin{tabular}{cc}
         \includegraphics[width=0.4\textwidth]{figs/shrinkage_bias_variance_linear_gaussian_uniform.pdf} & \includegraphics[width=0.4\textwidth]{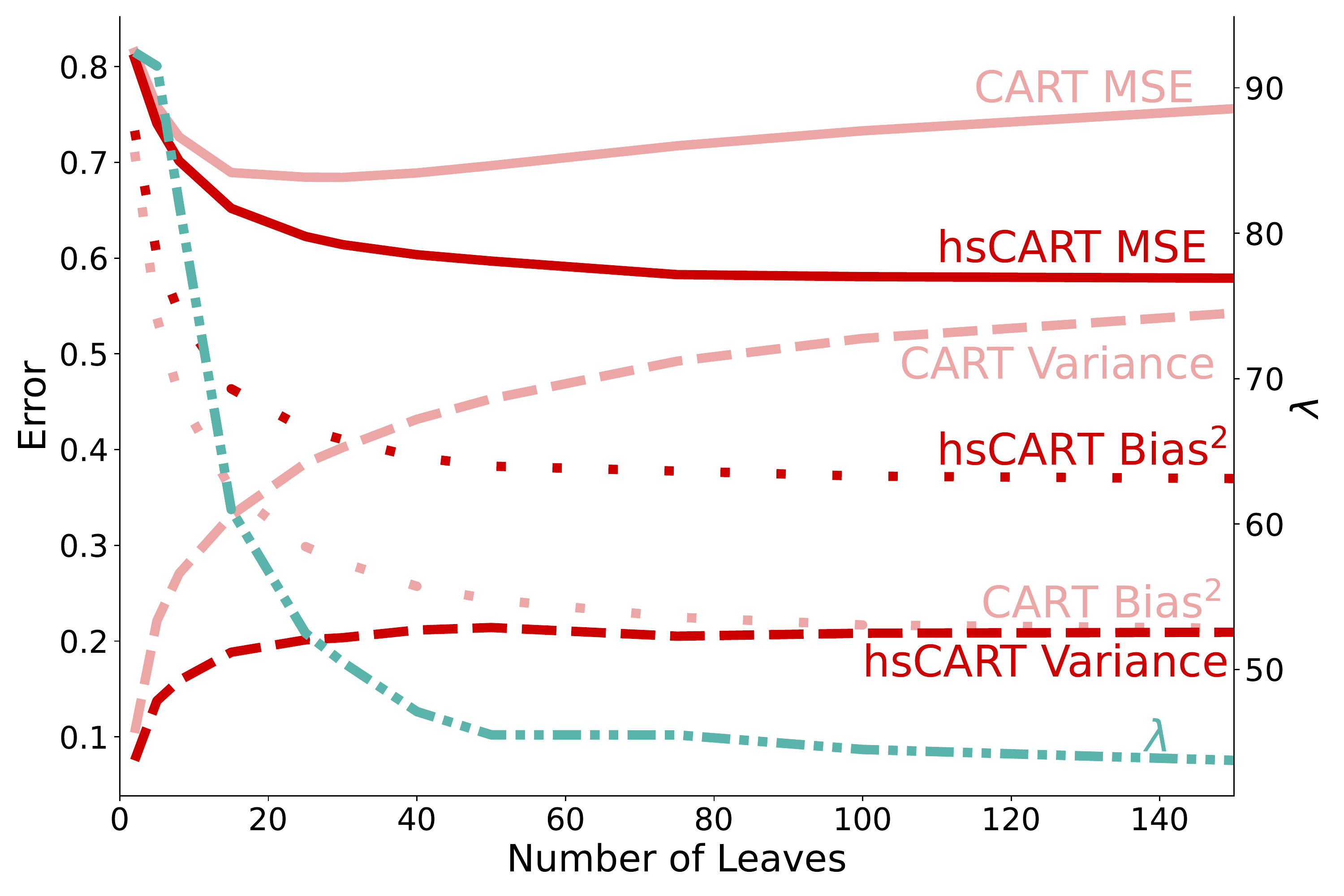}
    \end{tabular}
    \caption{CART (HTS) test error rate for a linear model with gaussian (left)/laplacian (right) noise stays low as the number of leaves increases, whereas CART test error displays U-shaped bias-variance tradeoff as model complexity increases. }
    \label{fig:linear_model_bias_variance}
\end{figure}
\begin{figure}[H]
    \centering
    \begin{tabular}{cc}
         \includegraphics[width=0.4\textwidth]{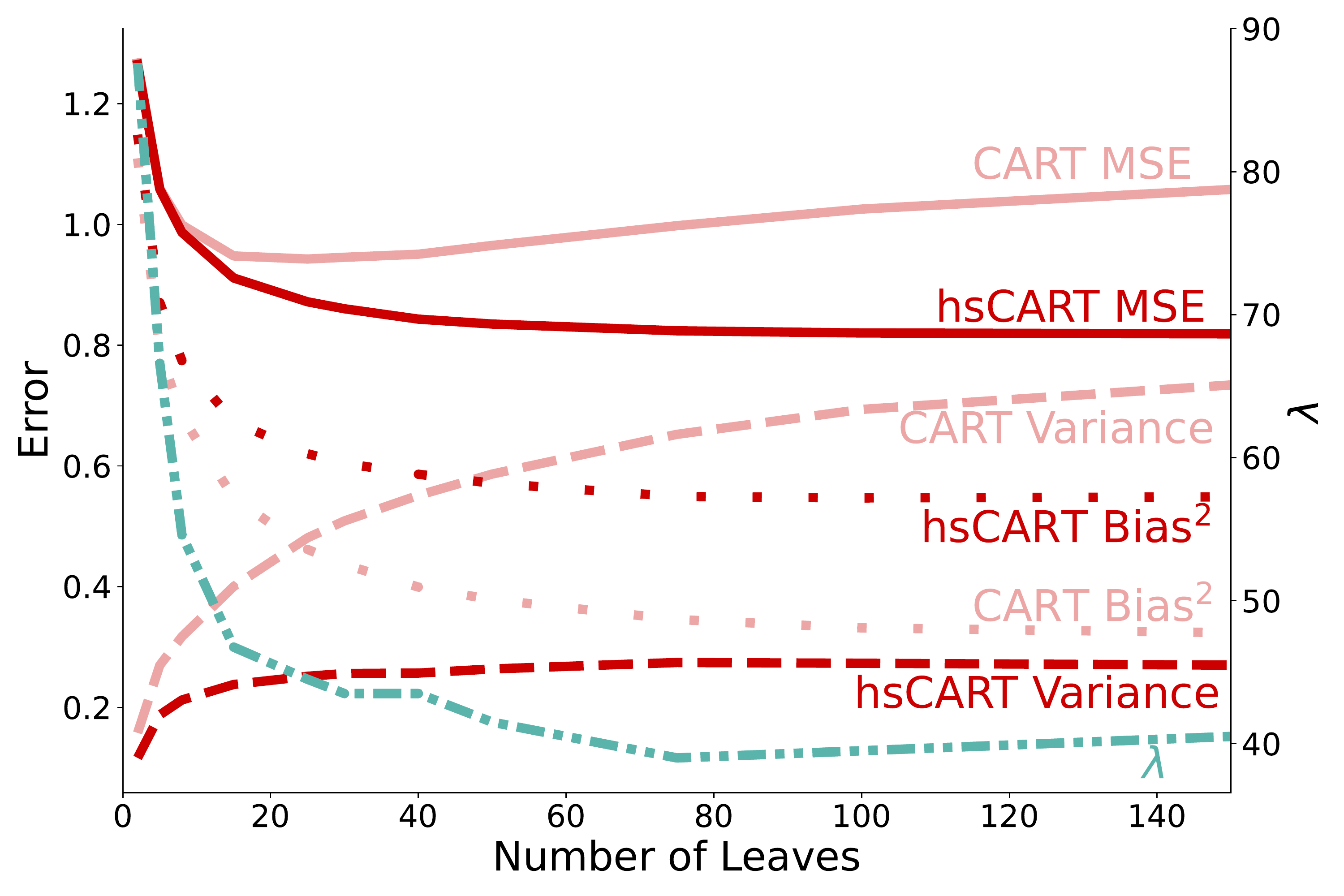} & \includegraphics[width=0.4\textwidth]{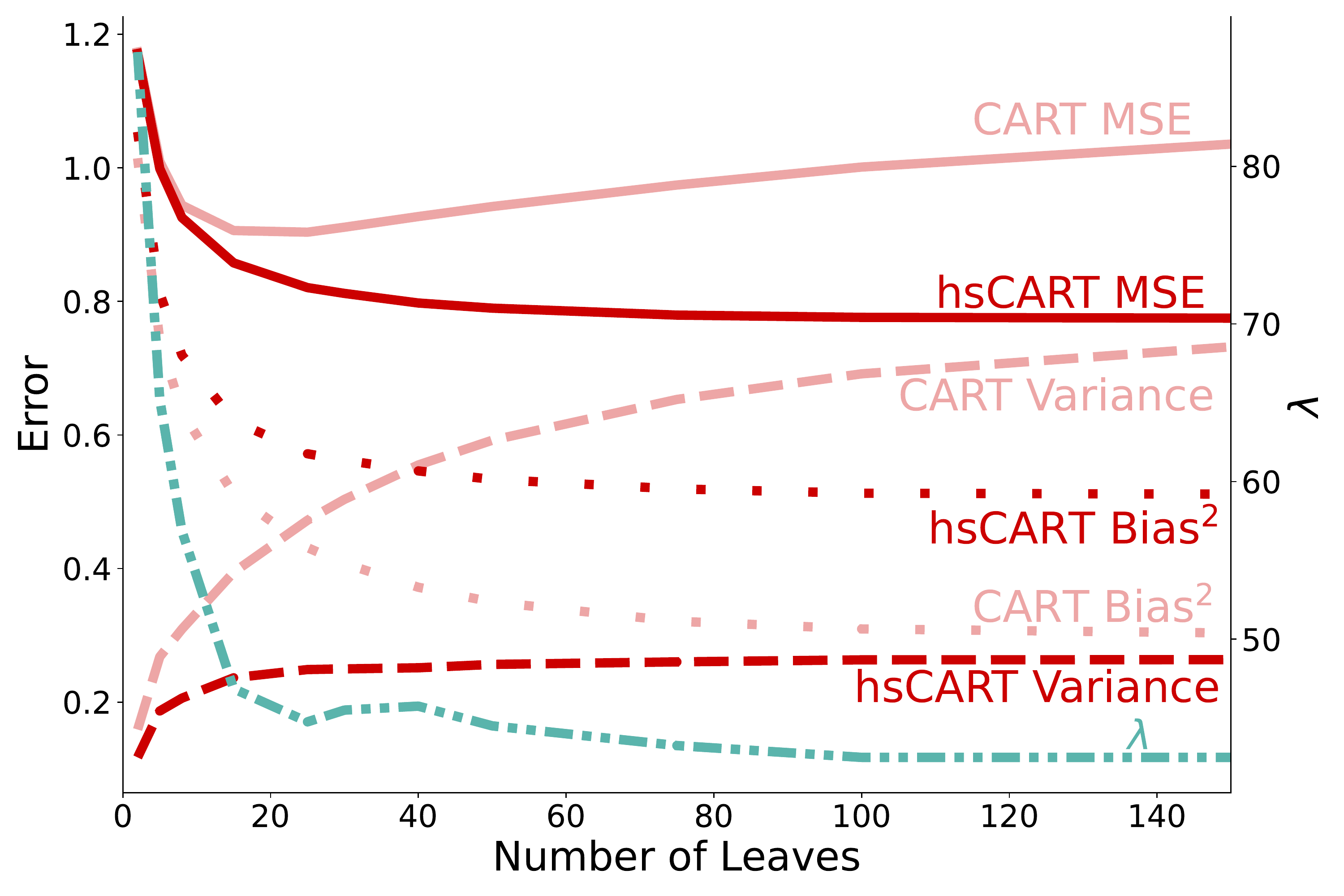}
    \end{tabular}
    \caption{CART (HTS) test error rate for a linear model with pairwise interactions and gaussian (left)/laplacian (right) noise stays low as the number of leaves increases, whereas CART test error displays U-shaped bias-variance tradeoff as model complexity increases.}
    \label{fig:polynomial_model_bias_variance}
\end{figure}

\section{Experimental results}
\label{sec:experimental_supp}
\subsection{C4.5 classification performance improves after adding \methods}
In this section, we display results on the classification data sets \cref{tab:datasets} for C4.5 before/after applying \methods. The experimental details are provided in \cref{subsec:performance}
\begin{figure}[H]
    \centering
    \includegraphics[width=1.0\textwidth]{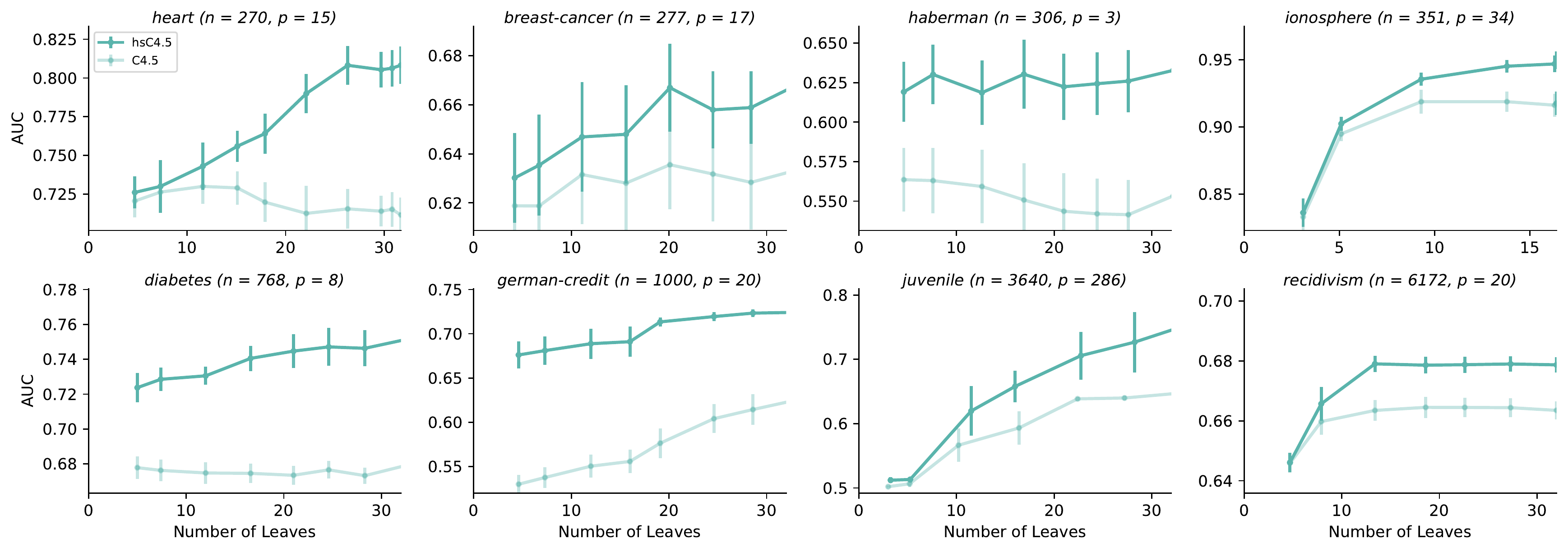}
    \caption{Classification results for C4.5. \methods significantly improves the performance  even with very few leaves. Errors bars show standard error of the mean computed over 10 random data splits}
    \label{fig:c45_classification_results}
\end{figure}

\subsection{GOSDT classification performance after adding \methods}
\label{subsec:gosdt}
Here, we compare \methods to GOSDT.
We use the best tree found by GOSDT within one hour, since GOSDT fails to complete for many datasets within a reasonable timeframe (24 hours).
To improve performance, we use a single pre-processing step of selecting the 5 most important features for each dataset (using RF feature importance).
This step ensures that the resulting tree, fitted within a time limit of one hour, is not extremely shallow (number of leaves $\leq 3$) and that GOSDT does not fail to due to a memory error.

For GOSDT, we are unable to tune $m$ exactly, and instead tune its cost-complexity regularization parameter.
\ref{fig:gosdt} shows the results, where we omit all the datasets for which all the resulting trees are extremely shallow. 
We note that GOSDT only outputs the root node for many data sets with regularization parameter larger than 0. 

\begin{figure}[H]
    \centering
    \includegraphics[width=0.6\textwidth]{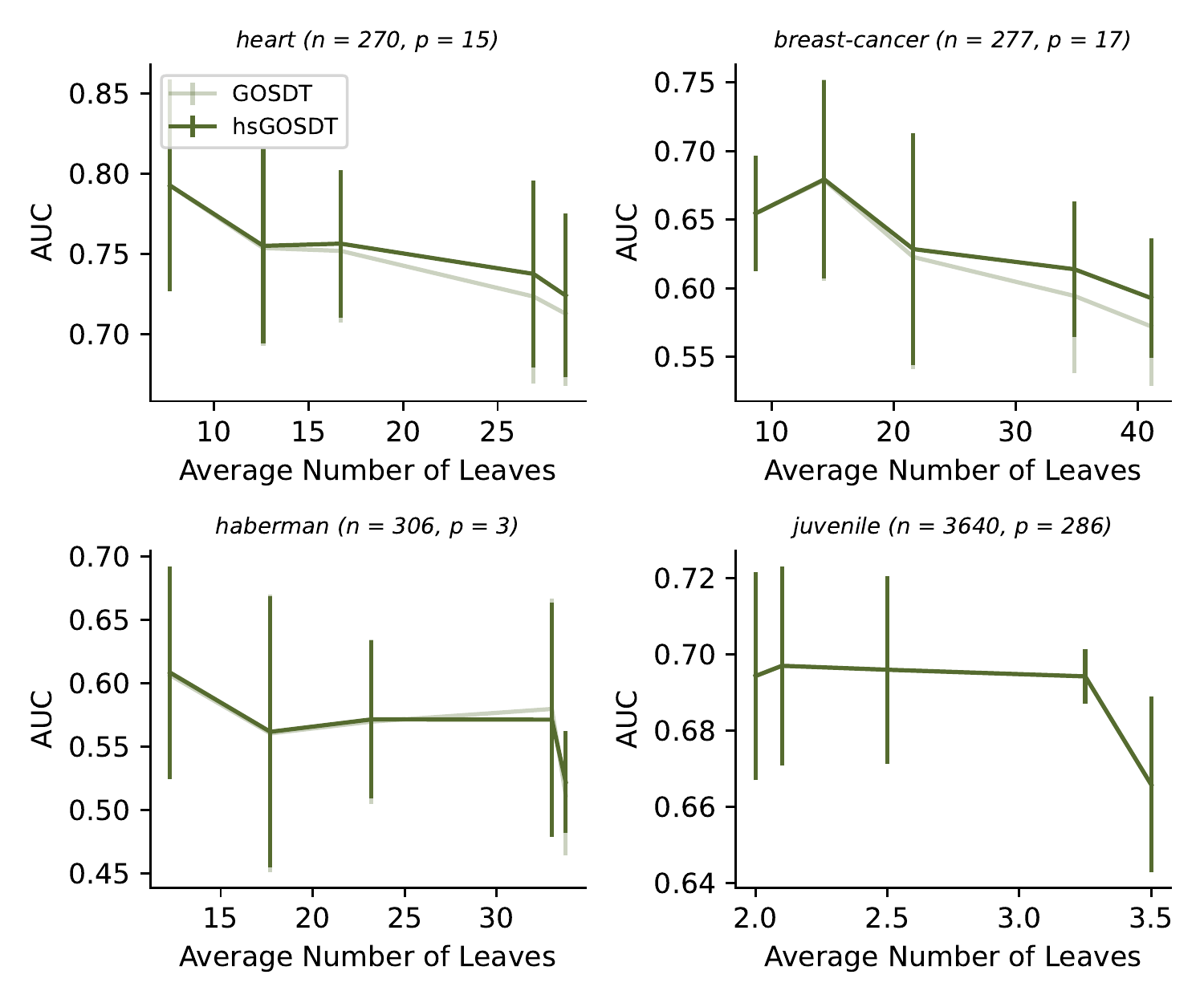}
    \caption{$\method$ performance across selected datasets.
    \methods is as good as and sometimes slightly improves GOSDT.}
    \label{fig:gosdt}
\end{figure}

\subsection{Leaf Based Results for Regression Datasets}
\label{subsec:lbs_regression}
We compare the performance of XG-Boost's LBS to \methods on the regression datasets in \cref{tab:datasets} with experimental details provided in \cref{subsec:lbs_classification}
\begin{figure}[H]
\includegraphics[width=1.0\textwidth]{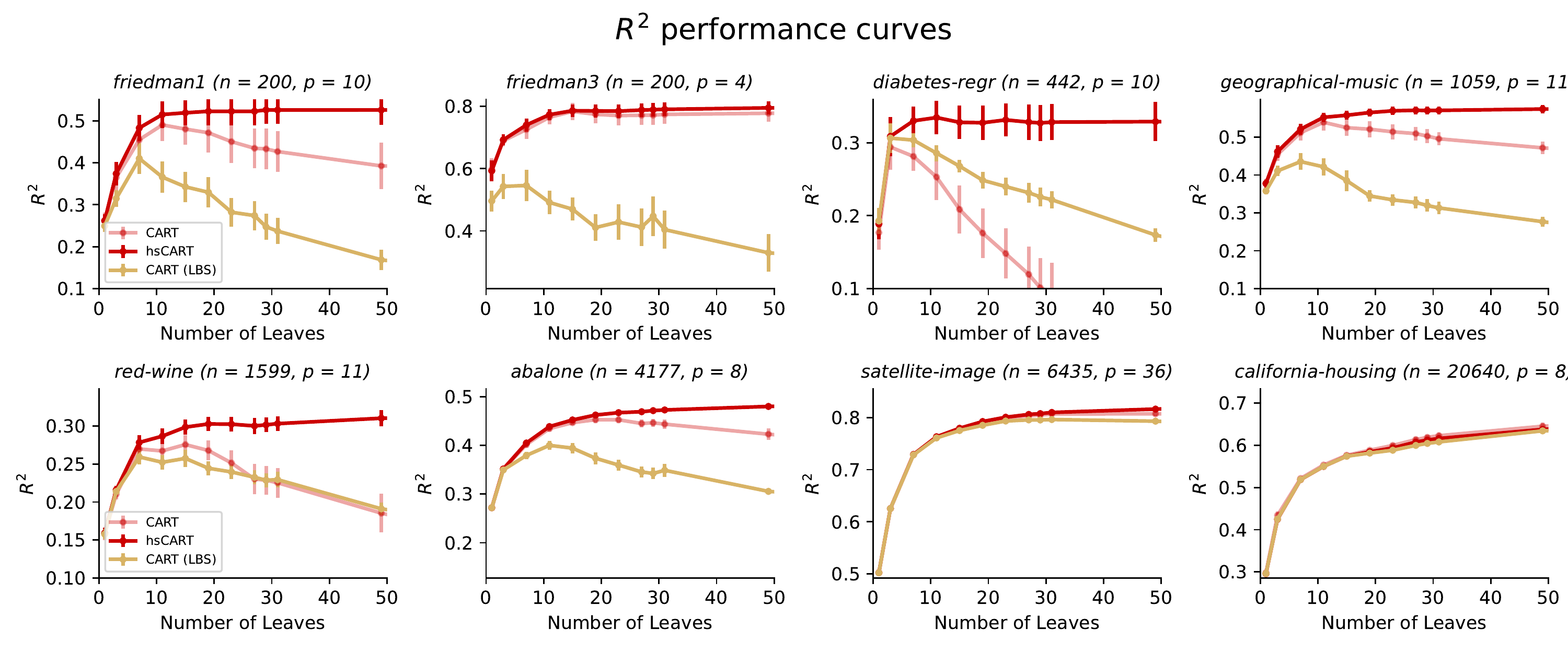}{}
    \caption{\methods performs better than LBS for regression datasets in \cref{tab:datasets}. Error bars show the standard error of the mean computed over 10 random data splits. }
    \label{fig:regression_lbs}
\end{figure}

\subsection{Experimental Results for Deeper Trees}
\label{subsec:deep_tree_supp_results}
In this section, we investigate the performance of decision tree algorithms: (i) CART, (ii) C4.5, (iii) CART with CCP  before/after applying \methods while growing deeper trees for all the data sets in \cref{tab:datasets} and \cref{tab:datasets}. We observe that the performance of the baselines drop dramatically while the performance of \methods continuously improves as the ratio of number of leaves to samples increases. This is explained by our of connection of \methods with ridge regression discussed on a supervised basis in \ref{sec:theory}.

\begin{figure}[H]
    \vspace{-8pt}
    \aboverulesep = 0.0mm
    \belowrulesep = 0.0mm
    \centering
    \begin{tabular}{lc}
    \begin{minipage}{0.1cm}
    \rotatebox[origin=c]{90}{\large \textbf{(A)} Classification}
    \end{minipage}%
    & \raisebox{\dimexpr-.5\height-1em}{\includegraphics[width=0.91\textwidth]{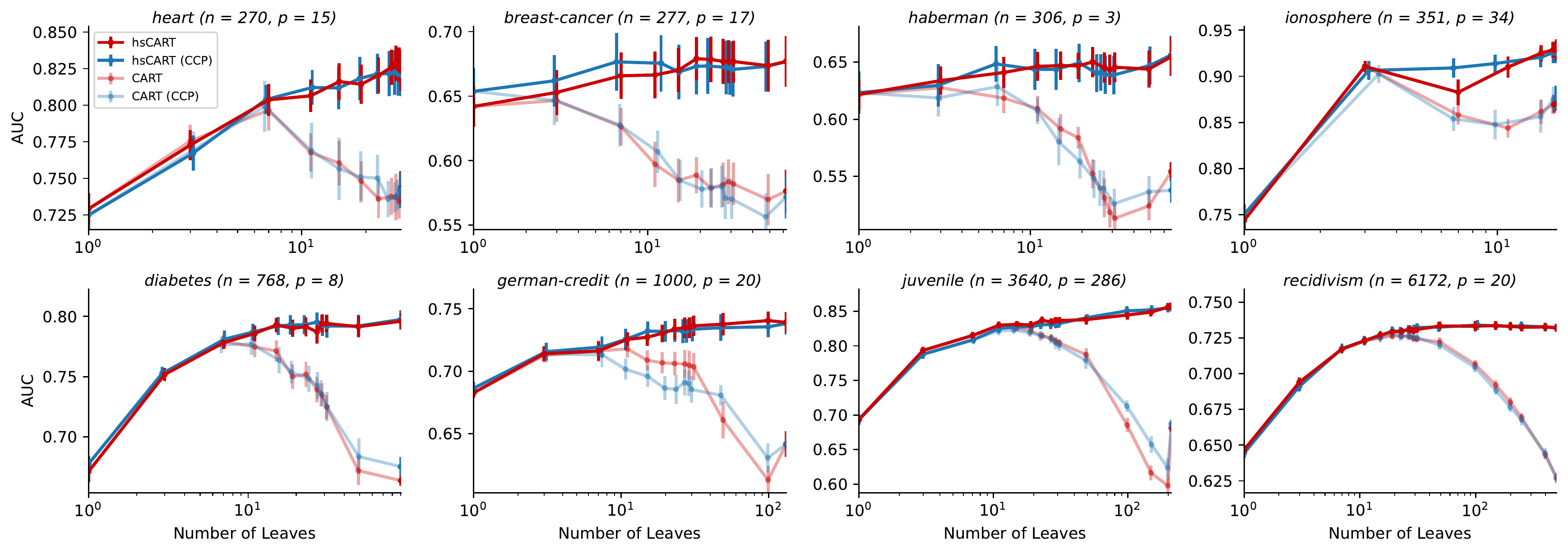}}\\
    \midrule
    \begin{minipage}{0.1cm}
    \rotatebox[origin=c]{90}{\large \textbf{(B)} Regression}
    \end{minipage}
    & \raisebox{\dimexpr-.5\height-1em}{\includegraphics[width=0.91\textwidth]{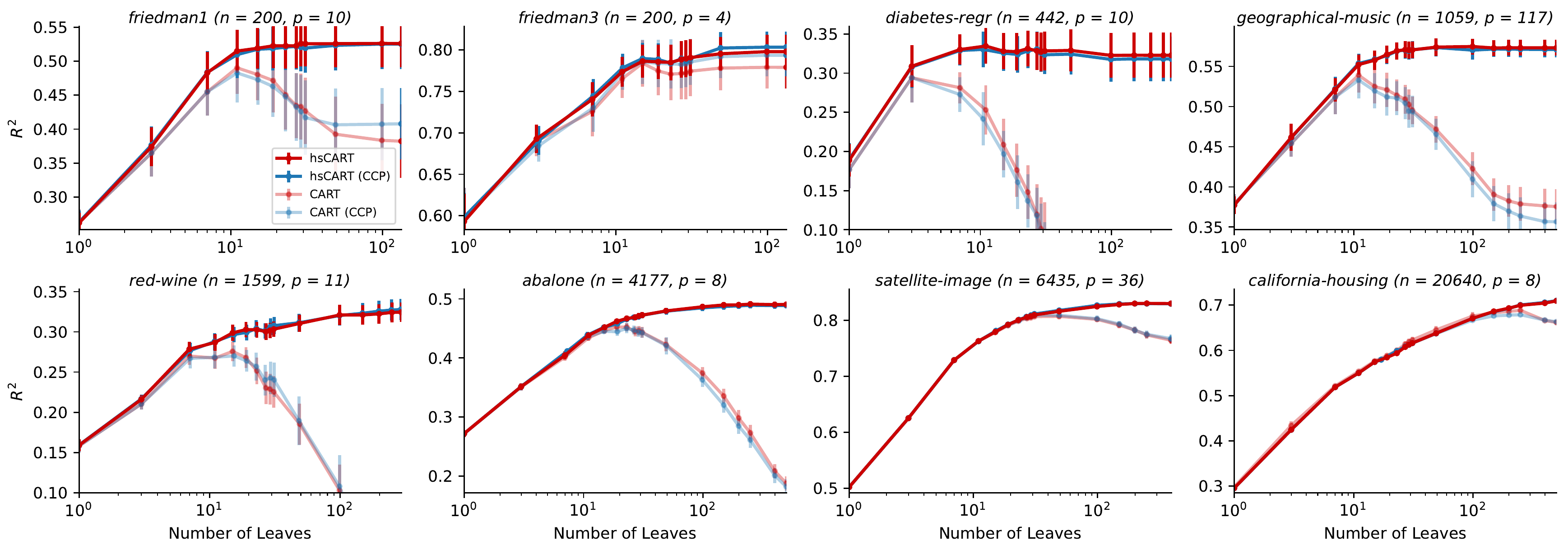}}\\
    \end{tabular}
    \vspace{-5pt}
    \caption{\methods is able to improve \textbf{(A)} Classification and \textbf{(B)} regression predictive performance for both baselines: CART and CART with CCP for larger data-sets when using deeper trees by preventing overfitting.
    As shown by our connection to Ridge regression in \cref{sec:theory},  \methods is beneficial when the ratio of number of leaves to samples is large. Error bars show standard error of the mean computer over 10 random data splits.}
    \label{fig:deep_performance}
\end{figure}

\subsection{BART}\label{sec:bart}
BART is a Bayesian ``sum of trees'' model, in particular the tree structure and leaf values are considered to be random variables, following a prior distribution. Conditioned on these random variables and the query point $\bx$, the response variable $y$ is assumed to follow a Gaussian distribution. The posterior distribution of the tree structures and leaf values is used for inference, where samples are generated via backfitting MCMC algorithm. Every single sample is a sum of trees function, and the final prediction is typically the average prediction over many such samples.

The tree growing algorithm in BART is fundamentally different from CART, in that it considers the mean response function, to be a random variable. In BART, the tree functions used for inference are samples from a posterior distribution, whereas in CART a tree is learned to minimize a certain objective. Another key difference is that BART shrinks the predictions to the grand mean, by centering the response variable and selecting a Gaussian prior with mean 0 for the leaf predictions. Therefore BART provides an interesting study case, both in terms a different tree growing method and a an alternative shrinkage method.

Our experimental setup is as follows:

First, a BART model is fitted using a single MCMC chain and otherwise default parameters using \textit{bartpy} package. Then, to control for model complexity (in terms of the number of leaves), a single sample from the posterior distribution is obtained, from which the BART predicted values and tree structure/s are extracted. Using the tree structure/s we impute leaf values and then apply \method, to obtain the  hsBART prediction.

\begin{figure}[H]
    \vspace{-8pt}
    \aboverulesep = 0.0mm
    \belowrulesep = 0.0mm
    \centering
    \begin{tabular}{lc}
    \begin{minipage}{0.1cm}
    \rotatebox[origin=c]{90}{\large \textbf{(A)} Classification}
    \end{minipage}%
    & \raisebox{\dimexpr-.5\height-1em}{\includegraphics[width=0.91\textwidth]{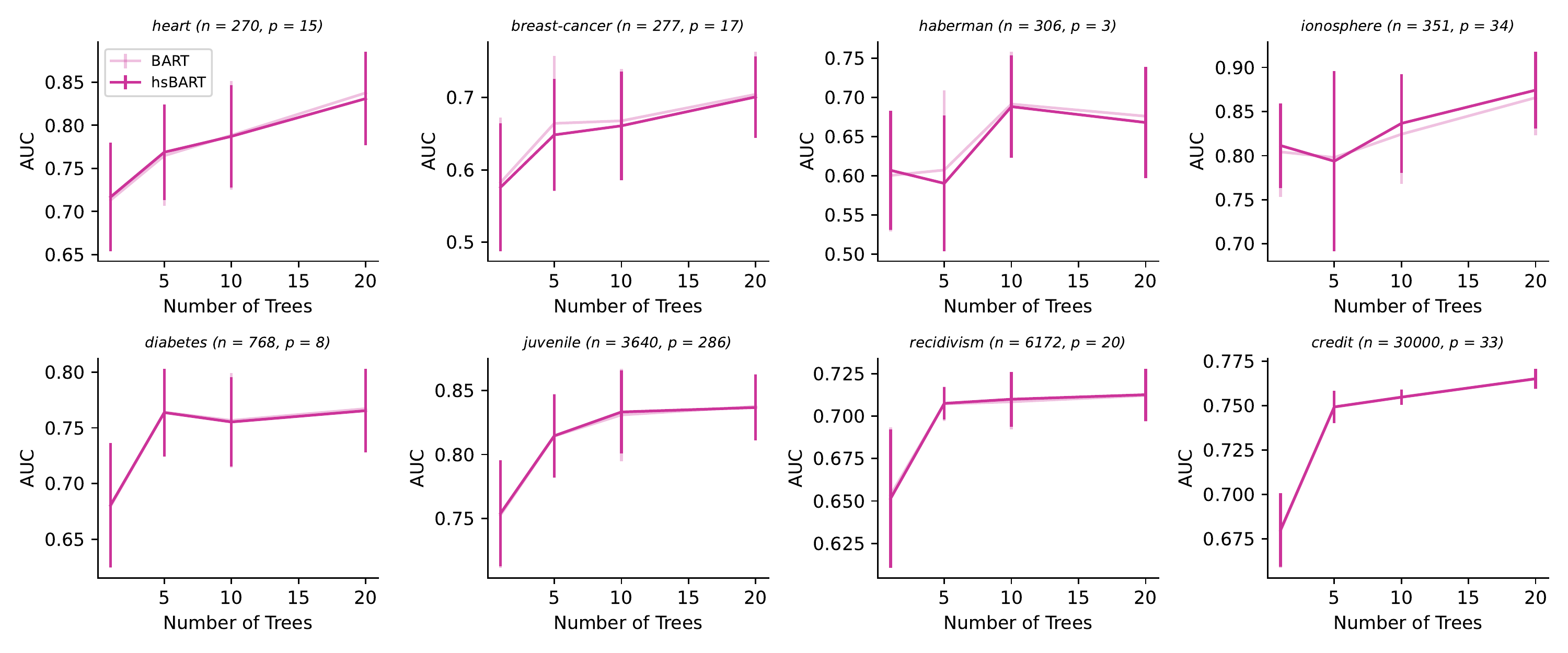}}\\
    \midrule
    \begin{minipage}{0.1cm}
    \rotatebox[origin=c]{90}{\large \textbf{(B)} Regression}
    \end{minipage}
    & \raisebox{\dimexpr-.5\height-1em}{\includegraphics[width=0.91\textwidth]{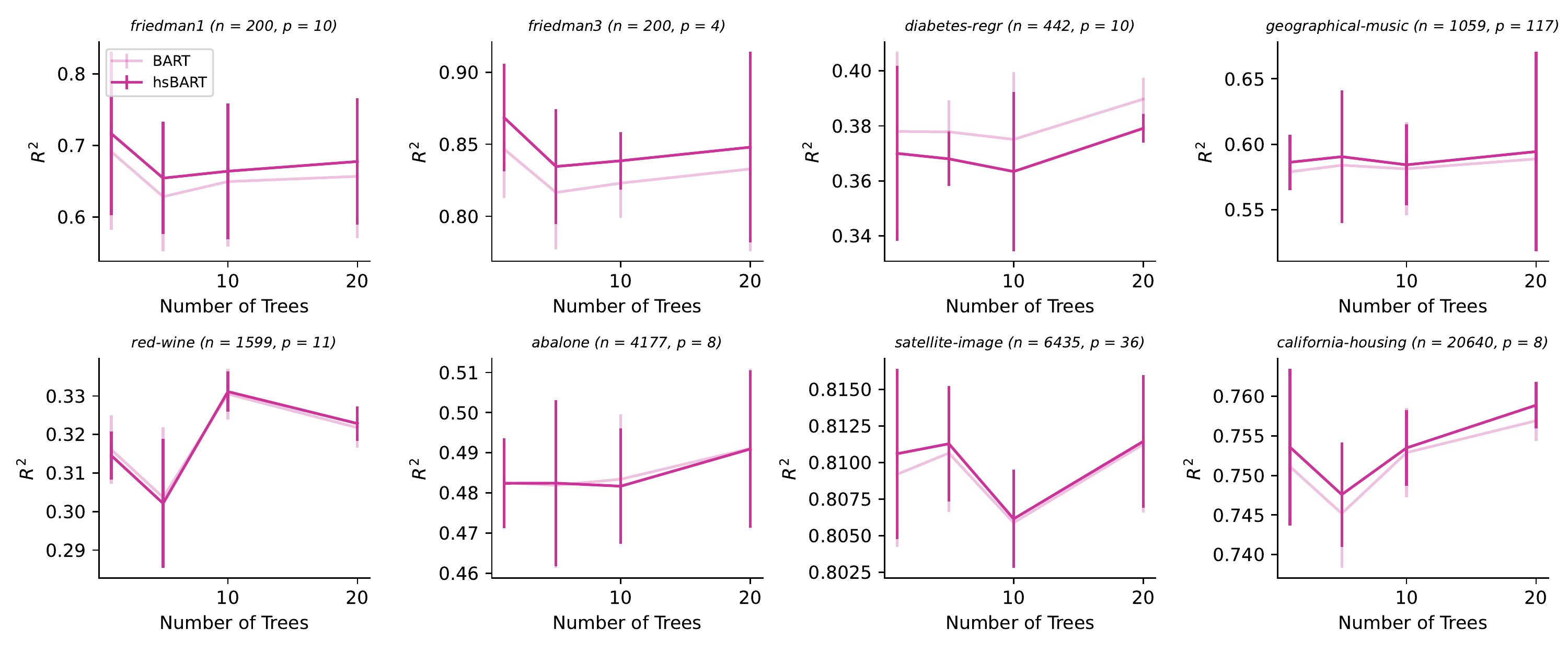}}\\
    \end{tabular}
    \vspace{-5pt}
    \caption{\methodlong~(solid lines) applied post-hoc to BART trees, shows similar performance for classification 
    \textbf{(A)}, and regression  \textbf{(B)} datasets in terms of AUC and $R^2$ respectively }
    \label{fig:performance_curves_bart}
\end{figure}

We hypothesize that the remarkably similar performance, for both regression and classification datasets, is due to BART's tendency to construct relatively small trees, when the recommended parameters are used \cite{chipman2010bart}.

\section{Experimental results for interpretability}

\subsection{Decision boundaries}
\label{subsec:decision_boundary_supp}
In this section, we plot the different decision boundaries learnt by a Random Forest model with 50 trees before/after applying \methods for all the data sets built on the two most important covariates as measured by random forest feature importance. For instance,  \cref{fig:decision_bounds_comp} $\textbf{(D)}$ visualizes the decision boundary for the recidivism data set using the covariates: \textit{age} and \textit{prior counts}.

\begin{figure}[H]
    \vspace{-8pt}
    \aboverulesep = 0.0mm
    \belowrulesep = 0.0mm
    \centering
    \begin{tabular}{lc}
    \begin{minipage}{0.1cm}
    \rotatebox[origin=c]{90}{\large \textbf{(A)} heart}
    \end{minipage}%
    & \raisebox{\dimexpr-.1\height-0.5em}{    \begin{tabular}{cc}
         \includegraphics[width=0.35\textwidth]{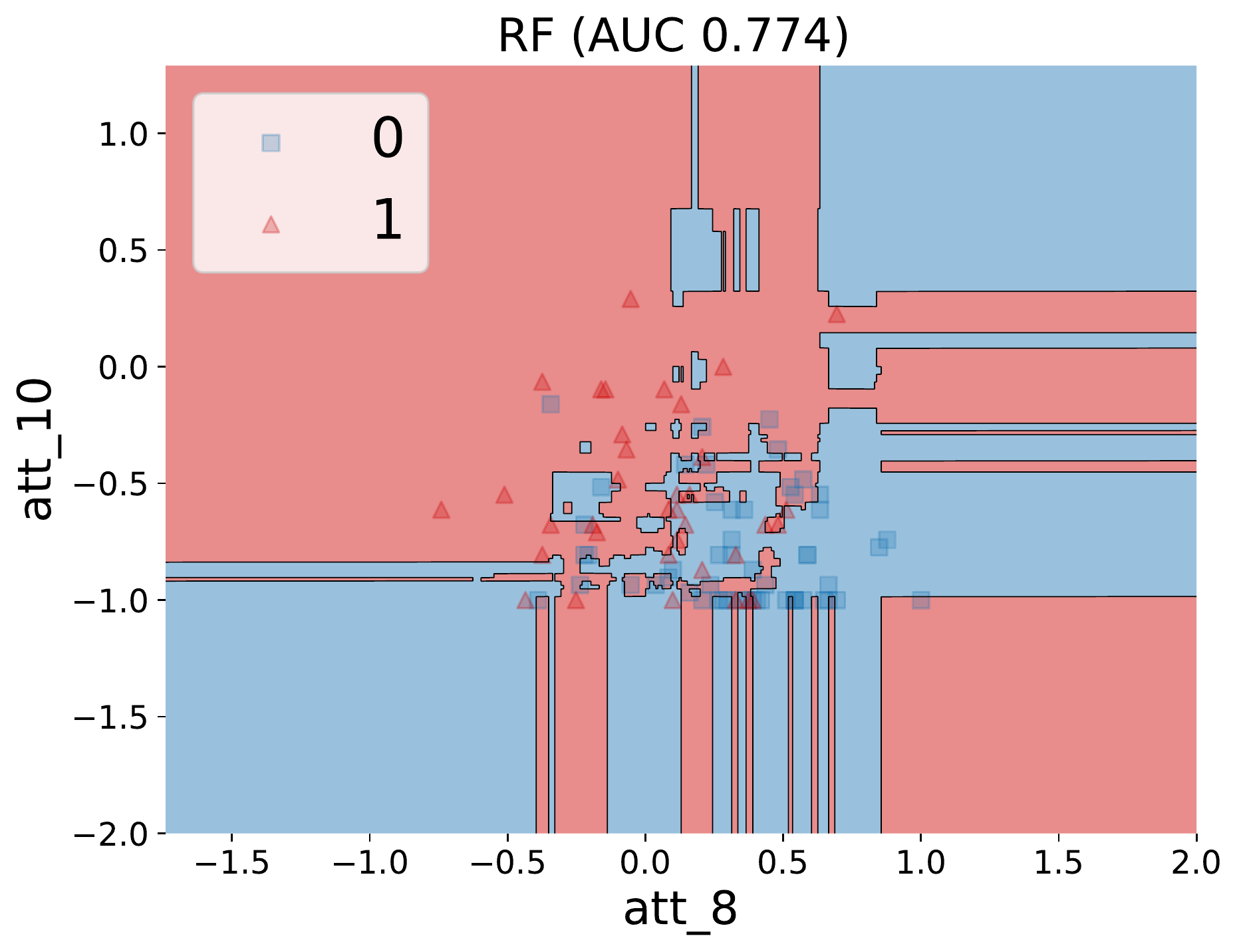} & \includegraphics[width=0.35\textwidth]{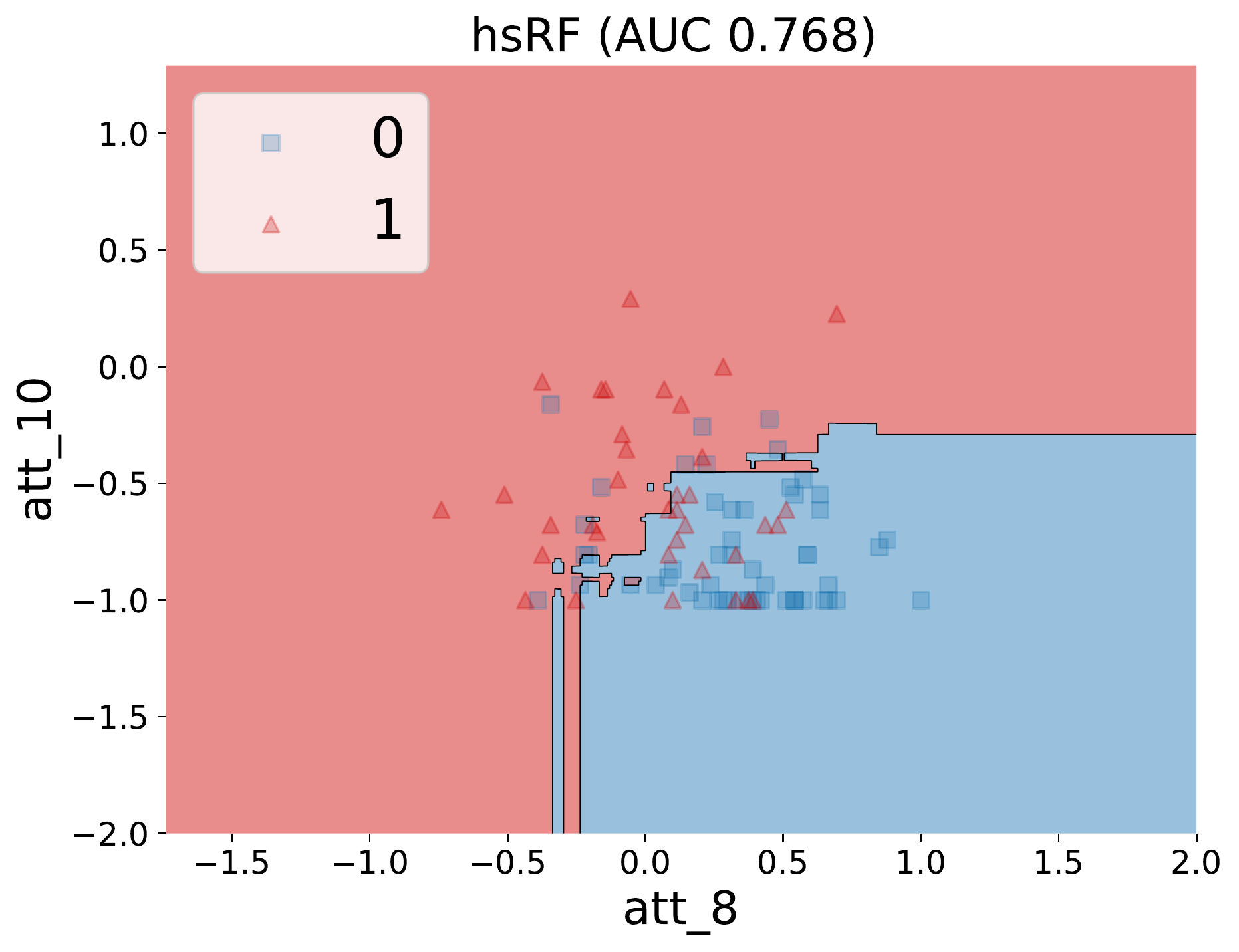}
    \end{tabular}
}\\
    \begin{minipage}{0.1cm}
    \rotatebox[origin=c]{90}{\large \textbf{(B)} breast cancer}
    \end{minipage}
    & \raisebox{\dimexpr-.1\height-0.5em}{    \begin{tabular}{cc}
         \includegraphics[width=0.35\textwidth]{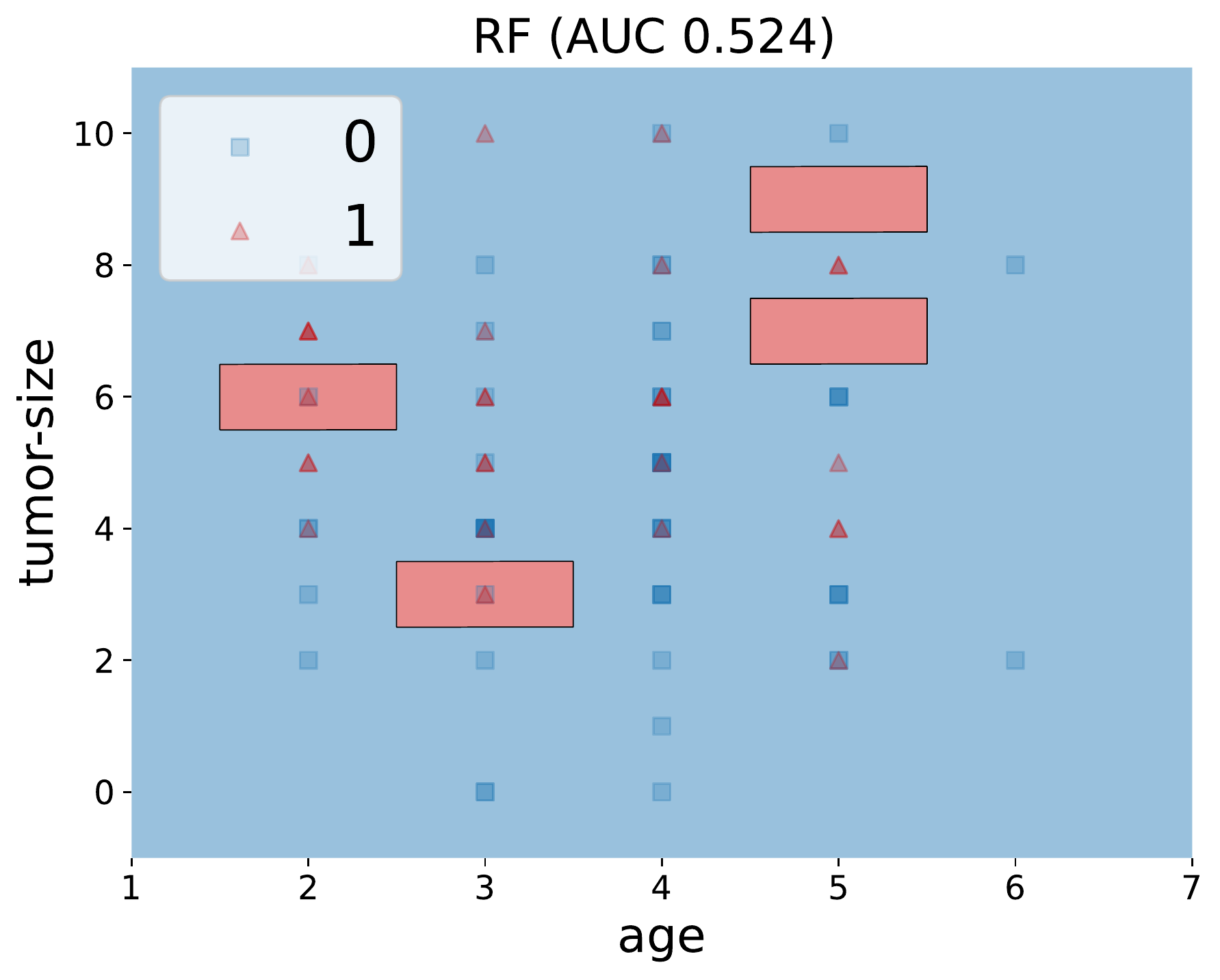} & \includegraphics[width=0.35\textwidth]{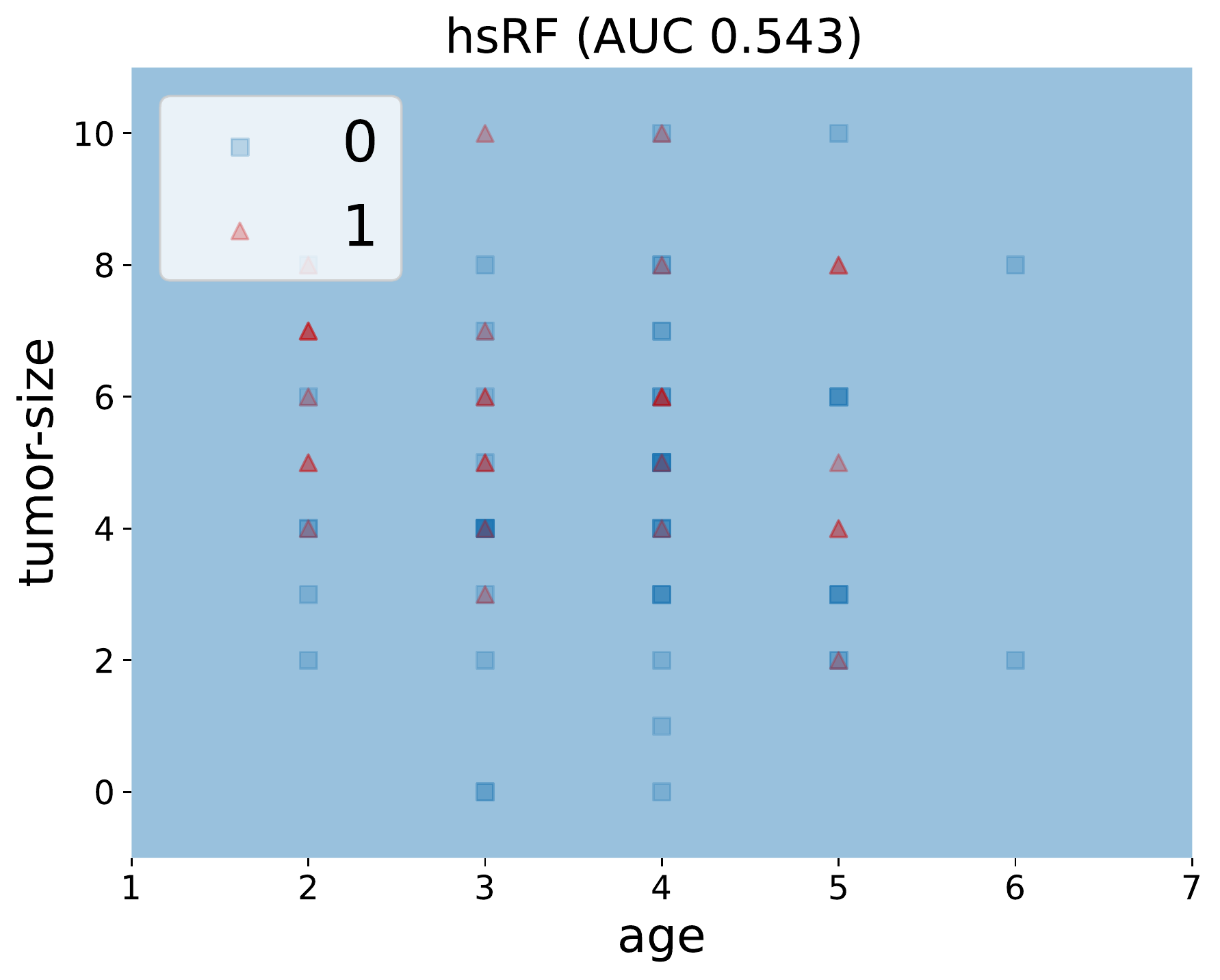}
    \end{tabular}
}\\
        \begin{minipage}{0.1cm}
    \rotatebox[origin=c]{90}{\large \textbf{(C)} haberman}
    \end{minipage}
    & \raisebox{\dimexpr-.1\height-0.5em}{    \begin{tabular}{cc}
          \includegraphics[width=0.35\textwidth]{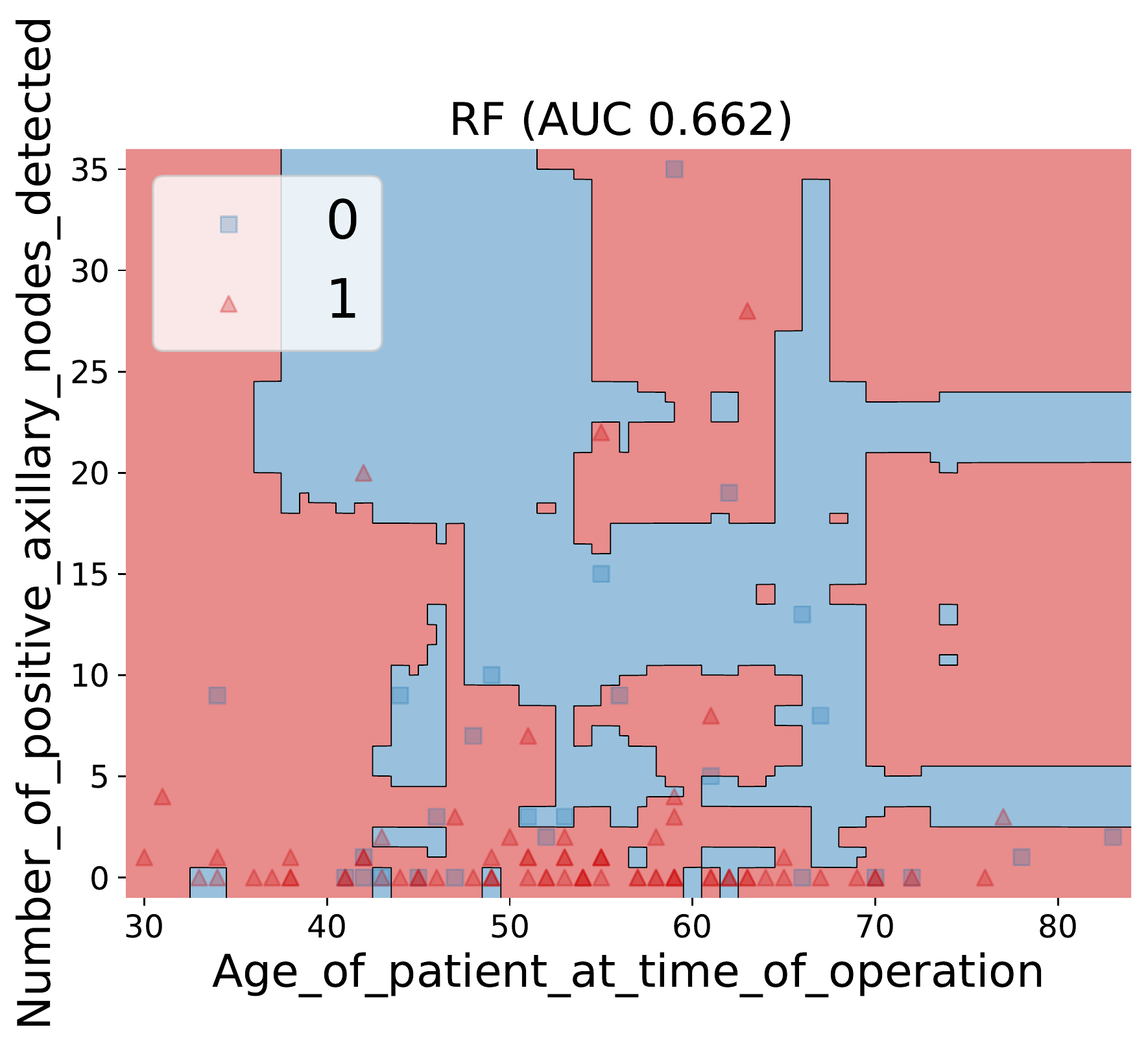} & \includegraphics[width=0.35\textwidth]{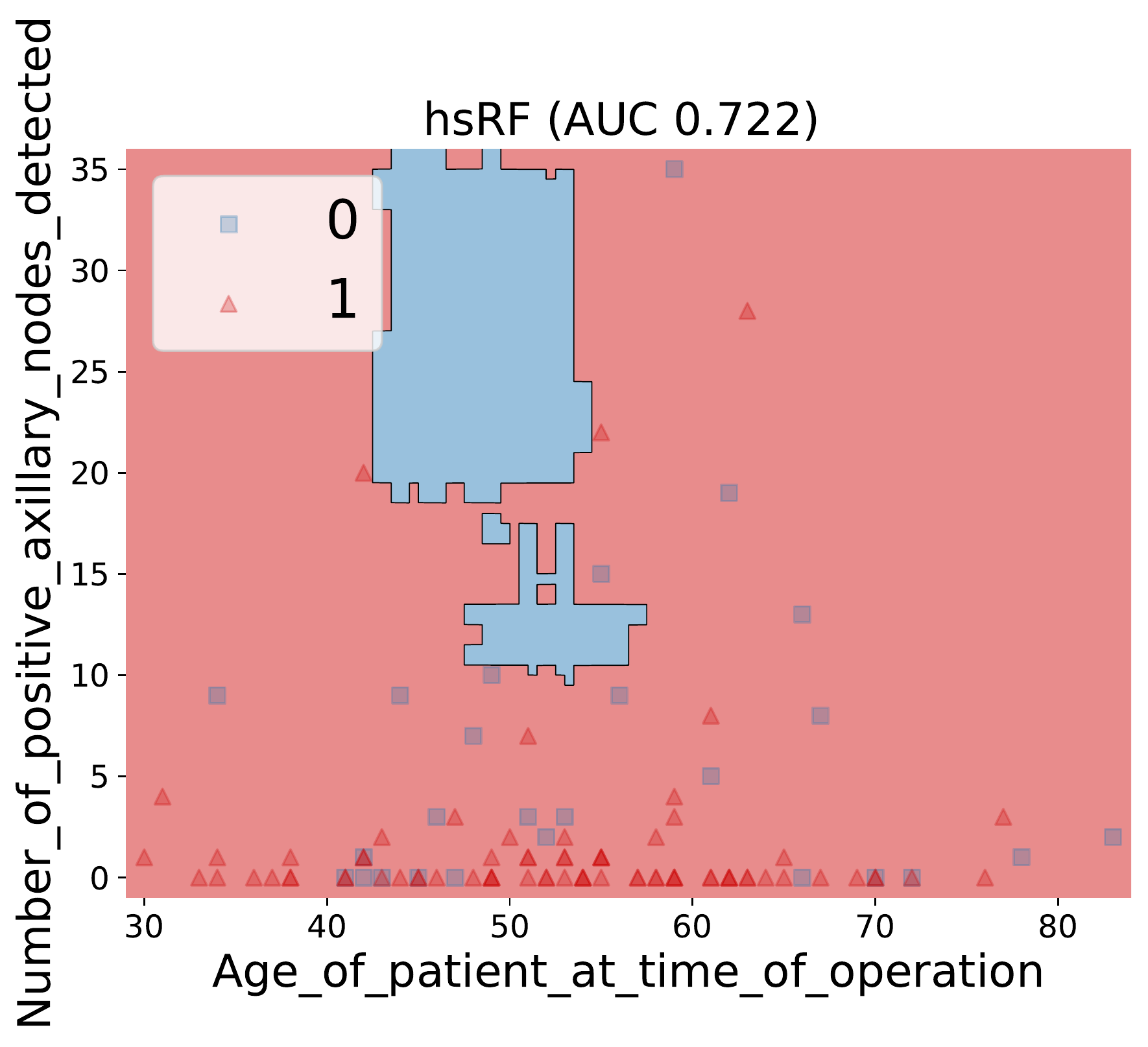}
     \end{tabular}}\\
        \begin{minipage}{0.1cm}
    \rotatebox[origin=c]{90}{\large \textbf{(D)} ionosphere}
    \end{minipage}
    & \raisebox{\dimexpr-.1\height-0.5em}{     \begin{tabular}{cc}
         \includegraphics[width=0.35\textwidth]{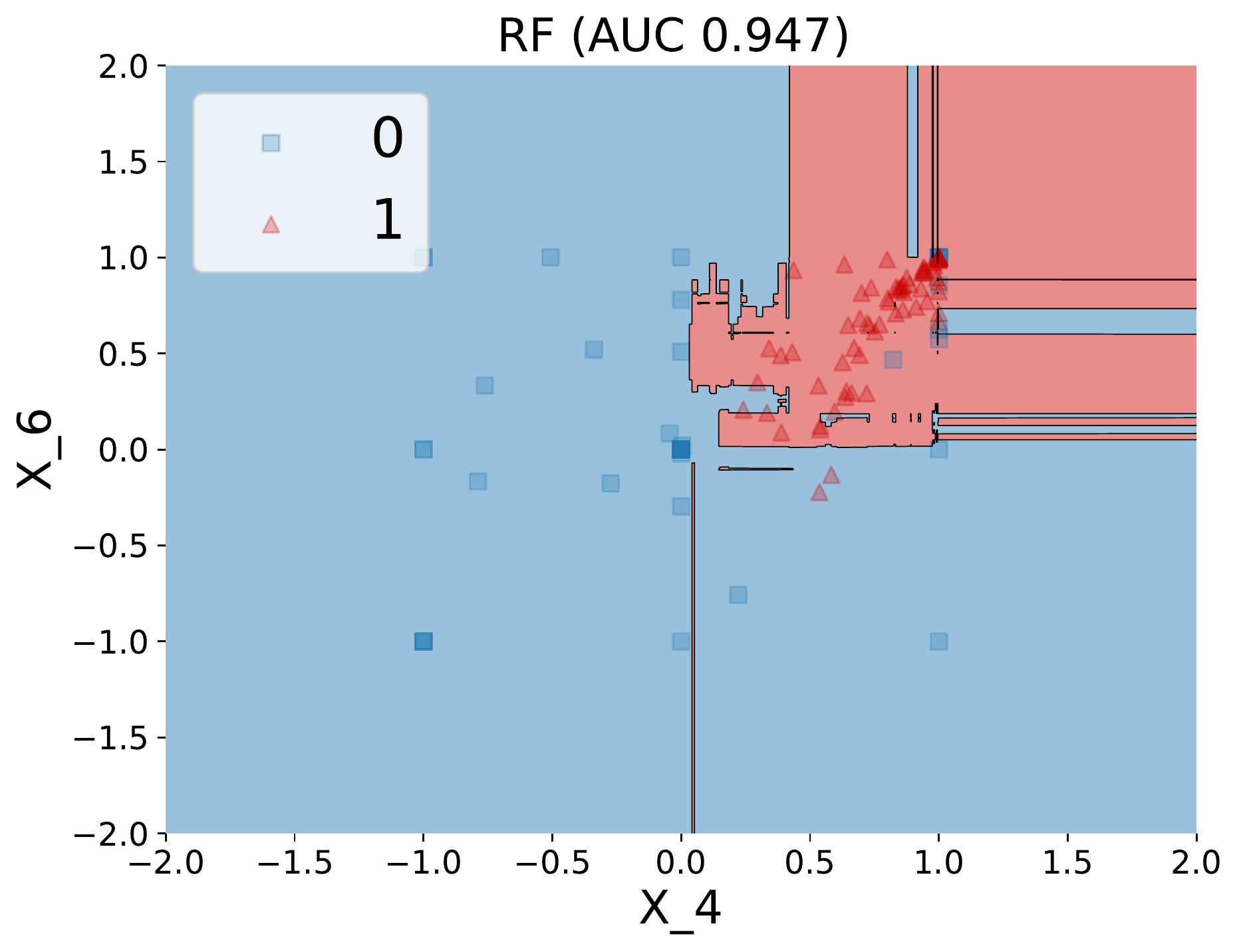} & \includegraphics[width=0.35\textwidth]{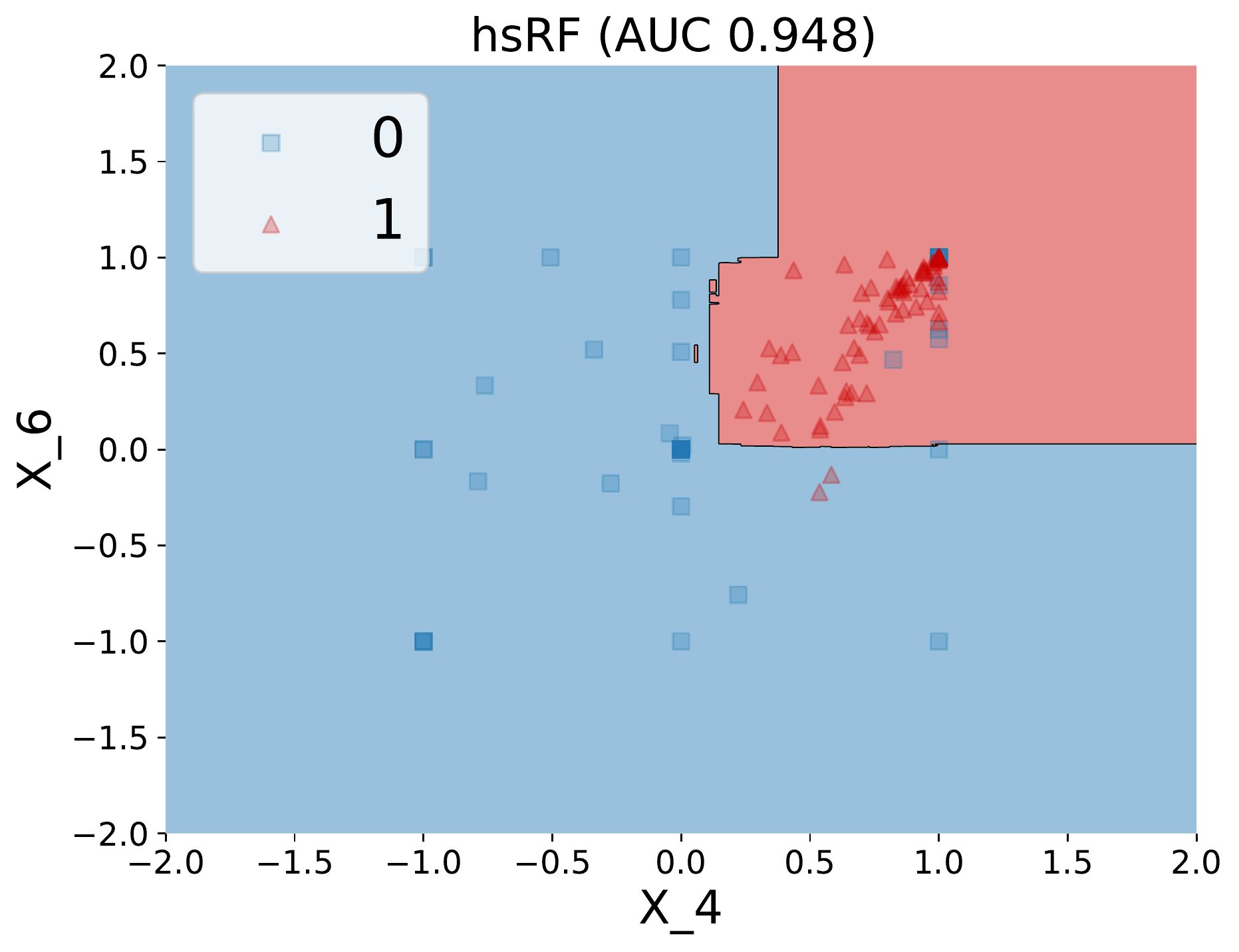}
    \end{tabular}}\\
    \end{tabular}
    \vspace{-10pt}
    \caption{Comparison of decision boundary learnt by Random Forests on the smaller data sets before / after applying \method.Even though \methods does not result in an increase in predictive performance, it learns a smoother, simpler decision boundary, resulting in improved interpretability and simplicity.}
    \label{fig:decision_bounds_simpl}
\end{figure}

\begin{figure}[H]
    \vspace{-8pt}
    \aboverulesep = 0.0mm
    \belowrulesep = 0.0mm
    \centering
    \begin{tabular}{lc}
    \begin{minipage}{0.1cm}
    \rotatebox[origin=c]{90}{\large \textbf{(A)} diabetes}
    \end{minipage}%
    & \raisebox{\dimexpr-.1\height-0.5em}{    \begin{tabular}{cc}
         \includegraphics[width=0.35\textwidth]{figs/decision_boundaries/diabetes_decision_boundary_random_forest.pdf} & \includegraphics[width=0.35\textwidth]{figs/decision_boundaries/diabetes_decision_boundary_shrunk.pdf}
    \end{tabular}
}\\
    \begin{minipage}{0.1cm}
    \rotatebox[origin=c]{90}{\large \textbf{(B)} german credit}
    \end{minipage}
    & \raisebox{\dimexpr-.1\height-0.5em}{    \begin{tabular}{cc}
         \includegraphics[width=0.35\textwidth]{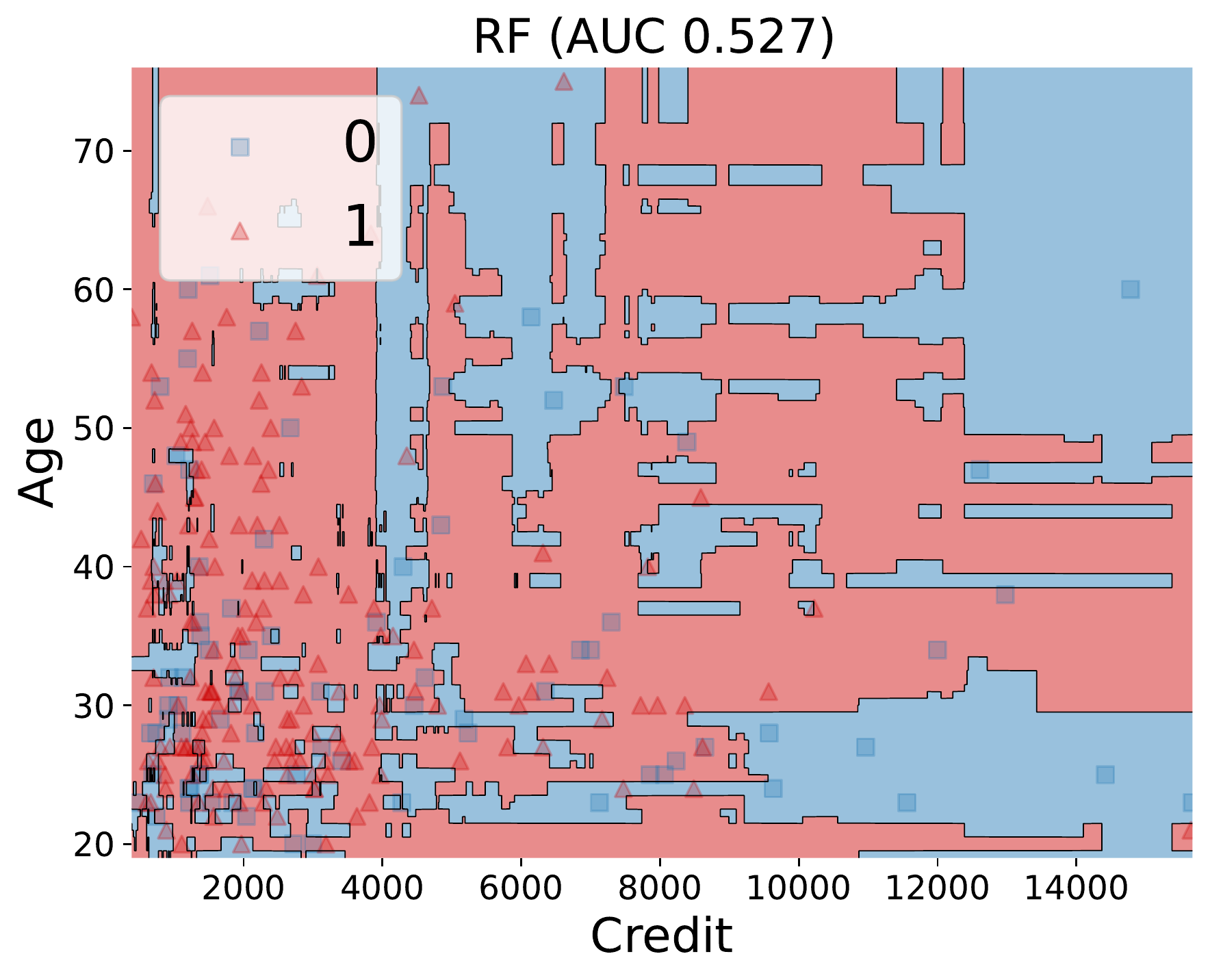} & \includegraphics[width=0.35\textwidth]{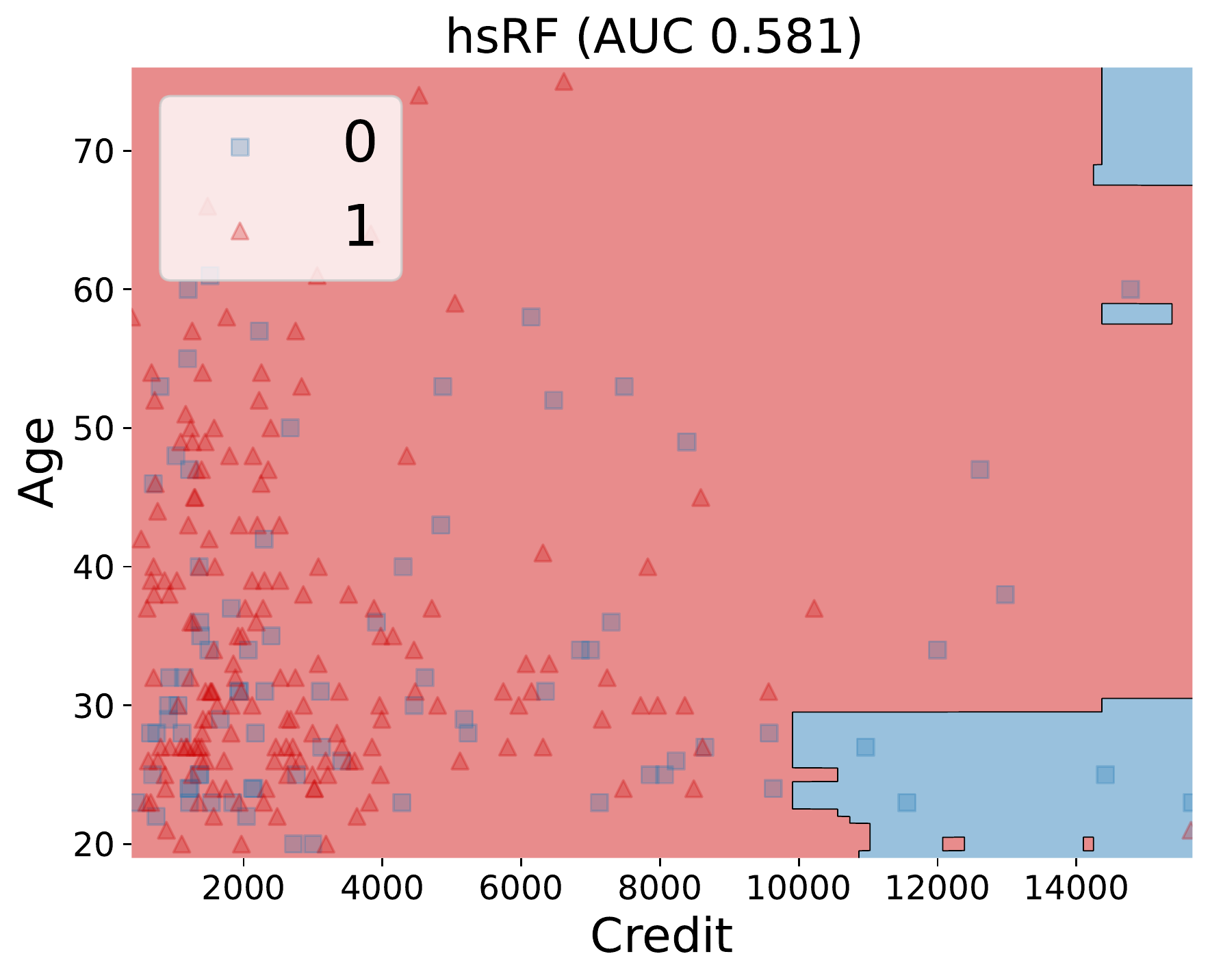}
    \end{tabular}
}\\
        \begin{minipage}{0.1cm}
    \rotatebox[origin=c]{90}{\large \textbf{(C)} juvenile}
    \end{minipage}
    & \raisebox{\dimexpr-.1\height-0.5em}{    \begin{tabular}{cc}
          \includegraphics[width=0.35\textwidth]{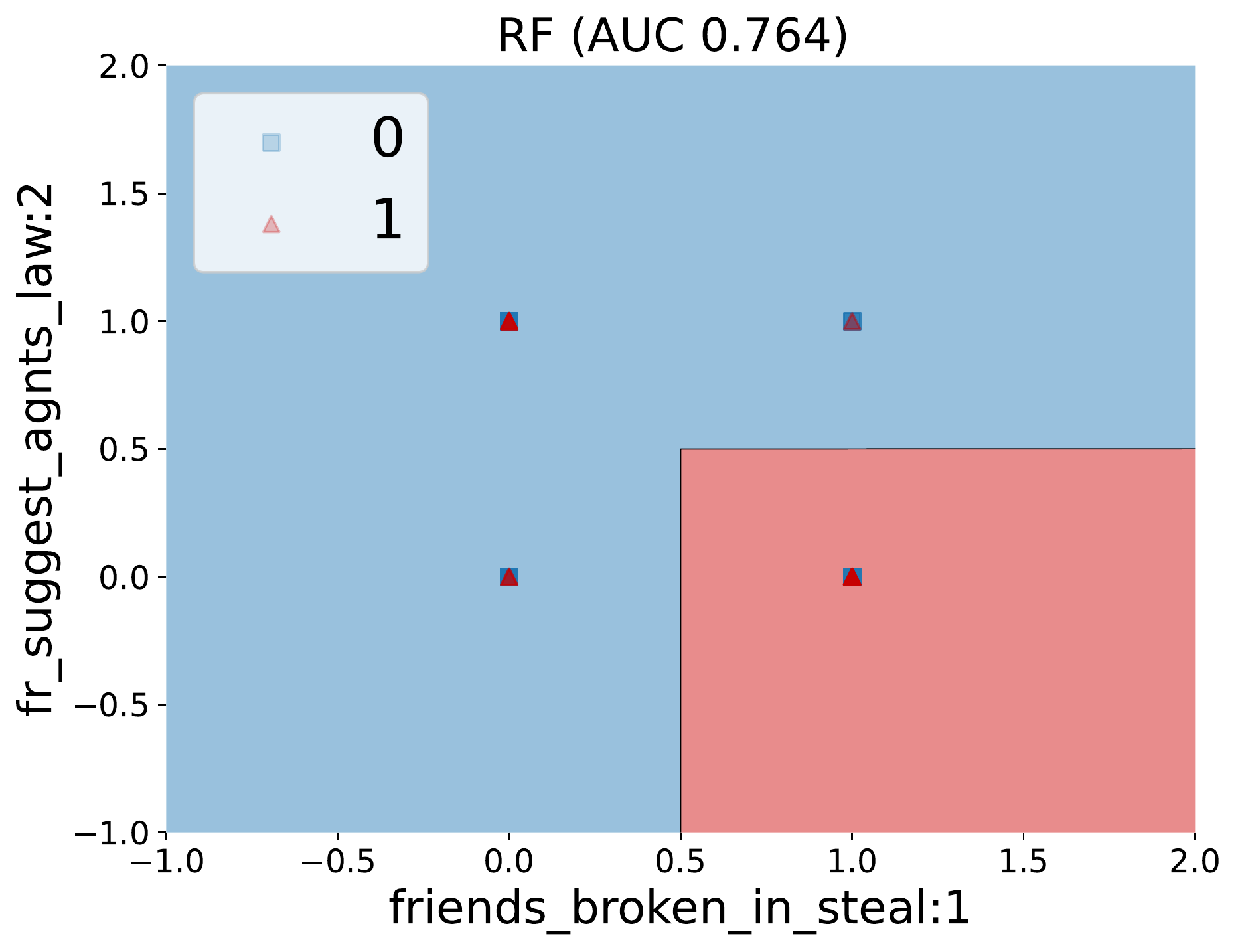} & \includegraphics[width=0.35\textwidth]{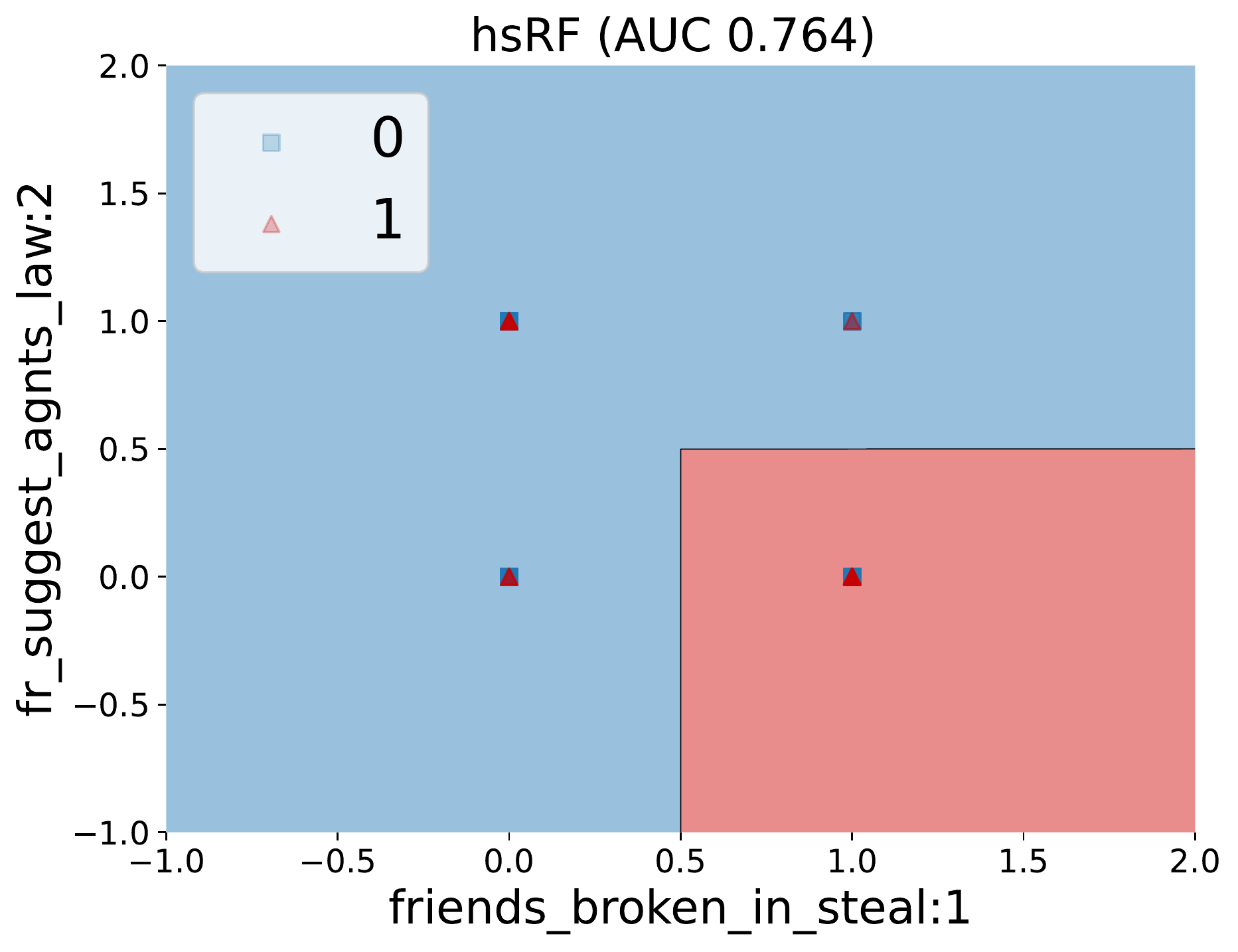}
     \end{tabular}}\\

        \begin{minipage}{0.1cm}
    \rotatebox[origin=c]{90}{\large \textbf{(D)} recidivism}
    \end{minipage}
    & \raisebox{\dimexpr-.1\height-0.5em}{     \begin{tabular}{cc}
         \includegraphics[width=0.35\textwidth]{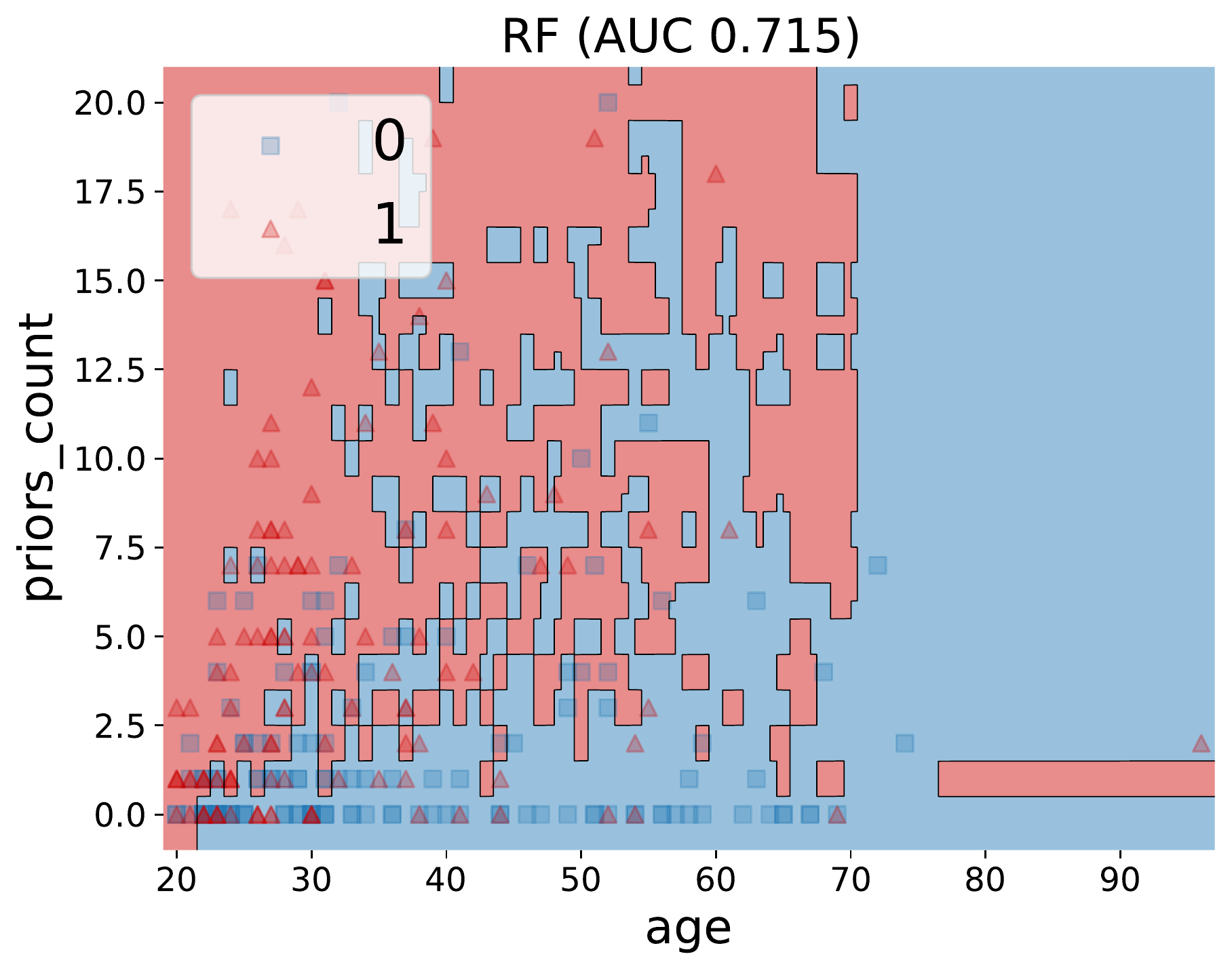} & \includegraphics[width=0.35\textwidth]{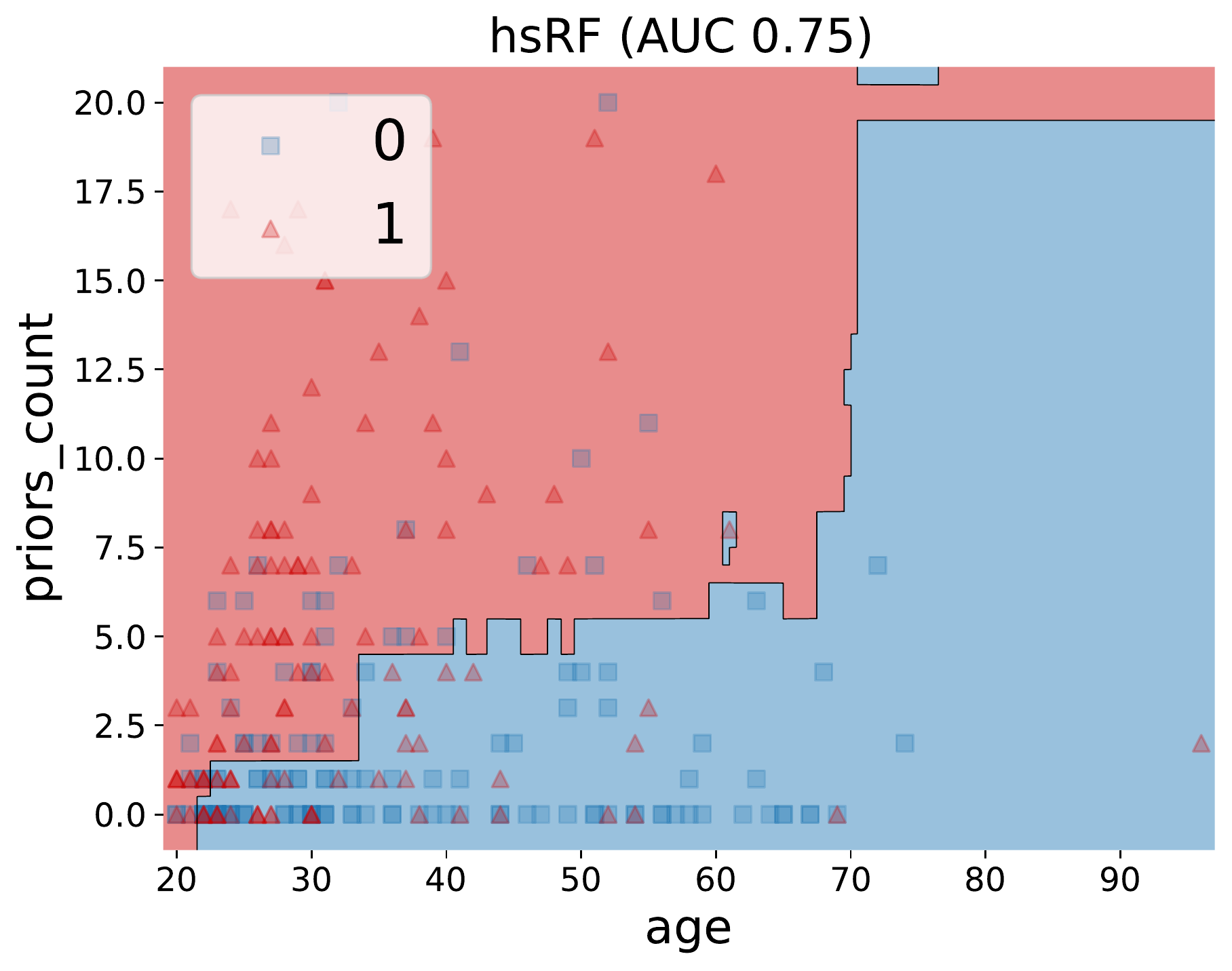}
    \end{tabular}}\\
    \end{tabular}
    \vspace{-10pt}
    \caption{Comparison of decision boundary learnt by Random Forests on the larger data sets before / after applying \method.Even though \methods does not result in an increase in predictive performance, it learns a smoother, simpler decision boundary, resulting in improved interpretability and simplicity.}
    \label{fig:decision_bounds_comp}
\end{figure}

\newpage
\subsection{SHAP Plots}
\label{subsec:shap_plots}
In this section, we investigate SHAP plots \cite{lundberg2017unified} for RF with and without \methods. SHAP plots are a popular tool to explain black-box models such as RFs, and provide explanations by computing a shapley value for every sample and feature. We provide SHAP plots for RF with and without \methods for every dataset in \cref{tab:datasets}. 

\begin{figure}[H]
    \vspace{-8pt}
    \aboverulesep = 0.0mm
    \belowrulesep = 0.0mm
    \centering
    \begin{tabular}{lc}
    \begin{minipage}{0.1cm}
    \rotatebox[origin=c]{90}{\large \textbf{(A)} heart}
    \end{minipage}%
    & \raisebox{\dimexpr-.5\height-1em}{\includegraphics[width=0.7\textwidth]{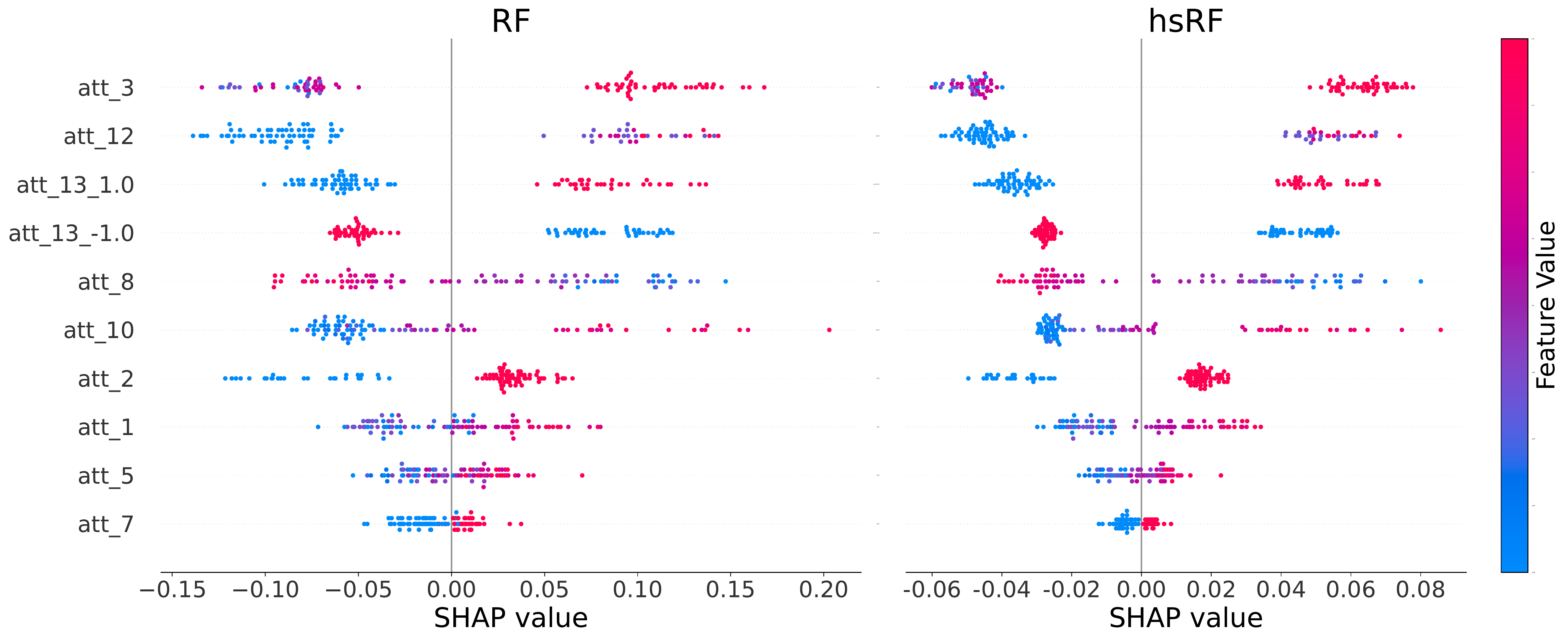}}\\
    \begin{minipage}{0.1cm}
    \rotatebox[origin=c]{90}{\large \textbf{(B)} breast cancer}
    \end{minipage}
    & \raisebox{\dimexpr-.5\height-1em}{  \includegraphics[width=0.7\textwidth]{figs/SHAP_summary_plot/SHAP_summary_plot_breast_cancer.pdf}
}\\
        \begin{minipage}{0.1cm}
    \rotatebox[origin=c]{90}{\large \textbf{(C)} haberman}
    \end{minipage}
    & \raisebox{\dimexpr-.5\height-1em}{ \includegraphics[width=0.7\textwidth]{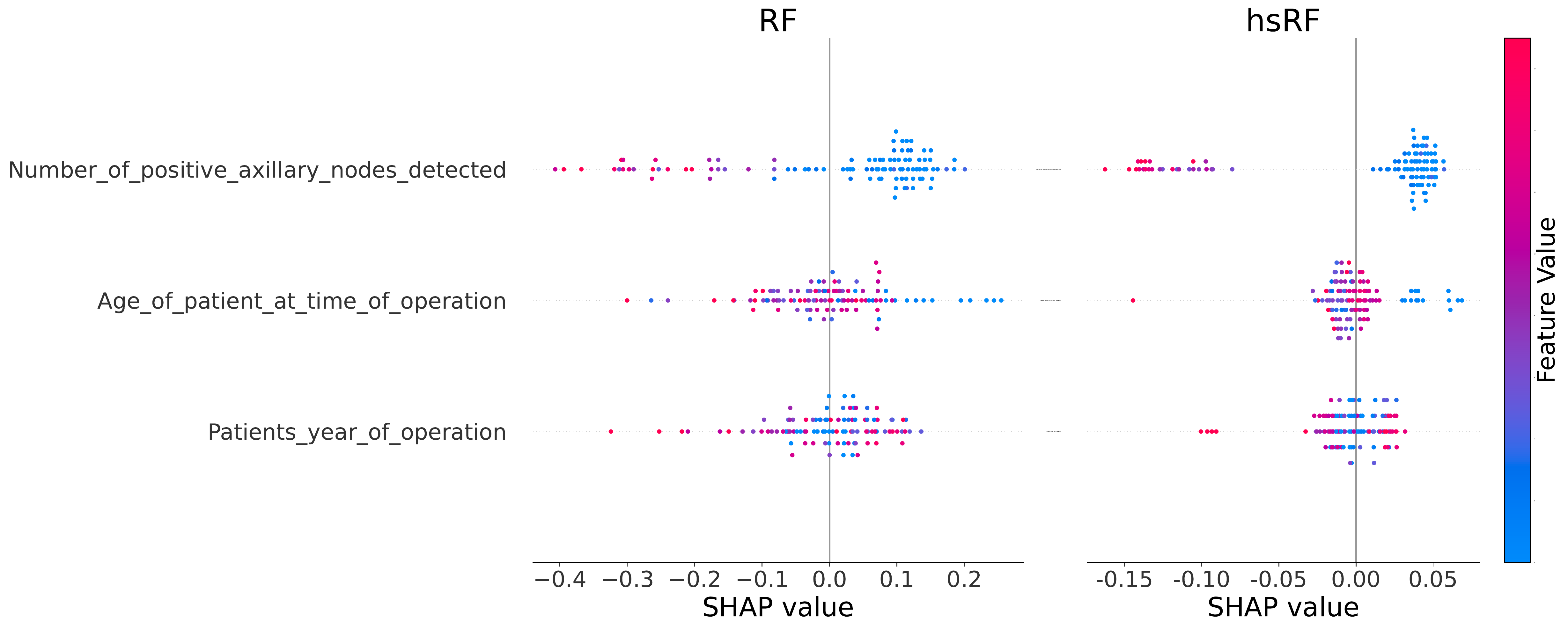}}\\

        \begin{minipage}{0.1cm}
    \rotatebox[origin=c]{90}{\large \textbf{(D)} ionosphere}
    \end{minipage}
    & \raisebox{\dimexpr-.5\height-1em}{ \includegraphics[width=0.7\textwidth]{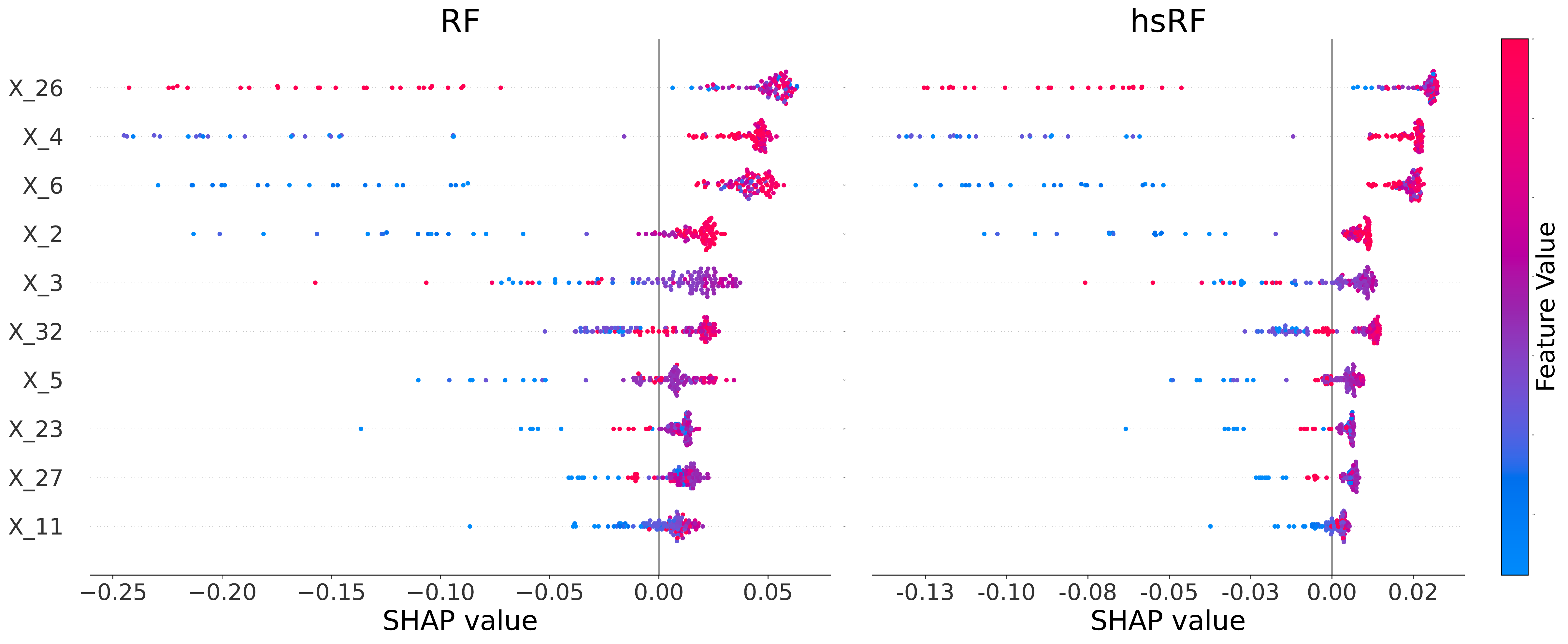}}\\
    \end{tabular}
    \vspace{-5pt}
    \caption{Comparison of SHAP plots learnt by Random Forests on the smaller data sets before / after applying \method.
    \methods displays more clustered SHAP values for each feature, reflecting less heterogeneity in the effect of each feature on predicted response.}
    \label{fig:shap_vals_small}
\end{figure}

\begin{figure}[H]
    \vspace{-8pt}
    \aboverulesep = 0.0mm
    \belowrulesep = 0.0mm
    \centering
    \begin{tabular}{lc}
    \begin{minipage}{0.1cm}
    \rotatebox[origin=c]{90}{\large \textbf{(A)} diabetes}
    \end{minipage}%
    & \raisebox{\dimexpr-.5\height-1em}{\includegraphics[width=0.7\textwidth]{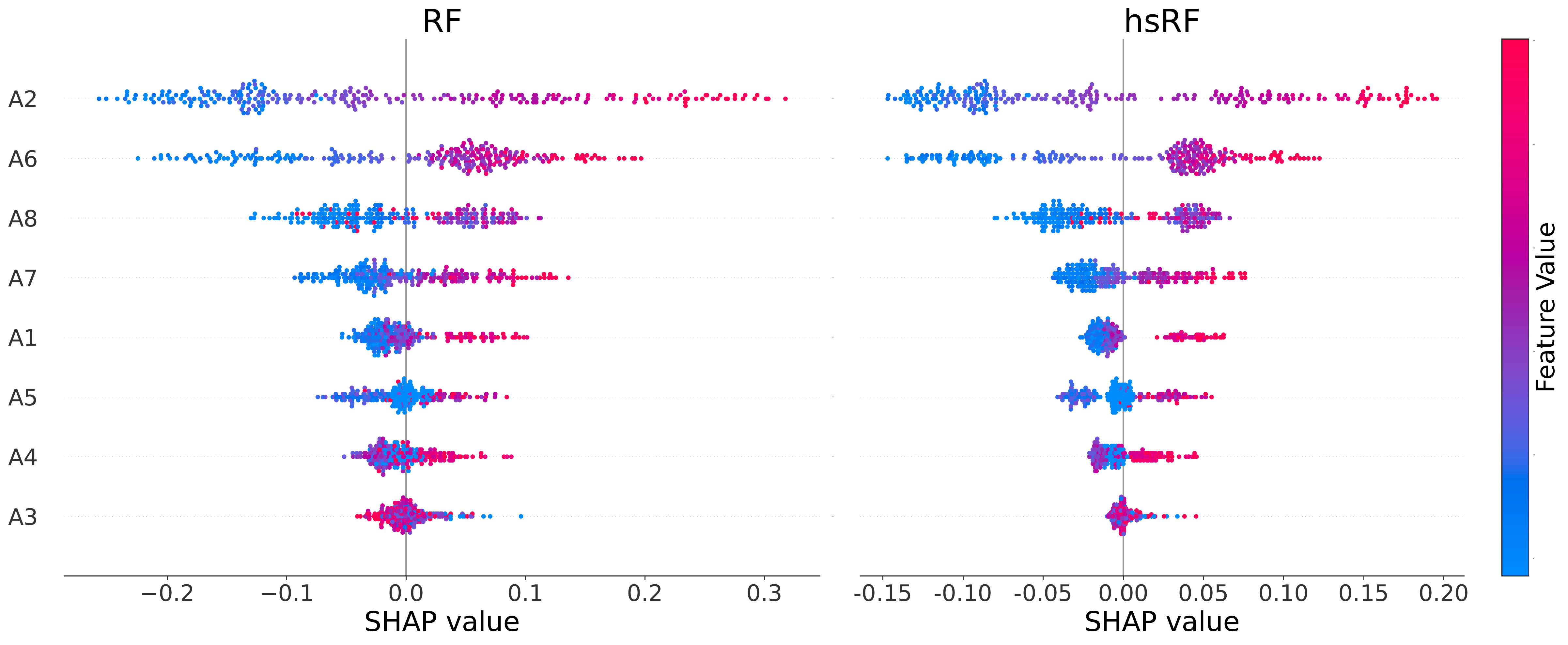}}\\
    \begin{minipage}{0.1cm}
    \rotatebox[origin=c]{90}{\large \textbf{(B)} german credit}
    \end{minipage}
    & \raisebox{\dimexpr-.5\height-1em}{  \includegraphics[width=0.7\textwidth]{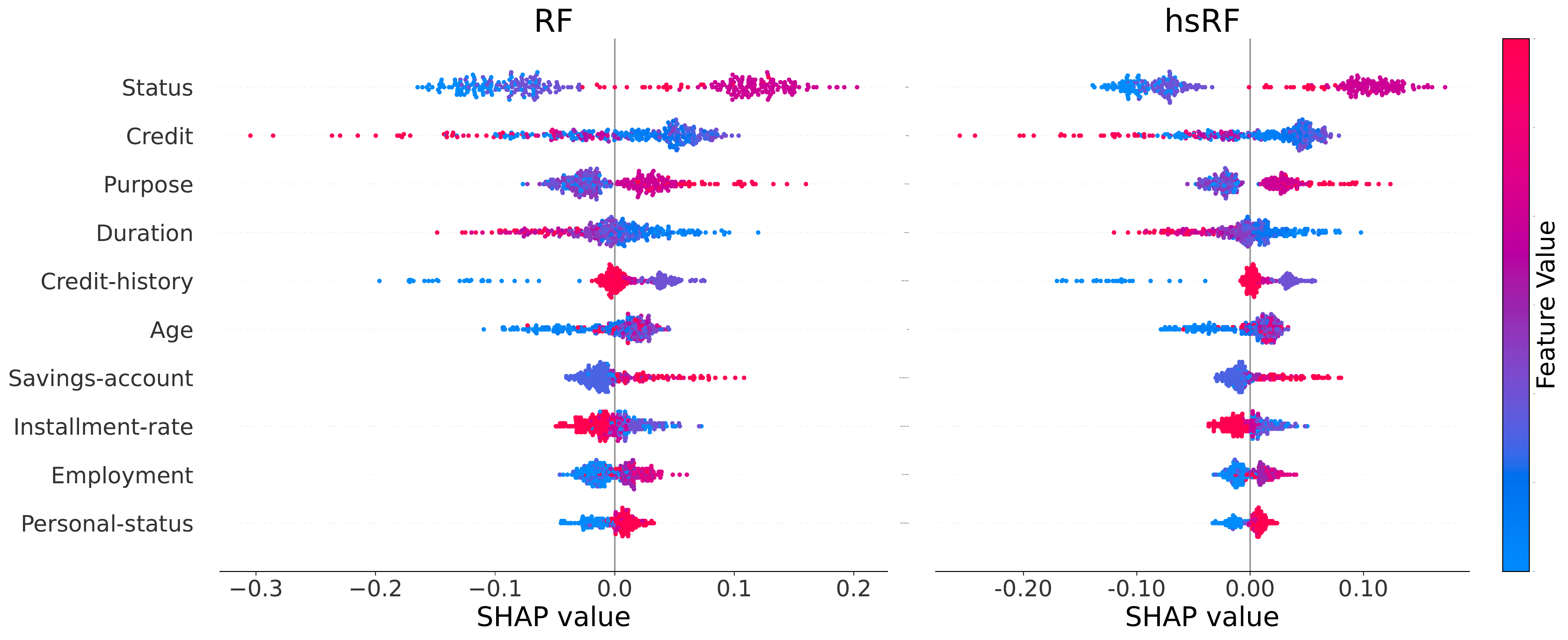}
}\\
        \begin{minipage}{0.1cm}
    \rotatebox[origin=c]{90}{\large \textbf{(C)} juvenile}
    \end{minipage}
    & \raisebox{\dimexpr-.5\height-1em}{ \includegraphics[width=0.7\textwidth]{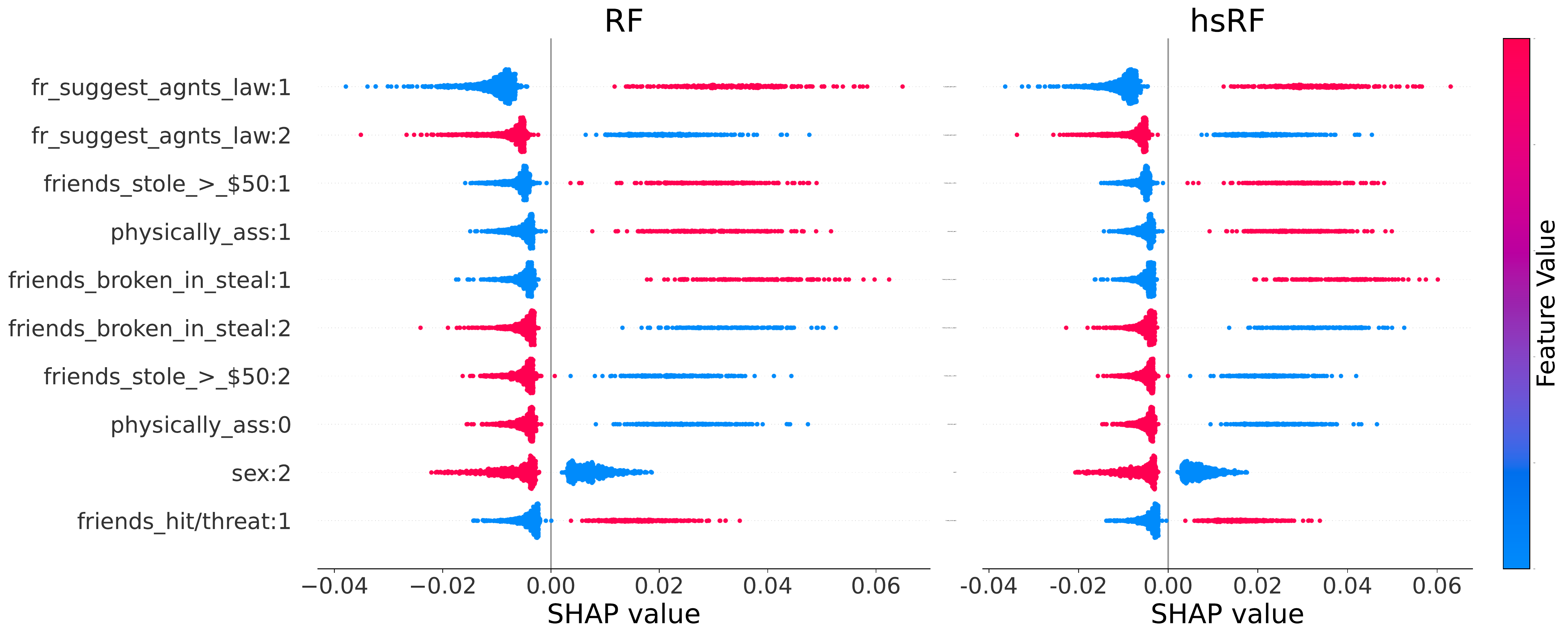}}\\

        \begin{minipage}{0.1cm}
    \rotatebox[origin=c]{90}{\large \textbf{(D)} recidivism}
    \end{minipage}
    & \raisebox{\dimexpr-.5\height-1em}{ \includegraphics[width=0.7\textwidth]{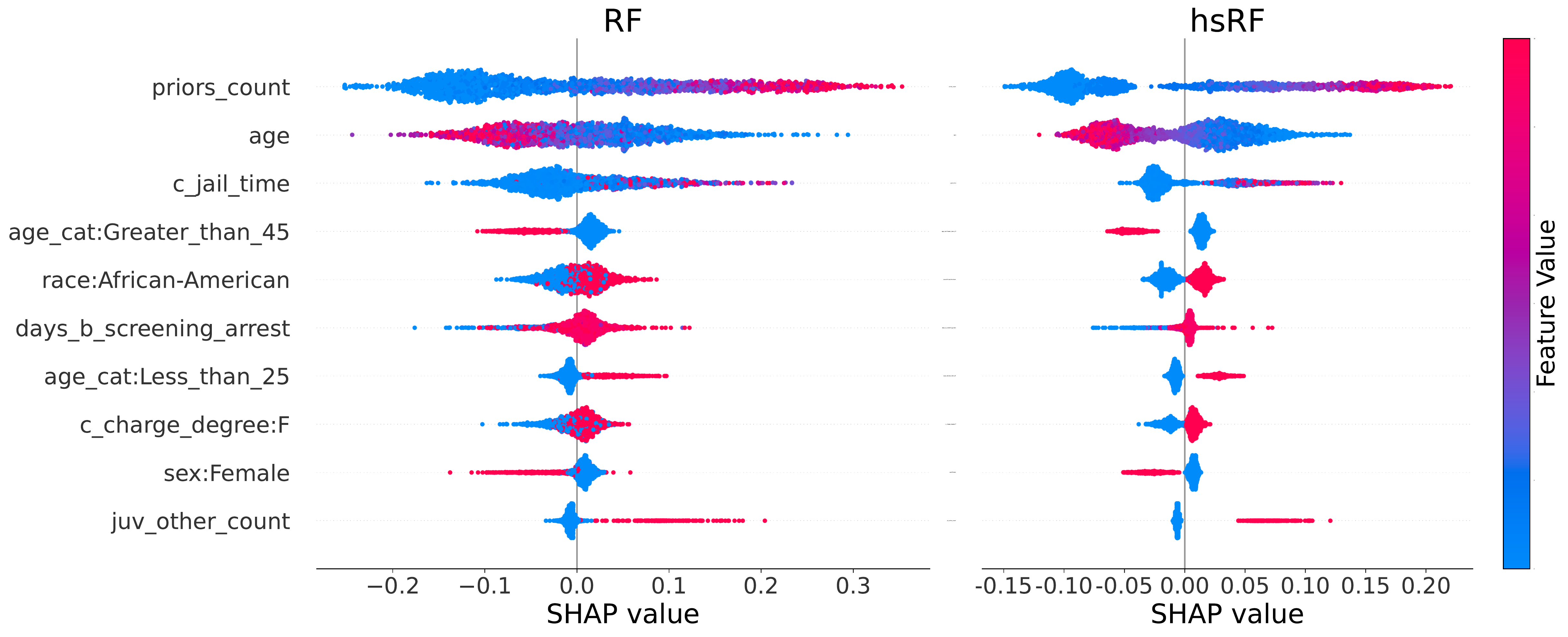}}\\
    \end{tabular}
    \vspace{-5pt}
    \caption{Comparison of SHAP plots learnt by Random Forests on the larger data sets before / after applying \method.
    \methods displays more clustered SHAP values for each feature, reflecting less heterogeneity in the effect of each feature on predicted response.}
    \label{fig:shap_vals_large}
\end{figure}

\clearpage
\subsection{SHAP Variability plots}
\label{subsec:SHAP_variability_supp}
We investigate the stability of SHAP scores to data perturbations arising from different train/test splits for all classification data sets in \cref{tab:datasets}. In order to do this, we randomly choose 50 held-out samples in each data set and measure the variance of their SHAP scores per feature across 100 different train-test splits for RF with and without \methods. We then average the variance per feature across all 50 held-out samples and visualize them in the figures below. As seen in the figures below, we see that SHAP values for RF with \methods are more stable to perturbations, thereby verifying their increased interpretability.  
\begin{figure}[htbp]
    \begin{tabular}{cc}
    \centering
        (A) Heart & (B) Breast cancer\\
         \includegraphics[width=0.45\textwidth]{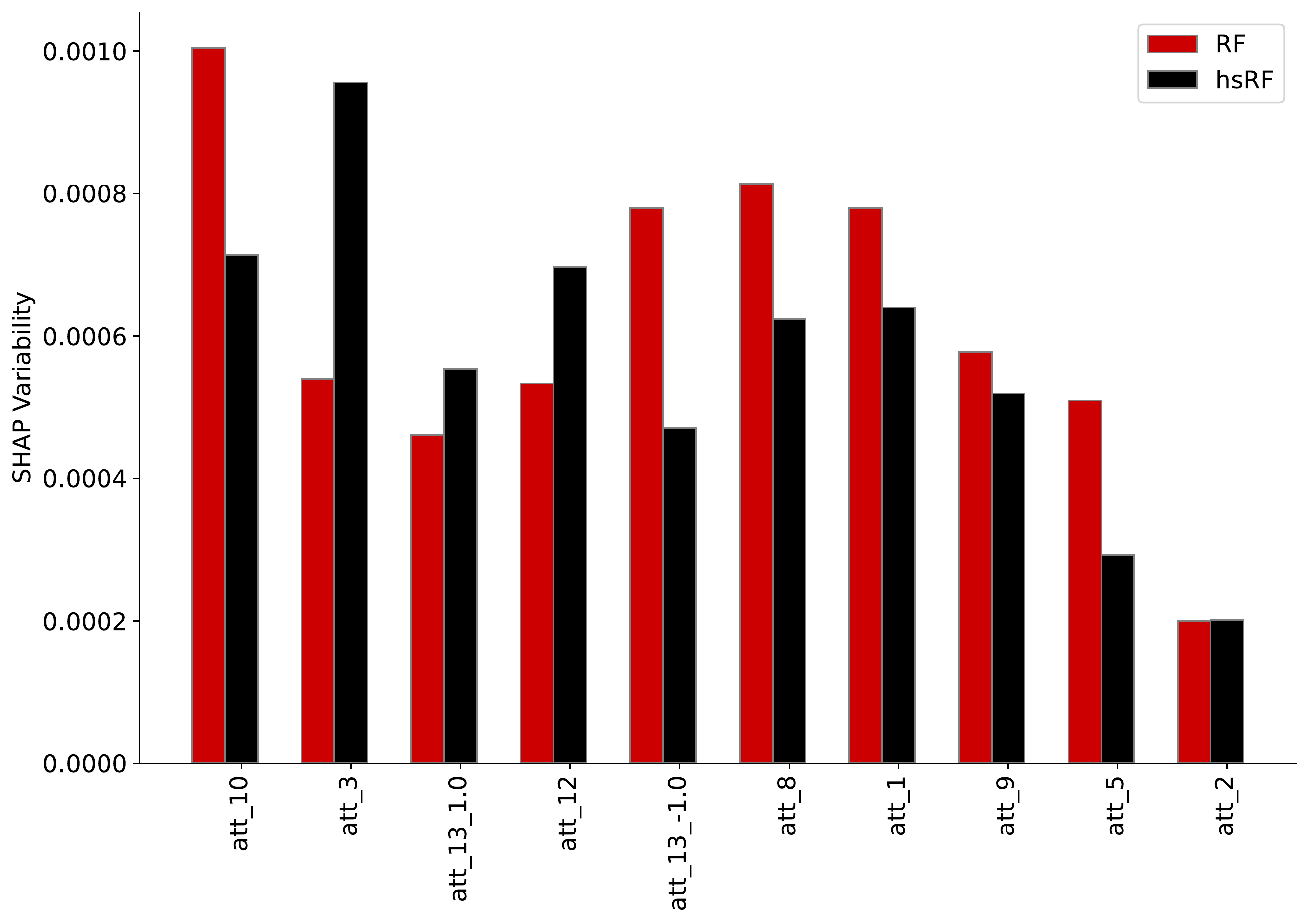} & \includegraphics[width=0.45\textwidth]{figs/SHAP_variability_plot/SHAP_variability_plot_breast_cancer.pdf}\\
        (C) Haberman & (D) Ionosphere\\ 
        \raisebox{-\height}{\includegraphics[width=0.45\textwidth]{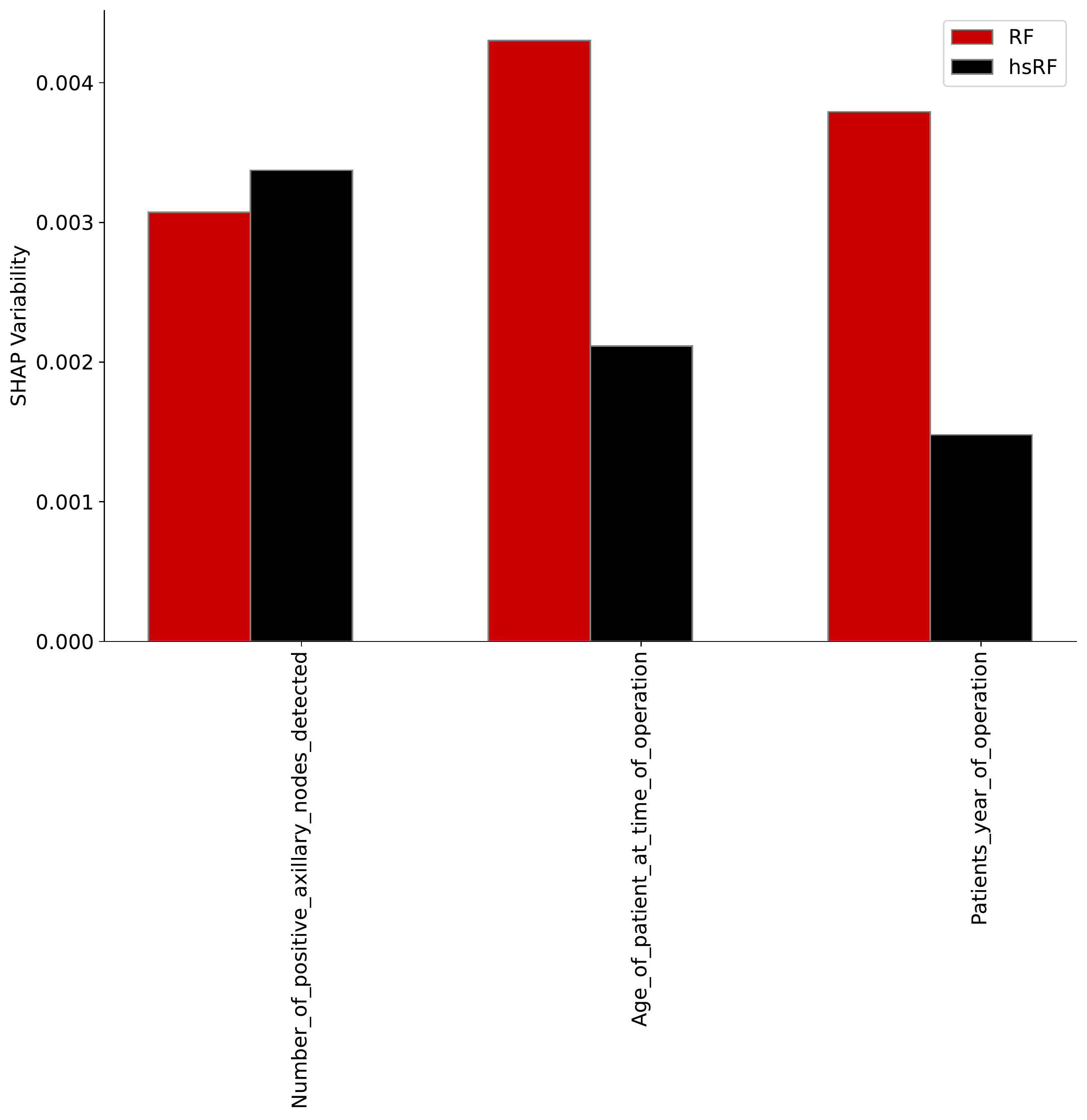}} & 
        \raisebox{-\height}{\includegraphics[width=0.45\textwidth]{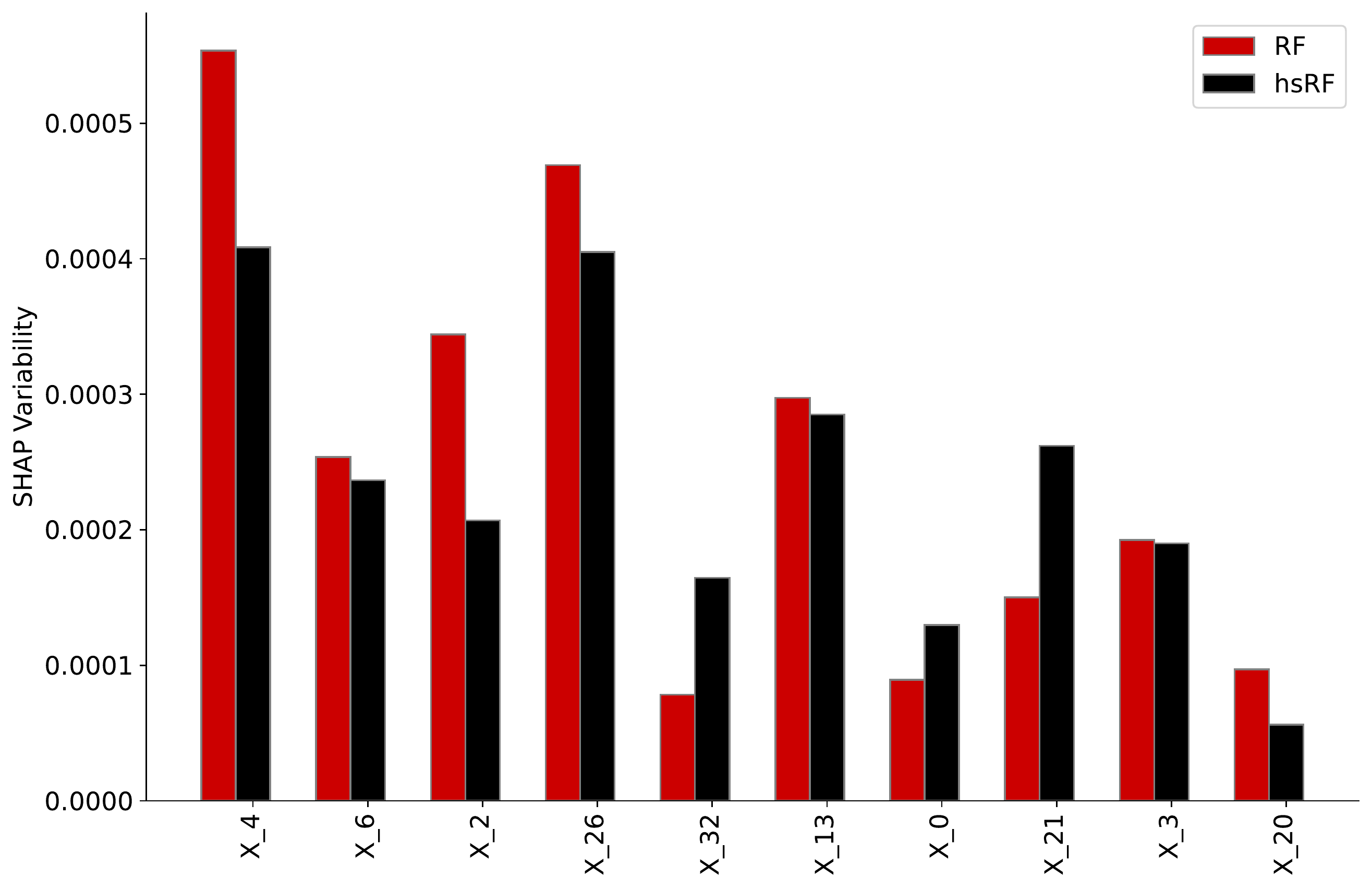}}\\
    \end{tabular}
    \caption{Comparison of SHAP plots learnt by Random Forests on different (small) datasets before / after applying \method.
    \methods displays lower variability across different data perturbations, indicating enhanced stability.}
    \label{fig:shap_variability_small}
\end{figure}

\begin{figure}[H]
    \begin{tabular}{cc}
    \centering
        (A) Diabetes & (B) Credit\\
         \raisebox{-\height}{\includegraphics[width=0.45\textwidth]{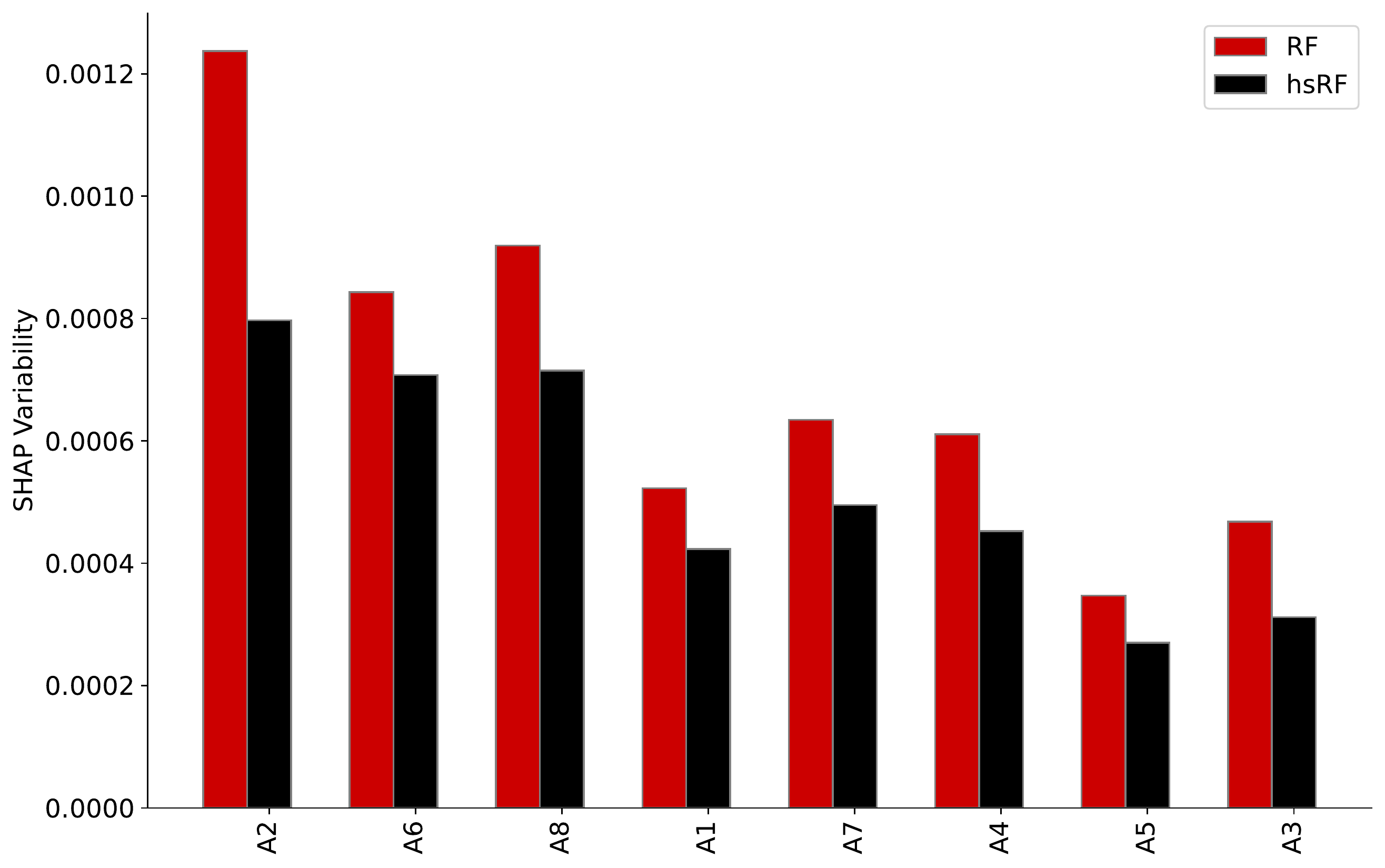}} & \raisebox{-\height}{\includegraphics[width=0.45\textwidth]{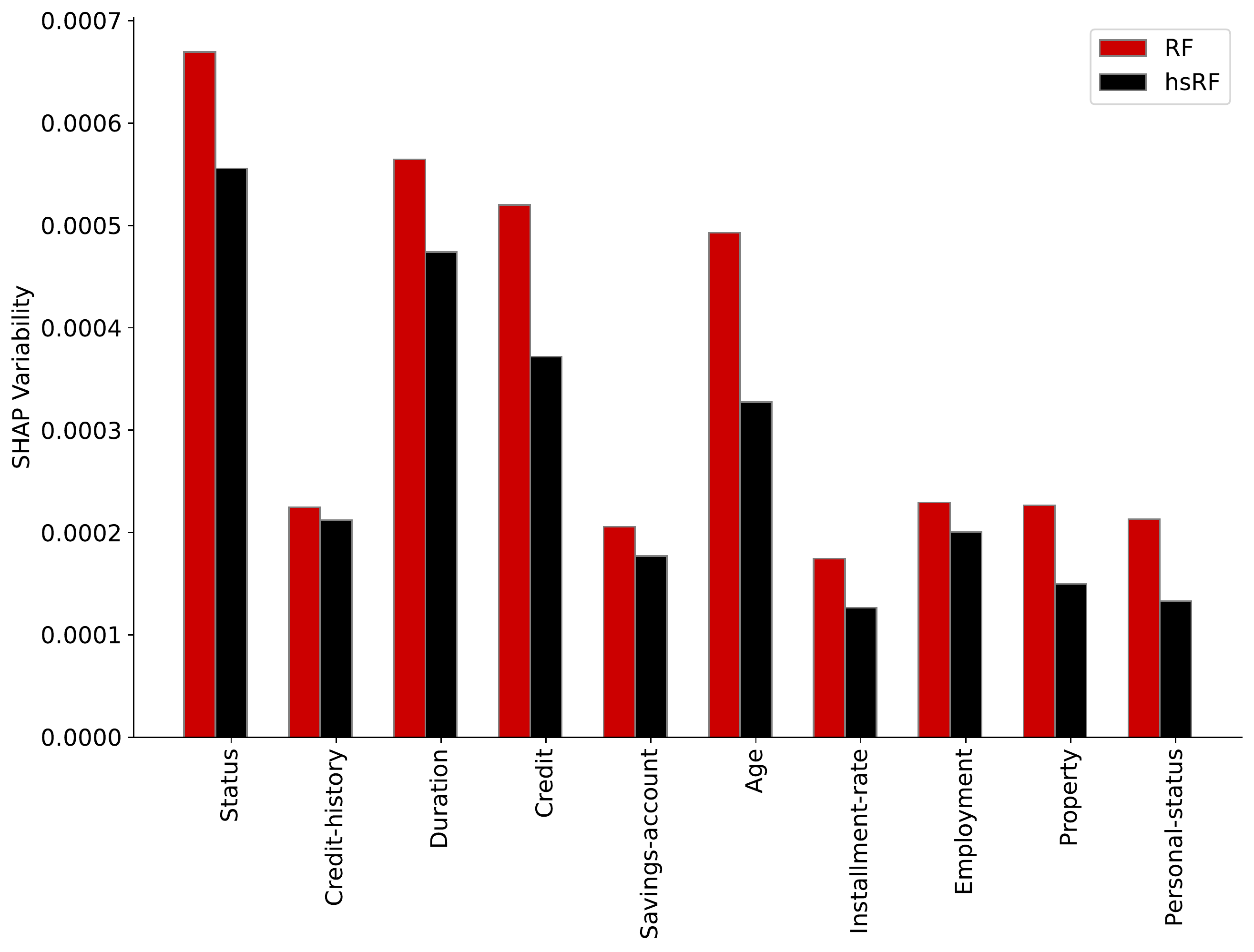}}\\
        (C) Juvenile & (D) Compas\\  
        \raisebox{-\height}{\includegraphics[width=0.45\textwidth]{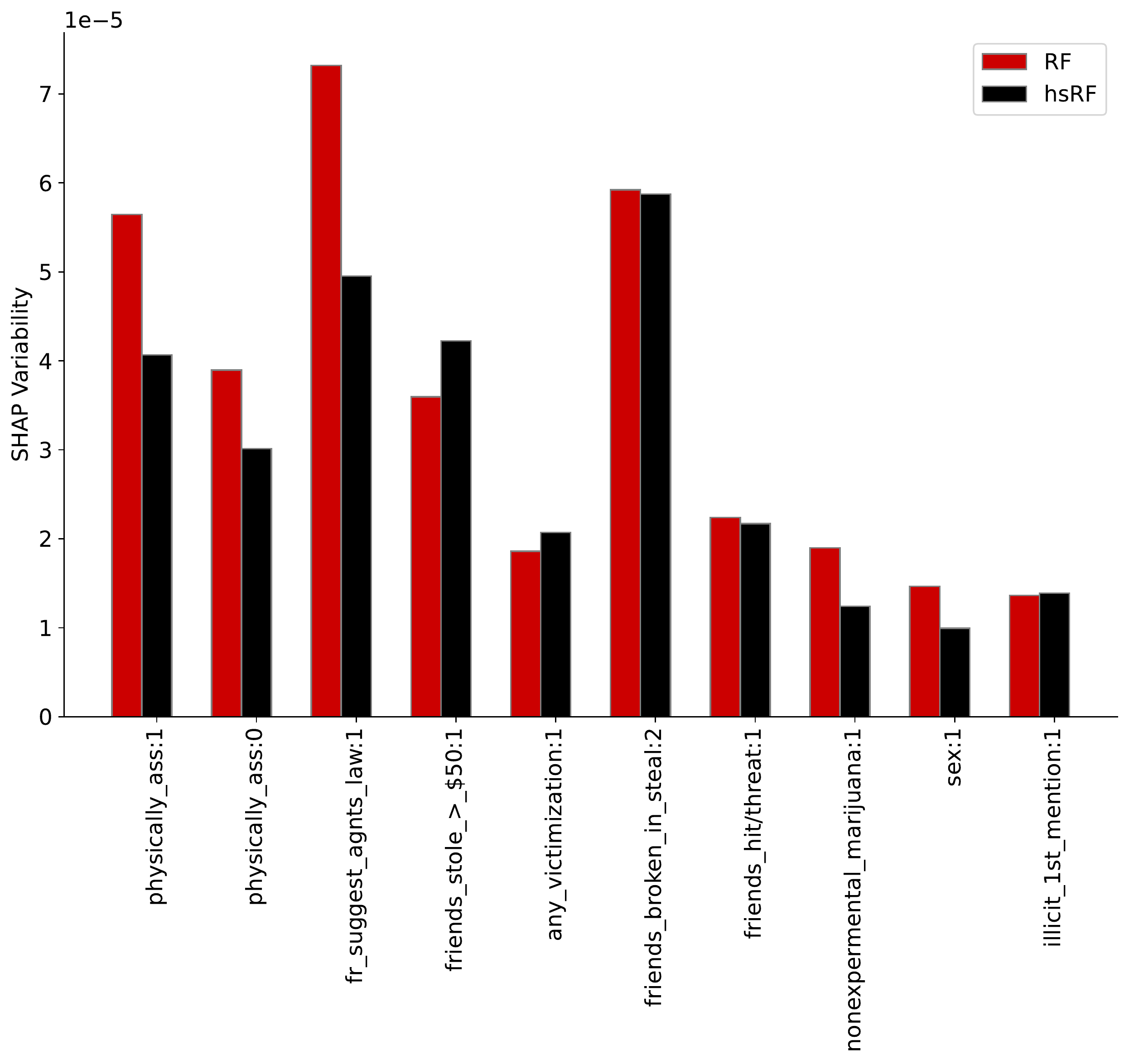}} & \raisebox{-\height}{\includegraphics[width=0.45\textwidth]{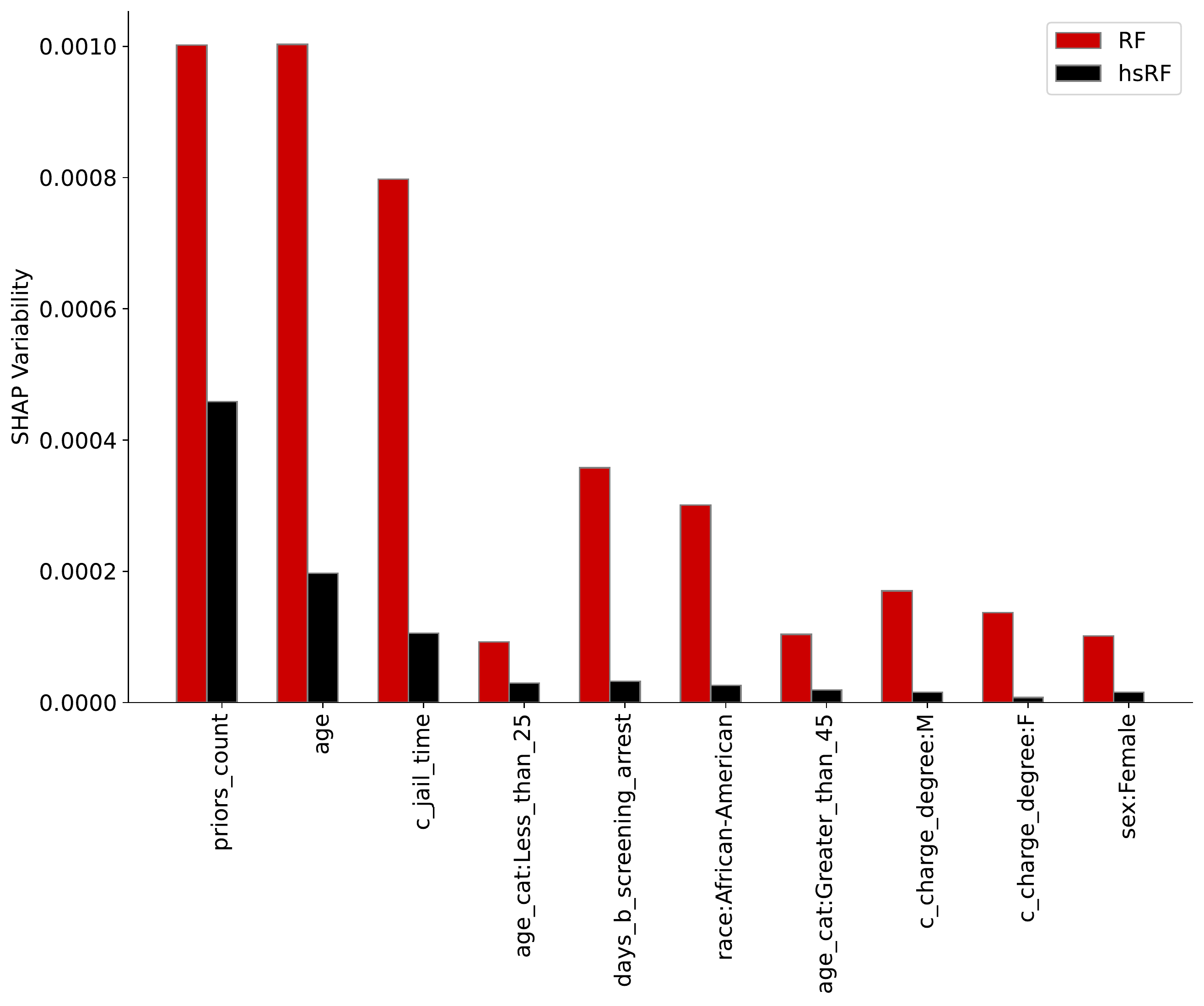}}\\
    \end{tabular}
    \caption{Comparison of SHAP plots learnt by Random Forests on different (large) datasets before / after applying \method.
    \methods displays lower variability across different data perturbations, indicating enhanced stability.}
    \label{fig:shap_variability_big}
\end{figure}

\section{Theory}
\label{sec:supp_theory}
In this section, we provide a proof of \cref{thm:ridge_to_shrinkage} and the heuristic arguments discussed in \cref{sec:theory}.

\subsection{Proof of Theorem \ref{thm:ridge_to_shrinkage}}
\label{subsec:proof_of_thm}
Throughout this proof, we denote the left and right child of a node $\node_{i}$ by $\node_{i,L},\node_{i,R}$. Further, we assume WLOG that the sample mean $\hat\E_{\node_{0}}\braces{y} = 0$. Then using the well known solution to ridge regression, we have that 
\begin{equation*}
    \hat{\bbeta}_{\lambda} = (\Psi(\data)\Psi(\data)^{T} + \lambda I)^{-1}\Psi(\data)^{T}\by
\end{equation*}
Using the orthogonality of the decision stumps and the norm of the features, we have that the transformed covariance matrix  $\Psi(\data)\Psi(\data)^{T}$ is a diagonal matrix with entries $(\Psi(\data)\Psi(\data)^{T})_{ii} = N(\node_{i})$. Therefore, we have the following expansion for the $i$-th coordinate of $\hat{\bbeta}_{\lambda}$
\begin{align}
\label{eq:beta_lambda_expansion}
    \hat{\beta}_{\lambda,i} & =  \frac{\langle \psi_{t_{i}}(\data),\by\rangle}{N(\node_{i}) + \lambda}\nonumber \\ 
    & = \frac{N(\node_{i,R})\sum_{\bx_{i}\in \node_{i,L}}y_{i} - N(\node_{i,L})\sum_{\bx_{i}\in \node_{i,R}}y_{i} }{\sqrt{N(\node_{i,L})N(\node_{i,R})}(N(\node_{i}) + \lambda)} \nonumber \\ 
    & = \frac{\sqrt{N(\node_{i,L})N(\node_{i,R})}}{N(\node_{i}) + \lambda}\paren*{\hat\E_{\node_{i,L}}\braces{y} - \hat\E_{\node_{i,R}}\braces{y}} \nonumber \\ 
     & = \frac{\sqrt{N(\node_{i,L})N(\node_{i,R})}}{N(\node_{i}) + \lambda}\paren*{\hat\E_{\node_{i,L}}\braces{y} - \frac{N(\node_{i})\hat\E_{\node_{i}}\braces{y}}{N(\node_{i,R})} + \frac{\hat\E_{\node_{i,L}}\braces{y}N(\node_{i,L})}{N(\node_{i,R})}}\nonumber \\ 
     & = \sqrt{\frac{N(\node_{i})^{2}N(\node_{i,L})}{N(\node_{i,R})(N(\node_{i})+\lambda)^2}}\paren*{\hat\E_{\node_{i,L}}\braces{y} - \hat\E_{\node_{i}}\braces{y}}
\end{align}

Then given a query point $\bx$ with leaf-to-root path $\node_L \subset \node_{L-1} \subset \cdots \subset \node_0$, assume WLOG that each descendant is always the left child of its ancestor. Then using \eqref{eq:beta_lambda_expansion} and , the prediction at $\bx$ can then be expanded as follows 
\begin{align*}
    \hat{f}_{\lambda}(\bx) & =  \Psi(\bx)^{T}\hat{\bbeta_{\lambda}} \\ 
    & = \sum^{L-1}_{l=0}
    \frac{N(\node_{i,R})\langle \psi_{t_{l}}(\data),\by\rangle}{\sqrt{N(\node_{l,L})N(\node_{l,R})}(N(\node_{l}) + \lambda)} \\
    & = \sum^{L-1}_{l=0} \frac{N(\node_{l})\paren*{\hat\E_{\node_{l,L}}\braces{y} - \hat\E_{\node_{l}}\braces{y}}}{N(\node_{l}) + \lambda} \\
    &= \sum^{L-1}_{l=0} \frac{\hat\E_{\node_{l,L}}\braces{y} - \hat\E_{\node_{l}}\braces{y}}{1 + \lambda/N(\node_{l})}
\end{align*}
We see that the equation above is precisely the same as that proposed for \methods in \eqref{eq:hierarchical_shrinkage}. 
\subsection{Heuristics for equation \eqref{eq:motivating_constant_lambda}}
\label{subsec:motivating_constant_lambda_supp}

Consider the generative model
\begin{equation}
    y = f(\bx) + \epsilon.
\end{equation}
where $\E\braces*{\epsilon~|~\bx} = 0$.
Suppose for the moment that $f(\bx) = \bbeta^T\bx$ is linear, so that
\begin{equation} \label{eq:supp_diff_in_cell_means}
    \E\braces*{y~|~\node_L} - \E\braces*{y~|~\node_R}
     = \bbeta^T\paren*{\E\braces*{\bx~|~\node_L} - \E\braces*{\bx~|~\node_R}}.
\end{equation}
Now further assume that for some $\beta_{min}$ and $\beta_{max}$, we have 
$$
\beta_{min} \leq \abs*{\beta_{i}} \leq \beta_{min}
$$
for all $i$.
Plugging these into \eqref{eq:supp_diff_in_cell_means} allows us to compute
\begin{equation}
\label{eq:diameter_bound}
\beta_{min}^2\textnormal{diam}(\node)^2 \leq \paren*{\E\braces*{y~|~\node_L} - \E\braces*{y~|~\node_R}}^2 \leq \beta_{max}^2\textnormal{diam}(\node)^2.
\end{equation}

Next, for a given node $\node$, let its side lengths be denoted by $l_{i}$, with corresponding measure$\mu(\node) = \prod^{d}_{i=1}l_{j}$. Assuming that all the side lengths of $\node$ are similar, we have that the diameter of the node $\textnormal{diam})(\node) \approx \mu(\node)^{1/d}$. Furthermore, assuming that when a node $\node$ is split into left and right children $\node_{L},\node_{R}$, we split the node fairly evenly, such that the measure of $\node_{L},\node_{R}$ are roughly equivalent. Then we have that the measure of the node $\node$  is $\approx 2^{-depth(\node)}$, which implies that $diam(\node) \approx 2^{-depth(\node)/d}$. Substituting this into \eqref{eq:diameter_bound} gives us the heuristic claimed  in \eqref{eq:motivating_constant_lambda}.

For a more general $C^1$ regression function $f$, note that the $C^1$ assumption implies that $f$ is approximately linear locally, i.e. when the nodes are small enough.

\end{document}